\documentclass{article} 
\usepackage{iclr2022_conference,times}

\usepackage{amsmath,amsfonts,bm}

\def\eqref#1{equation~\ref{#1}}

\def\1{\bm{1}}

\DeclareMathAlphabet{\mathsfit}{\encodingdefault}{\sfdefault}{m}{sl}
\SetMathAlphabet{\mathsfit}{bold}{\encodingdefault}{\sfdefault}{bx}{n}

\newcommand{\R}{\mathbb{R}}

\usepackage[utf8]{inputenc} 
\usepackage[T1]{fontenc}    
\usepackage{hyperref}       
\usepackage{url}            
\usepackage{booktabs}       
\usepackage{amsfonts}       
\usepackage{nicefrac}       
\usepackage{microtype}      
\usepackage{xcolor}         

\usepackage{dashrule}
\usepackage{booktabs}
\usepackage{caption}
\usepackage{color}
\usepackage{amsfonts}
\usepackage{amsmath}
\usepackage{algorithm}
\usepackage{algorithmic}
\usepackage{graphicx}
\usepackage{nicefrac}
\usepackage{bbm}
\usepackage{amsthm}
\usepackage{wrapfig}
\usepackage{tikz}
\usepackage[titletoc,title]{appendix}
\usepackage{stmaryrd}
\usepackage{amsmath}
\usepackage{enumitem}
\usepackage{mathtools}
\usepackage{subcaption,booktabs}

\newtheorem{definition}{Definition}
\newtheorem{lemma}{Lemma}
\newtheorem{corollary}{Corollary}
\newtheorem{theorem}{Theorem}
\newtheorem{proposition}{Proposition}

\newtheorem{remark}{Remark}

\newtheorem{fact}{Fact}

\def\be{\begin{equation}}
\def\ee{\end{equation}}
\def\beas{\begin{eqnarray*}}
	\def\eeas{\end{eqnarray*}}
\def\bea{\begin{eqnarray}}
\def\eea{\end{eqnarray}}

\newcommand{\x}{{\mathbf x}}
\newcommand{\y}{{\mathbf y}}
\newcommand{\g}{{\mathbf g}}

\newcommand{\uu}{{\mathbf u}}

\newcommand{\aaa}{{\mathbf a}}
\newcommand{\bb}{{\mathbf b}}

\newcommand{\A}{{\mathcal A}}
\newcommand{\B}{{\mathcal B}}

\newcommand{\X}{{\mathcal X}}

\newcommand{\EE}{{\mathbb E}}

\newcommand{\N}{{\mathbb N}}
\newcommand{\SSS}{{\mathbb S}}
\newcommand{\nocontentsline}[3]{}
\newcommand{\tocless}[2]{\bgroup\let\addcontentsline=\nocontentsline#1{#2}\egroup}

\newcommand{\abs}[1]{\left\lvert#1 \right\rvert}
\newcommand{\norm}[1]{\left\|#1 \right\|}

\newcommand{\mat}[1]{\llbracket#1\rrbracket}
\newcommand{\sep}[2]{\mathrm{sep}_{(#1)}\left( #2 \right)}
\newcommand{\rank}[1]{\mathrm{rank}\left( #1 \right)}
\newcommand{\trace}[1]{\mathrm{trace}\left( #1 \right)}

\renewcommand{\bibsection}{}

\def\multiset#1#2{\ensuremath{\left(\kern-.3em\left(\genfrac{}{}{0pt}{}{#1}{#2}\right)\kern-.3em\right)}}

\newcommand{\eg}{\emph{e.g.}}
\newcommand{\ie}{\emph{i.e.}}

\newcommand{\wrt}{w.r.t.}

\newcommand{\RNum}[1]{\uppercase\expandafter{\romannumeral #1\relax}}

\title{The Inductive Bias of In-Context Learning: \\ Rethinking Pretraining Example Design}

\author{Yoav Levine, Noam Wies, Daniel Jannai, Dan Navon, Yedid Hoshen \& Amnon Shashua \\
The Hebrew University of Jerusalem\\
\texttt{\{yoav.levine, noam.wies, daniel.jannai, dan.nav\}@mail.huji.ac.il}
}

\iclrfinalcopy
\begin{document}

	\maketitle
	
	\begin{abstract}
		Pretraining Neural Language Models (NLMs) over a large corpus involves chunking the text into training examples, which are contiguous text segments of sizes processable by the neural architecture. 
		We highlight a bias introduced by this common practice:
		we prove that the pretrained NLM can model much stronger dependencies between text segments that appeared in the same training example, than it can between text segments that appeared in different training examples.  
		This intuitive result has a twofold role. First, it formalizes 
		the motivation behind a broad line of recent successful NLM training heuristics, proposed for the pretraining and fine-tuning stages, which do not necessarily appear related at first glance.
		Second, our result clearly indicates further improvements to be made in NLM pretraining for the benefit of Natural Language Understanding tasks. 
		As an example, we propose ``kNN-Pretraining": we show that including semantically related non-neighboring sentences in the same pretraining example yields improved sentence representations and open domain question answering abilities.
		This theoretically motivated degree of freedom for \textit{pretraining example design} indicates new training schemes for self-improving representations. 
	\end{abstract}
	
	\section{Introduction}
	\ifdefined\SQUEEZE \vspace{-3mm} \fi

	Beyond excelling in their core task of pure language modeling, modern Neural Language Models (NLMs) show impressive zero- and few-shot abilities in more general Natural Language Understanding (NLU) tasks~\citep{GPT3}. 
	This implies that the training corpus contains the information required for performing such tasks, and moreover it implies that the common pretraining process 
	grants the trained NLM some access to these higher level capabilities. 
	In this paper, we highlight a connection between the quality of the emergent NLU capabilities and a basic component in the NLM training scheme: the process of segmenting the corpus into training examples.
	
	Specifically, NLMs self-train over huge training corpora (typically, billions to trillions of words).
	A basic, automatic, operation in the training pipeline is to segment these corpora into {training examples}: contiguous text chunks of sizes processable by the neural architecture (typically, up to thousands of words). 
	We formalize an expressivity bias that this segmentation process introduces, to be referred to as the \textit{in-context bias}, which directly affects the NLM's ability to integrate cross-corpus information.
	We show that the NLM can model much stronger dependencies between sentences that were shown together at least once {in-context}, \ie, in the same training example, than between sentences that were never shown together in the same input.
	This inductive bias may be good for language modeling, but it implies that NLU capabilities that involve integrating information from different examples across the corpus (see, \eg, figure~\ref{fig:fig1}), are under-favored by design in the current setting.
	Thus, if one sentence in the corpus can elucidate the meaning of another sentence (e.g., defines a hard concept or provides auxiliary information), our result implies that a model that saw them in different training examples will enjoy this elucidation less than a model that saw them in the same training example.   
	
	While standard approximation results examine the expressivity of an architecture over a single input, our theoretical approach pertains to the entire training process, and examines the expressive capacity of the resultant NLM with respect to the training set. 
	Therefore, our approximation result ties an optimization parameter (the learning-rate) to the regular NLM architecture expressivity parameters (depth, width). 
	Intuitively, sentences that were never shown in the same input can only access each other via the weights of the network during training. The mechanism for ``storing" information in the network involves a very small learning-rate term $\eta$; our analysis formalizes and quantifies an ``expressivity  toll" that the model pays when making use of such harder-to-access stored information.
	
	We employ the tool of a function’s separation rank with respect to subsets of its variables, which
	quantifies its ability to model input dependencies between these subsets. The separation rank
	was employed for analyzing the dependencies modeled by convolutional~\citep{cohen2017inductive}, recurrent~\citep{levine2018benefits}, and self-attention ~\citep{levine2020limits} networks with respect to a single input example.
	In order to analyze an NLM's ability to model dependencies between \textit{different} training examples, we refine the usage of this measure in two manners: (1) we introduce the $\varepsilon$-separation rank, which measures the effective ability of a function to model dependencies in a finite precision setting, and (2) we modify the separation rank such that it can account for the more intricate mechanism of mixing between variables that occurs in the sequential case.  
	
	Specifically, we upper bound the log of the separation rank of a depth $L$ width $d_x$ self-attention based NLM, with respect to two sentences that are shown in its input, by $\tilde{O}(d_xL)$, and prove that this bound is tight. 
	On the other hand, we upper bound this measure with respect to two sentences that were never shown in the same input by $\tilde{O}(d_x[L-0.5 \log_3(\eta^{-1})])$.
	Given common learning-rate values of $\eta\in[10^{-6},10^{-4}]$, 
	{this implies a guaranteed ``depth deficit" of $\sim 6$ layers for modeling dependencies between sentences that are not seen in the same training example}. After the presentation of our results, we point at empirical evidence that imply that this depth deficit is more significant, and may behave like a fraction of $L$. We leave attempts to tighten the depth deficit estimates to future work.  
	
	\vspace{-3mm}
	\subsection{The in-context bias drives a variety of existing approaches}
	\vspace{-2mm}
	Several recent works intuitively rely on the above formalized in-context expressivity bias in different manners, and significantly improve both task-specific training and pretraining of NLMs. 
	\cite{gao2020making} advance the frontier in $k$-shot learning via finetuning. 
	They show that by concatenating several related training examples per input, instead of using standard fine-tuning practice of one example per input, the $k$-shot performance on sentence similarity tasks is considerably boosted.
	Another example was pointed out in~\cite{Humeau2020Poly-encoders,thakur2020augmented}: when training for sentence similarity tasks, including both sentences in the same input leads to a performance gain of around $10$ points relative to separately encoding each sentence.
	In the challenging setting of open-domain question answering, 
	\cite{izacard2020leveraging} jointly attend to all documents that may contain the answer, and show large gains relative to prior methods that consider these documents in separate forward passes.  
	
	Turning our focus to methods that leverage the in-context bias for improved pretraining, the most straightforward effort is a body of work aimed at reducing the quadratic dependence of the Transformer computation on input sequence length~\citep{tay2020long}. While allowing for more text in-context during training, this does not improve the model's ability to integrate text across different documents in the corpus. The following approaches take a further step and enable direct cross-corpus connections during pretraining. 
	\cite{lewis2020pre} attend to related documents 
	when maximizing the likelihood of a target document. The scope of related documents is restricted by meta-data: taken from the same Wikipedia entry as the input, or published on the same date.  
	\cite{guu2020realm} expand the scope of the related documents, by training a Knowledge-Retrieval  model that has access to the entire Wikipedia corpus. They retrieve several related documents per target document, but condition on each related document independently.
	Outside of the natural language domain,~\cite{rao2021msa} train a Transformer based protein-LM that receives multiple related protein sequences in-context. Their protein-LM surpasses previous methods which process one sequence per input
	by a wide margin, with significant parameter efficiency.	
	
	\vspace{-3mm}
	\subsection{Leveraging the in-context bias for NLU oriented training}
	\vspace{-2mm}
	
	Though the in-context bias is intuitive, the above subsection surveys recent advances that leverage it in non-trivial manners. 
	Having formalized the theoretical advantage for in-context integration of related text, the roots of the above successes can be unified, and importantly, new methods for tilting the pretraining bias towards NLU tasks are indicated.
	Following the presentation of our theoretical results in section~\ref{sec:2}, we detail in section~\ref{sec:3} two controlled setting exemplifications of new methods that directly leverage the in-context bias.
	
	Our first experiment augments the Task Adaptive PreTraining (TAPT) setting of~\cite{gururangan2020don},
	in which an NLM that was pretrained on a general corpus continues pretraining (with its original objective) on the training set of an NLU task.
	We perform TAPT on the SentEval sentence similarity benchmark~\citep{conneau2018senteval}, and during TAPT introduce the following augmentation: along with SentEval sentences, we simultaneously pretrain on related sentences from Wikipedia, the general pretraining corpus.
	The related sentences are found via k-Nearest Neighbors (kNN) search between the   embeddings of SentEval examples and all Wikipedia sentences; we thus dub this approach \textit{kNN-TAPT}.
	Importantly, during kNN-TAPT, each input includes a training example from the task, appended \textit{in-context} by its Wikipedia neighbors.
	We demonstrate significant gains of the kNN-TAPT over regular TAPT on SentEval sentence similarity tasks. 
	A dedicated ablation study shows the significance of adding the general corpus neighbors {in-context}, versus in separate training examples, during kNN-TAPT. 
	
	\begin{wrapfigure}{r}{0.62\textwidth}
		\begin{center}
			\vspace{0mm}
			\includegraphics[scale=0.35,clip=false,trim=107 120 100  52]{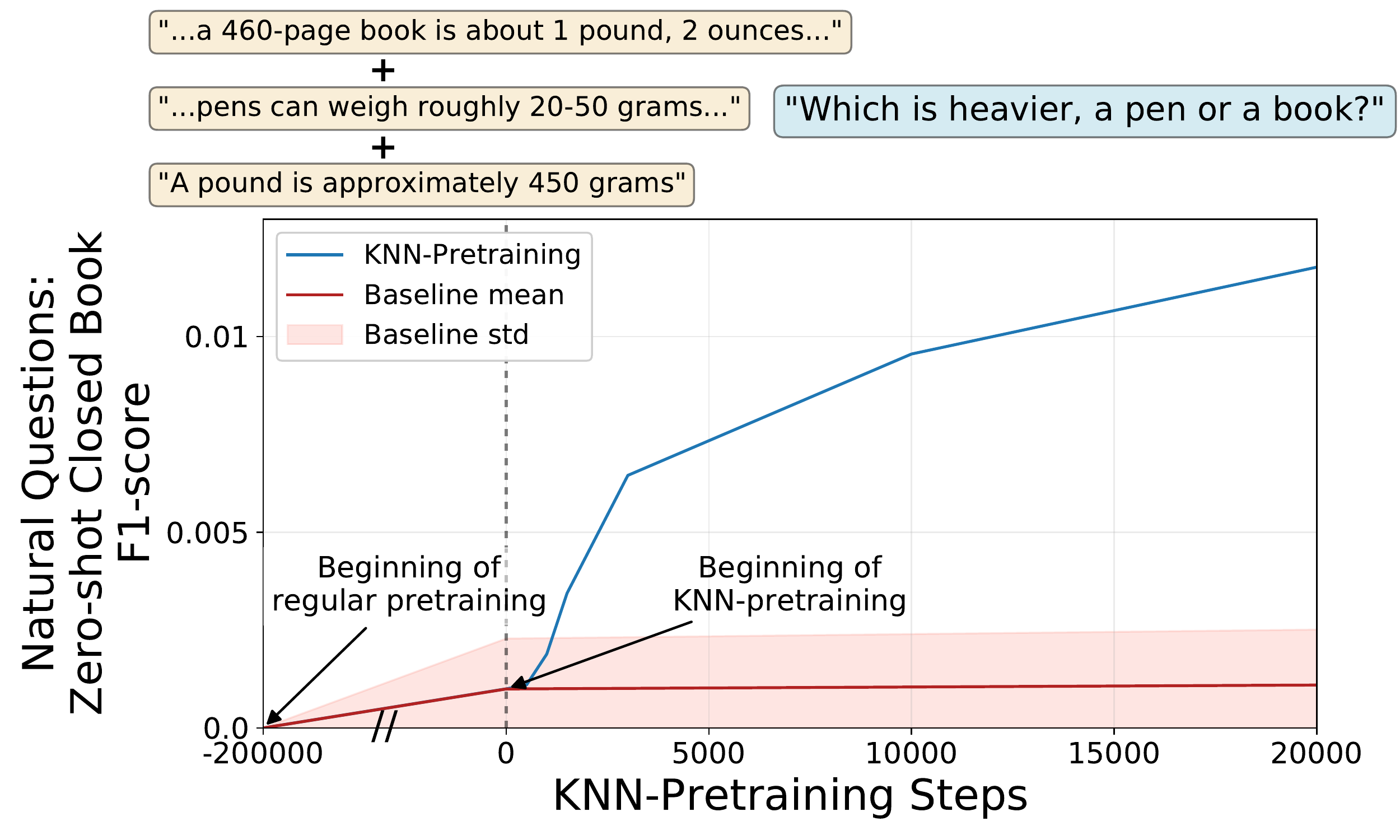}
		\end{center}
		\vspace{10mm}
		\caption{A $10\%$ addition of kNN-Pretraining boosts zero-shot closed book QA score by$\sim5$X (evaluation set size is 20,000).\label{fig:fig1}}\vspace{-3.5mm}
	\end{wrapfigure}
	
	Our second experiment introduces a task-independent pretraining phase, 
	dubbed kNN-Pretraining.
	As in kNN-TAPT, we group together sentences with similar sentence representations 
	in the same training example, but in kNN-Pretraining we use only sentences from the general pretraining corpus.
	This can be viewed as a sentence-focused variation of the above surveyed pretraining schemes in~\cite{lewis2020pre} and~\cite{guu2020realm}, who operate on full documents (up to $512$ each), and is very similar to RETRO by~\cite{borgeaud2021improving} (DeepMind), who show the benefits of this approach given much larger resources.
	Figure~\ref{fig:fig1} 
	shows that after regular pretraining for $200$K steps on Wikipedia, the zero-shot closed book performance of $3$ different randomly initialized GPT2-medium models ($345$M parameters) on open domain questions from Wikipedia~\citep{kwiatkowski2019natural} is very low (correct on less than $50$ questions out of 20,000 in the evaluation set). Adding kNN-Pretraining for $20$K steps raises  performance significantly (correct on roughly $250$ questions in the evaluation set), reflecting the enhanced ability to integrate knowledge from related sentences, acquired via the in-context bias. 
	
	In summary, our main contributions are:
	\ifdefined\SQUEEZE \vspace{-1.1mm} \fi
	\begin{itemize}
	\ifdefined\SQUEEZE \vspace{-1.1mm} \fi
\item We formally establish the \textit{in-context bias}: information  within pretraining examples is better represented than information integrated across pretraining examples.
  	\ifdefined\SQUEEZE \vspace{-1.1mm} \fi
	\item We ask and answer a new type of network expressivity question:  \textit{how expressive is a network with respect to examples seen during its training process?}
  	\ifdefined\SQUEEZE \vspace{-1.1mm} \fi
\item We demonstrate that in-context bias motivated ``pretraining example design" elicits better representations from the same data: \textit{kNN-Pretraining} improves on several NLU tasks.
	\end{itemize}
	\ifdefined\SQUEEZE \vspace{-3mm} \fi

	\ifdefined\SQUEEZE \vspace{-3mm} \fi
	\section{Theoretical analysis: The in-context bias of self-attention}\label{sec:2}
	\ifdefined\SQUEEZE \vspace{-3mm} \fi
	In this section, we consider the entire NLM training procedure as a functional  
	that receives an unlabeled training corpus and outputs a trained NLM. 
	Our analysis focuses on the corpus segmentation into training examples as a hyper-parameter of this functional.
	We reduce the high-level notion of representing ``cross-corpus correlations" to a quintessential case study: we quantify the NLM's ability to model dependencies between two sentences that appear in the same training example (the \textit{in-context} representation) and in different training examples (the \textit{sequential} representation).
	
	We believe that the in-context bias can be shown to exist in a broad range of architectures, but we focus on self-attention since almost all modern NLMs are based on the Transformer architecture of~\cite{vaswani2017attention}.
	Our theoretical framework is based on that of \cite{levine2020limits,wies2021vocabulary}, who analyze a simplified, theoretically accessible, self-attention network.
	They study the expressivity of this self-attention architecture with respect to its input, and use a measure of a multivariate function's ability to correlate two subsets of its variable set, referred to as the separation rank.
	The analyzed framework captures the connectivity of self-attention but omits its softmax and ReLU non-linearities (see eq.~\ref{eq:our_layer} below).
	We refer the reader to \cite{levine2020limits,wies2021vocabulary} for a discussion on the impact of these relaxations. Essentially, they are shown to weaken the overall network power but still allow a meaningful comparison of the self-attention integration abilities.  
	Importantly, both works derive unforeseen theoretical conclusions from analyses of the separation rank measure for this architecture class, and then provide extensive empirical corroboration for their manifestation in common Transformer architectures, reinforcing the relevance of this setting.
	In the following, we describe in section~\ref{sec:2:1} the analyzed \textit{in-context} and \textit{sequential} self-attention representations of two sentences. 
	Then, in section~\ref{sec:2:2}, we present the separation rank, which we use in section~\ref{sec:2:3} for quantifying the advantage of in-context representations versus sequential ones.
	
	\ifdefined\SQUEEZE \vspace{-4mm} \fi
	\subsection{The analyzed in-context and sequential representations}\label{sec:2:1}
	\ifdefined\SQUEEZE \vspace{-3mm} \fi
	
	For an input sequence of $N$ embedding vectors $\{\x^j\in\R^{d_x}\}_{j=1}^N$, denote the function realized by the analyzed $H$-headed depth-$L$ width-$d_x$ Transformer architecture at output location $i\in[N]$ by: $\g^{i,L,d_x}_{\mathcal{W}}\left(\x^{1},...,\x^{N}\right)\in\R^{d_x}$, where $\mathcal{W}$ stands for learned parameters, recursively defined: 
	
	\ifdefined\SQUEEZE \vspace{-8mm} \fi
	\begin{align}\label{eq:our_layer}
	\g_{\mathcal{W}}^{i,l+1,d_x}\left(\g_{\mathcal{W}}^{1,l,d_x},...,\g_{\mathcal{W}}^{N,l,d_x}\right)=\sum_{h=1}^{H}&W^{\textrm{O},l,h}\sum_{j=1}^{N}a_{hj}^{i}W^{\textrm{V},l,h}\g_{\mathcal{W}}^{j,l,d_x}
	\\\nonumber
	a_{hj}^{i}:=\left\langle W^{\textrm{Q},l,h}\g_{\mathcal{W}}^{i,l,d_x},W^{\textrm{K},l,h}\g_{\mathcal{W}}^{j,l,d_x}\right\rangle	~~~&;~~~~
	\g_{\mathcal{W}}^{i,0,d_x}=\x^{i}
	\end{align}
	\ifdefined\SQUEEZE \vspace{-5mm} \fi
	
	where $\mathcal{W}$ is composed of Key, Query, Value and Output matrices:
	$\forall l\in[L] ,h\in[H]$, $W^{\textrm{K},l,h},W^{\textrm{Q},l,h},W^{\textrm{V},l,h}$, $(W^{\textrm{O},l,h})^{\top}\in\R^{d_a\times d_x}$, where we assume the standard choice $d_a={d_x}/{H}$.
	For a word from vocabulary of size $V$, $w\in[V]$, the translation into the Transformer dimension is done via a mapping $E_{M^{\textrm{V}}}:[V]\to\R^{d_x}$:
	\ifdefined\SQUEEZE \vspace{-2mm} \fi
	\begin{equation}\label{eq:embedding}
	E_{M^{\textrm{V}}}\left(w\right) = \left(M^{\textrm{V}}\right)_{w}
	\ifdefined\SQUEEZE \vspace{-4mm} \fi
	\end{equation}
	
	where $M^{\textrm{V}}\in\R^{d_x\times V}$ is the learned vocabulary embedding matrix, and $\left(M^{\textrm{V}}\right)_w$ is its $w$th column, also referred to as the learned word embedding for $w$. 
	Overall, the function of the analyzed Transformer over a sequence of $N$ words $\{w^j\in[V]\}_{j=1}^N$ can be written by composing eqs.~\ref{eq:our_layer} and~\ref{eq:embedding}:
	
	\ifdefined\SQUEEZE \vspace{-4mm} \fi
	\begin{equation}\label{eq:transformer}
	\y^{i,L,d_x}_{\mathcal{W}, M^{\textrm{V}}}\left(w^{1},...,w^{N}\right)= \g^{i,L,d_x}_{\mathcal{W}}\left(E_{M^{\textrm{V}}}\left(w^{1}\right),...,E_{M^{\textrm{V}}}\left(w^{N}\right)\right)
	\end{equation}
	\ifdefined\SQUEEZE \vspace{-4mm} \fi
	
	For simplicity of presentation, we examine two sentences $S_1$ and $S_2$ of equal length ${N}$: $S_1=\{w_1^j\}_{j=1}^N$ and $S_2=\{w_2^j\}_{j=1}^N$.
	The in-context representation simply concatenates both in the input:
	\ifdefined\SQUEEZE \vspace{-2mm} \fi
	\begin{align}\label{eq:in_context}
	\y^{i,L,d_x}_{\textrm{in-context}}\left(S_1,S_2\right):&=\y^{i,L,d_x}_{\mathcal{W}, M^{\textrm{V}}}\left(w_1^{1},...,w_1^{N},w_2^{1},...,w_2^{N}\right).
	\end{align} 
	\ifdefined\SQUEEZE \vspace{-5mm} \fi
	
	For the sequential approach, we consider a setup in which sentence $S_1$ is inserted into the network at training step $t$ and sentence $S_2$ is inserted into the network at training step $t+1$.
	The output of the network at training step $t$ is therefore:
	$\y^{i,L,d_x}_{\mathcal{W}_{t}, M^{\textrm{V}}_{t}}\left(S_1\right)$,
	where $\mathcal{W}_{t},M^{\textrm{V}}_{t}$ stand for all the learned weights before training step $t$.
	Focusing on autoregressive NLMs for simplicity of presentation (the analysis holds for bidirectional NLMs as well), the log-likelihood loss is given by $\mathcal{L}\left(S_1\right)=-\sum_{j=1}^N\log \left[\left(softmax\left\{\left(M_t^{\textrm{V}}\right)^\top\y^{j,L,d_x}_{\mathcal{W}_{t}, M^{\textrm{V}}_{t}}\left(S_1\right)\right\}\right)_{w^{j+1}_1}\right]$, and the gradient update for any learned weight $\theta\in\{\mathcal{W}_{t}, M^{\textrm{V}}_{t}\}$ is:
	$\theta_{t+1}\left(S_1;\eta\right) = \theta_{t} - \eta \cdot  \nicefrac{\partial\mathcal{L}\left(S_1\right)}{\partial \theta_{t}}$, where $\eta$ is the learning rate.
	Accordingly, the analyzed sequential representation is the network output after training step $t+1$:
	\ifdefined\SQUEEZE \vspace{-1mm} \fi
	\begin{align}\label{eq:seq_1}
	\y^{i,L,d_x,\eta}_{\textrm{sequential}}\left(S_1,S_2\right):&=\y^{i,L,d_x}_{\mathcal{W}_{t+1}\left(S_1;\eta\right),M^{\textrm{V}}_{t+1}\left(S_1;\eta\right)}\left(S_2\right).
	\end{align}
	\ifdefined\SQUEEZE \vspace{-5mm} \fi
	
	In practice, two relevant  non-neighboring sentences are not necessarily shown in consecutive pretraining steps.
	In comparison to the realistic scenario of $S_1$ and $S_2$ appearing at \textit{any} training step,
	this simplifications tilts the representation in favor of modeling high correlations between $S_1$ and $S_2$.
	Thus, by upper bounding the ability to correlate $S_1$ and $S_2$ in the setting of eq.~\ref{eq:seq_1} (as we do in section~\ref{sec:2:3}), we establish an inherent limitation of the network to access information that was stored in its weights via the gradient update mechanism.       
	In the next subsection, we present our approach for measuring a network's ability to correlate two sentences seen during training, which we will use in order to separate between the {in-context} and {sequential} settings.
	
	\ifdefined\SQUEEZE \vspace{-3mm} \fi
	\subsection{A measure for modeling in-context and sequential dependencies}\label{sec:2:2}
	\ifdefined\SQUEEZE \vspace{-2mm} \fi
	In this section, we refine the separation rank, used in prior work in order to analyze the dependencies between two sentences appended in-context. 
	In section~\ref{sec:2:2:1} we present the separation rank and introduce a finite precision refinement of it, referred to as \textit{the effective separation rank}, which helps to elucidate the degradation in integration ability caused by the gradient update mechanism.
	In section~\ref{sec:2:2:2} we point at a structural problem in employing the separation rank in the same manner in which it was employed in prior work that analyzed only architecture expressivity, and introduce the \textit{the sequential separation rank}, meaningful for both the in-context and sequential cases. 
	
	\ifdefined\SQUEEZE \vspace{-2mm} \fi
	\subsubsection{The effective separation rank}\label{sec:2:2:1}
	\ifdefined\SQUEEZE \vspace{-2mm} \fi
	The separation rank, introduced in \cite{beylkin2002numerical} for high-dimensional numerical analysis, 
	was employed for various applications, \eg,~chemistry~\citep{harrison2003multiresolution}, particle engineering~\citep{hackbusch2006efficient}, and machine learning~\citep{beylkin2009multivariate}. 
	More recently,
	the separation rank has been established as a measure of dependencies modeled by deep convolutional and recurrent networks \wrt~their inputs~\citep{cohen2017inductive,cohen2017analysis,levine2018benefits}, and tied to quantum entanglement measures for proving that these deep learning architectures can model elaborate many-body quantum particle correlations~\citep{levine2018deep,PhysRevLett.122.065301,PhysRevLett.124.020503}. 
	Recently, \cite{levine2020limits,wies2021vocabulary} employed this measure for studying the expressivity of a self-attention architecture with respect to its input.
	
	For a function $y(A,B)$ over variables $A=\{\aaa^j\in\X\}_{j=1}^M$ and $B=\{\bb^j\in\X\}_{j=1}^M$, the separation rank \wrt~$(A,B)$ is the minimal number of summands that together sum up to equal $y(A,B)$, where each summand is \emph{multiplicatively separable \wrt~$(A,B)$}, \ie,~is equal to a product of two functions~--~one that intakes only $A$ variables and another that intakes only $B$ variables. 
	Formally, the \emph{separation rank} of $y:\X^{2M}\to\R$ \wrt~$(A,B)$ is defined as follows:
	\ifdefined\SQUEEZE \vspace{-3.5mm} \fi
	\begin{equation}\label{eq:sep}
	\sep{A,B}{y}:=\min\left\{R\in\N:\exists{g_1{\ldots}g_R,g'_1{\ldots}g'_R:\X^M\to\R}~s.t.~y\left(A,B\right)=\sum_{r=1}^{R}g_{r}\left(A\right)g'_r\left(B\right)
	\right\}
	\end{equation}
	\ifdefined\SQUEEZE \vspace{-7mm} \fi
	
	If the separation rank of a function \wrt~$(A,B)$ is~$1$, it is multiplicatively separable \wrt~$(A,B)$, meaning it cannot take into account consistency between $A$ and $B$.
	The higher $\sep{A,B}{y}$ is, the farther~$y$ is from this situation, \ie,~the more it models dependency between $A$ and $B$.
	We will further make use of the \textit{effective separation rank}:
	\ifdefined\SQUEEZE \vspace{-2mm} \fi
	\bea\label{eq:e_sep}
	\varepsilon\textrm{-}\sep{A,B}{y}:=\min\left\{R'\le\sep{A,B}{y}\in\N:\exists\tilde{y}:\X^{2M}\to\R~s.t.\right.
	\quad~\\
	\left.\left[\sep{A,B}{\tilde{y}}=R'\right]\wedge\left[\forall x\in\X^{2M}:\left|\tilde{y}\left(x\right)-y\left(x\right)\right|\leq\varepsilon\right]\right\} 
	\nonumber
	\eea
	\ifdefined\SQUEEZE \vspace{-6mm} \fi
	
	In words, if a function has a high separation rank, but it can be approximated up to error $\varepsilon$ by a function with a low separation rank, then it has a low $\varepsilon$-separation rank.

	Prior works compare two functions by establishing the differences between their \textit{separation ranks}. In principle, these differences could manifest only in irrelevant magnitudes
	(if many of the summands in the separation rank definition are negligibly small for the function with the higher separation rank, for example).
	The \textit{effective separation rank} is key to our analysis because we rely on the fact that information on past examples is stored in the network weights in a small magnitude (due to a small learning-rate).
	We show in section~\ref{sec:2:3} that much of the integration between text segments from different training examples occurs in very small magnitudes due to  high powers of the learning rate, limiting the effective integration, as measured by the $\varepsilon$-separation rank. 
	Our techniques for bounding the $\varepsilon$-separation rank are extendable to prior works, and while these did not examine the gradient update mechanism, their results can be reinforced due to the guarantees of this introduced measure.

	\ifdefined\SQUEEZE \vspace{-3mm} \fi
	\subsubsection{The sequential separation rank}\label{sec:2:2:2}
	\ifdefined\SQUEEZE \vspace{-1mm} \fi
	\citet{levine2020limits}, who were the first to apply the separation rank to functions realized by Transformer architectures, studied classical architecture expressivity questions which apply only to the in-context representation.
	Accordingly, they analyzed only the separation rank of $\g^{i,L,d_x}$, defined in eq.~\ref{eq:our_layer}, and
	the input variables considered for calculating the separation rank were the word embedding vectors.
	A fundamental difficulty arises when attempting to directly apply this method to the sequential representation: the word embedding vectors are learned parameters of the architecture. 
	In the sequential case, when the second sentence $S_2$ is introduced after the calculation at time-step $t$, the vectors used to describe it, if we were to follow prior practice, would already have depended on $S_1$. 
	
	In order to meaningfully measure the integration ability of two sentences across examples in the presence of the gradient update mechanism, 
	we 
	introduce 
	an auxiliary\textit{ sentence association layer} with new variables: $\aaa,\bb\in\R^{d_x}$, which explicitly associates each employed vocabulary embedding vector with the sentence index $s\in\{1,2\}$ that invoked its usage:
	\ifdefined\SQUEEZE \vspace{-1mm} \fi
	\begin{equation}\label{eq:sent_ass_layer}
	Z^s\left(E_{M^{\textrm{V}}}(w^j_s)\right) := \begin{cases}
	E_{M^{\textrm{V}}}(w^j_s)\odot \aaa \qquad\qquad\qquad  \textrm{if}~~ s=1\\
	E_{M^{\textrm{V}}}(w^j_s)\odot \bb \qquad\qquad\qquad \textrm{if}~~ s=2
	\end{cases}
	\end{equation}
	\ifdefined\SQUEEZE \vspace{-3mm} \fi
	
	where $\odot$ denotes element-wise multiplication.
	We define the sentence association operation over the analyzed representations, denoted
	$\mathcal{Z}_y\left(\aaa,\bb\right)$ with $y\in\{\y^{i,L,d_x}_{\textrm{in-context}}\left(S_1,S_2\right),\y^{i,L,d_x,\eta}_{\textrm{sequential}}\left(S_1,S_2\right)\}$ (eqs.~\ref{eq:in_context} or~\ref{eq:seq_1}), 
	to be the application of the sentence association layer of eq.~\ref{eq:sent_ass_layer} to all uses of the input embedding layer during the computation of $y$. 
	Meaning, for both mechanisms, that chosen word embeddings are marked with the identity of the sentence that invoked them.  
	Finally, we define the following specialization of the separation rank measure to our setting, referred to as the \textit{sequential separation rank} of $y\in\{\y^{i,L,d_x}_{\textrm{in-context}}\left(S_1,S_2\right),\y^{i,L,d_x,\eta}_{\textrm{sequential}}\left(S_1,S_2\right)\}$:
	\ifdefined\SQUEEZE \vspace{-2mm} \fi
	\begin{align}  \label{eq:seq_sep} \mathrm{seq}\textrm{-}&\mathrm{sep}(y):= \sep{\aaa,\bb}{\mathcal{Z}_y}\\\nonumber
	\varepsilon\textrm{-}\mathrm{seq}\textrm{-}&\mathrm{sep}(y):= \varepsilon\textrm{-}\sep{\aaa,\bb}{\mathcal{Z}_y}
	\end{align}
	\ifdefined\SQUEEZE \vspace{-5mm} \fi
	
	Clearly, when the introduced variables are vectors of $\mathbf{1}$, the auxiliary layer in eq. 8 is the identity operation and so $\mathcal{Z}_y(\mathbf{1},\mathbf{1})=y$ for both representations.
	More deeply, our expressivity questions query the ability of the in-context and sequential mechanisms to integrate two sets of variables, and $\mathcal{Z}_y$ captures the  essence of this ability by explicating where each set enters the computation. 
	
	In the next subsection, we show that for the in-context case, analyzed in prior work, the introduced measure of the sequential separation rank is asymptotically equal to the previously employed measure of separation \wrt~a partition of the input word embeddings~\citep{levine2020limits}. Thus, the properties of the existing framework are unchanged under the new definition.
	At the same time, for the sequential case brought forth in this paper, the sequential separation rank considers both the effect of $S_1$ on the gradient-updated word embedding and the introduction of $S_2$ into the computation.\footnote{To see this, note that for $s=2$ the operation in  eq.~\ref{eq:sent_ass_layer} includes both $M^{\textrm{V}}_{t+1}\left(\aaa\right)$ and variables from $\bb$.}
	In the following section, we make use of both  extensions to the separation rank in eqs.~\ref{eq:e_sep} and~\ref{eq:seq_sep}
	in order to establish the in-context bias.
	
	\ifdefined\SQUEEZE \vspace{-2mm} \fi
	\subsection{The expressive advantage of in-context learning }\label{sec:2:3}
	\ifdefined\SQUEEZE \vspace{-2mm} \fi
	We show below that the function computed by a self-attention based NLM when inserting sentences $S_1$ and $S_2$ together in its input (the {in-context representation})
	can model more elaborate dependencies between $S_1$ and $S_2$ than the function attained when showing $S_1$ in the input, modifying the network's weights according to its loss, and then showing $S_2$ in a subsequent input (the {sequential representation}).
	We begin by stating the following corollary, following from theorem~2 in \citet{levine2020limits} and proposition 1 in appendix \ref{corr_1_proof}, which upper bounds the sequential separation rank of the \textit{in-context} representation:
	\begin{corollary}\label{corr:1}
		Let $y^{(p,i),L,d_x}_\text{\emph{in-context}}$ be the $p\in[d_x]$ entry of the analyzed in-context representation defined in eq.~\ref{eq:in_context}. Assume that $L>\log_{3}d_x$. Then ($\tilde{O}$ notation omits log terms: $\log d_x,\log L,\log H$):
		\ifdefined\SQUEEZE \vspace{-3mm} \fi
		\begin{align}\label{eq:upper_bound_in_context}
		\log\left[\mathrm{seq}\textrm{-}\mathrm{sep}\left(y^{(p,i),L,d_x}_\text{\emph{in-context}}\right)\right]&=
		\tilde{O}\left(L\cdot d_x\right)
		\end{align}
		\ifdefined\SQUEEZE \vspace{-5mm} \fi
	\end{corollary}
	\ifdefined\SQUEEZE \vspace{-2mm} \fi
	
	However, the $\varepsilon$-separation rank of the \textit{sequential} representation is upper bounded by a lower term:
	\begin{theorem}\label{theorem:upper_bounds}
		(See proof in appendix \ref{upper_bound}). Let $y^{(p,i),L,d_x,\eta}_\text{\emph{sequential}}$ be the $p\in[d_x]$ entry of the analyzed sequential representation defined in eq.~\ref{eq:seq_1}.
		Assume that all learned parameters and all gradients are bounded: $\forall \theta\in\{\mathcal{W}, M^{\textrm{V}}\}:0<\Lambda_{\min}\le\abs{\theta},\abs{\nicefrac{\partial\mathcal{L}\left(S_1\right)}{\partial \theta}}\le\Lambda_{\max}$,\footnote{The upper boundedness assumption resembles practices of gradient clipping and weight decay, and the lower boundedness assumption resembles finite precision.}
		$N<d_x$, and that $L>\log_3d_x$.
		Then, $\forall\varepsilon>0$:
		\ifdefined\SQUEEZE \vspace{-2mm} \fi
		\begin{align}\label{eq:upper_bound_sequential}
		\log\left[
		\varepsilon\textrm{-}\mathrm{seq}\textrm{-}\mathrm{sep}
		\left(y^{(p,i),L,d_x,\eta}_\text{\emph{sequential}}\right)\right]&=
		\tilde{O}\left([L+0.5\log_3(\eta)]\cdot d_x\right)~~
		\end{align}	
		\ifdefined\SQUEEZE \vspace{-6mm} \fi
	\end{theorem}
	
	Therefore, a gap between upper bounds on the ability to model dependencies between $S_1$ and $S_2$ is indicated. 
	Since the learning rate $\eta$ is a small term, its log is negative and the gap is in favor of the in-context representation. 
	The following theorem guarantees that this gap is meaningful, by showing that the higher upper bound (of the in-context case) is tight in terms of effective rank:   
	
	\begin{theorem}\label{theorem:tight}
		(See proof in appendix \ref{lower_bound}). For $y^{(p,i),L,d_x}_\text{\emph{in-context}}$ as defined in corollary~\ref{corr:1}, there exists an assignment of the network weights for which the following holds:
		\ifdefined\SQUEEZE \vspace{-2mm} \fi
		\begin{equation}\label{eq:lower_bound}
		\log\left[\varepsilon\textrm{-}\mathrm{seq}\textrm{-}\mathrm{sep}
		\left(y^{(p,i),L,d_x}_\text{\emph{in-context}}\right)\right]=
		\tilde{\Omega}\left(L\cdot d_x\right)
		\end{equation}
		
		\ifdefined\SQUEEZE \vspace{-7mm} \fi
		where $\varepsilon=\tilde{O}\left(\binom{3^L+d_x-1}{3^L}^{-1}\right)$.
	\end{theorem}
	\ifdefined\SQUEEZE \vspace{-3mm} \fi
	Notably, corollary~\ref{corr:1} and theorem~\ref{theorem:tight} show that for the in-context case, the sequential separation rank asymptotically equals the regular separation rank, validating the relevance of this measure. 
	
		We now provide a \underline{high level proof sketch} that captures the manner in which the theoretical framework of sections 2.1 and 2.2 is used for establishing the above gap (full proof in the appendix).
	For the \textbf{in-context} case, notice that each self-attention layer, defined in eq.~\ref{eq:our_layer}, is a degree $3$ polynomial over its $2N\cdot d_x$ inputs, rendering the whole network a degree $3^L$ polynomial.
	We write this polynomial as a sum over many monomials, and by definition, the separation rank of any monomial composing the polynomial is $1$. Since the separation rank of a sum of functions is upper bounded by the sum of their separation ranks, we upper bound the separation rank by the number of these monomials, yielding eq.~\ref{eq:upper_bound_in_context}.
		The main difference in the \textbf{sequential representation} case is that the $S_1$ variables affect the computation only via the gradient, so their impact is expected to be limited.
		However, considering that $S_2$ first encounters gradient updated vocabulary matrix entries $\left(m^{\textrm{V}}\right)_{t+1}=\left(m^{\textrm{V}}\right)_{t}-\eta\nicefrac{\partial\mathcal{L}\left(S_1\right)}{\partial \left(m^{\textrm{V}}\right)_{t}}$, it appears that both $S_1$ and $S_2$ variables enter the self-attention stack via its input, similarly to the in-context case. So the integration between $S_1$ and $S_2$ occurs right from the start, and indeed we show that the \textit{separation rank} of both representations is similar. 
		However, since any function of $S_1$ is accompanied by the learning-rate $\eta$, the monomials for which there are many  $S_1$ variables will be multiplied by high powers  of $\eta$. This causes many monomials to be negligibly small, and accordingly not to contribute to the  \textit{$\varepsilon$-separation rank}. 
		By combinatorial considerations we show that the number of monomials that are not attenuated by $\eta$  (have sufficiently large magnitude) yields eq.~\ref{eq:upper_bound_sequential}. \qedsymbol
	
	The above theorems establish that from an expressivity perspective, the small magnitude of commonly employed learning-rates hinders the ability to integrate information across different training examples.
	Specifically, the established gap 
	implies that the power of the joint representation of two sentences shown in different training examples is upper bounded by that of a network shallower by $0.5\log_3(\eta^{-1})$ layers that has seen them in the same context.
	Common learning-rate values are on the order of $\eta\in[10^{-6},10^{-4}]$, implying a deficit of $\sim6$ layers in the sequential case.
	As shown in in~\cite{levine2020limits,tay2021scale}, in many practical regimes of network size depth is crucial for expressivity, reinforcing the implications of this gap. 
	
	The weaker upper bound, of the sequential case, is not guaranteed to be tight. 
	This means that theoretically, the sequential representation may in fact be much weaker than what we have proven, \eg, that showing two sentences in the same context yields a representation that cannot be matched merely by showing them in separate contexts and adding a realistic number of layers. 
	However, \cite{roberts2020much} show evidence supporting our indicated link between architectural parameters and the in-context bias. 
	They show that when performing open domain question answering tasks (their defined ``closed book" setting), a large T5 model that sees only the question performs comparably to smaller models that are allowed to attend to the documents that contain the answer.
	This directly implies a certain strength of the sequential mechanism, namely, that information which was seen during training can be accessed via the weights when the model is realistically stronger, as implied by our bounds. Notably, the large T5 model is 2-4 times the depth of the contrasted smaller models  ($48$ versus $12$-$24$ layers), suggesting that the upper bound can be tightened to a fraction of $L$, or that factors that are beyond expressivity also contribute to the in-context bias (\eg, optimization, generalization). Investigation of these aspects is left for future work.

	\ifdefined\SQUEEZE \vspace{-4mm} \fi
	\section{kNN based pretraining example design}\label{sec:3}
	\ifdefined\SQUEEZE \vspace{-3mm} \fi
	Our theoretical analysis quantifies the relation between the small magnitude of the learning rate, and the deficiency in the ability to model dependencies between different training examples. 
	Clearly, small learning-rates are critical for optimization purposes, so the formalized phenomenon should not be solved via high learning-rates during training. 
	Instead, our analysis makes it clear that if correlations between specific sentences are important for a given task, appending them in-context yields better representations for the task. 
	Below, we describe two controlled experiments that demonstrate the importance of this indicated ``pretraining example design" degree of freedom.
	In both experiments, correlated sentences are identified via kNN search in their RoBERTa-large sentence representation space~\citep{reimers2019sentence}, performed using the FAISS library~\citep{johnson2019billion}.

	\ifdefined\SQUEEZE \vspace{-3mm} \fi
	\subsection{kNN Task Adaptive PreTraining}\label{sec:3:1}
	\ifdefined\SQUEEZE \vspace{-3mm} \fi
	The Task Adaptive PreTraining (TAPT) method, in which an NLM pretrains on the training set of an NLU task, leads to impressive gains~\citep{gururangan2020don}. 
	Notably, TAPT is most effective after the regular pretraining stage on a general corpus. 
	This implies that during TAPT, the model generates improved representations by integrating the task related text with the knowledge stored in its weights from the preceding general pretraining phase.
	Under this premise, we postulated that performance will improve if we make relevant sentences from the general corpus more available to the model during the TAPT phase. 
	According to the above analysis, a simple and effective way to bias the model towards representing desired correlations between sentences is to append them in context.
	
	We thus propose the \textit{kNN-TAPT} phase, in which the training examples are composed of task examples, concatenated with their general corpus neighbors in embedding space. 
	We applied kNN-TAPT on the SentEval sentence similarity tasks. Showing similar sentences from Wikipedia is expected to be particularly useful on these tasks, so this is a good experimentation ground to search for effects of the in-context bias.
	For each SentEval example, we searched over ~$100$M Wikipedia sentences and appended in-context neighbors that have embeddings with over $0.8$ cosine similarity to the SentEval example embedding, with a special token inserted between different sentences.
	We continued until finding no more neighbors or reaching a maximum of $256$ tokens in the RoBERTa vocabulary~\citep{liu2019roberta}. 
	This search yielded $170$K examples, over which we continued training a pretrained RoBERTa-base model for $5$ epochs, using the first epoch for learning-rate warmup and examining peak learning rates of $\{1,3,5,7\}\cdot10^{-5}$. See appendix \ref{exp_details} for implementation details.
	
	Table~\ref{tab:1}, shows zero-shot SentEval sentence similarity scores, attained by using the average word embedding of an inserted sentence as its examined sentence representation (shown by~\cite{reimers2019sentence} to be most meaningful in zero shot).
	All models were trained according to the above prescription, besides the baseline RoBERTa which was simply evaluated.
	kNN-TAPT improves over regular TAPT, by over $1$ point on average, implying that the Wikipedia neighbors are indeed useful to the TAPT stage.
	We compared $4$ kNN-TAPT variants  as an ablations study. Importantly all variants labeled with kNN-TAPT  train on the same training data during the TAPT stage -- the SentEval sentence similarity tasks training sets and their Wikipedia nearest neighbors, and differ only in the arrangement of the data into training examples. 
	The ``neighbors" flag relates a SentEval example to its actual neighbors from the kNN search, while the ``random" flag relates it to random Wikipedia sentences from the overall neighbors pool attained in the search.  
	The ``in batch" flag implies that related sentences were shown in the same batch, where  every training example includes only one sentence from either SentEval or Wikipedia. 
	In contrast, the ``in context" flag implies that related sentences were shown in the same training example. 
	
	The weakness of ``neighbors, in-batch" implies that the a-priori plausible approach of biasing the model to learn from these Wikipedia neighbors via placing them in the same batch is not nearly as effective as the theoretically motivated in-context approach.
	Leading sentence representations employ in-batch techniques (see for example the contrasive setting of ~\cite{gao2021simcse}), and this signal strongly suggests developing in-context parallels.  
	The fact that the original TAPT scheme outperforms the in-batch approaches implies that including the Wikipedia sentences in separate training examples is harmful. We postulate that this is because training examples that have only Wikipedia sentences actually dilute the original TAPT signal. 
	Indeed, by this view, the reason that ``random, in-context" performs comparably to TAPT, is that it does not dilute the original TAPT signal -- every training example includes a SentEval example. 
	Overall, the clear advantage of the ``neighbors, in-context" kNN-TAPT variant encourages leveraging the in-context bias for TAPT in further tasks.

	\begin{table}
		\begin{center}
			\resizebox{\linewidth}{!}{
				\begin{tabular}{lcccccccc}
					\toprule
					& \textbf{STS12} & \textbf{STS13 }& \textbf{STS14} & \textbf{STS15} & \textbf{STS16} & \textbf{STS-B} & \textbf{SICK-R} & \textbf{Avg.}\\
					\toprule
					Basline Roberta Model & 32.1  & 56.3  & 45.2  & 61.3  & 62.0  & 55.4         & 62.0            & 53.5 \\
					TAPT                  & 43.0  & 62.2  & 51.6  &\textbf{ 70.6}  & 64.9  & 63.0         & 63.5            & 59.8 \\
					\midrule
					kNN-TAPT (random, in-batch)      & 40.2  & 62.7  & 51.9  & 64.9  & 62.1  & 61.5         & 65.4            & 58.4 \\
					kNN-TAPT (neighbors, in-batch)  & 40.8  & 62.4  & \textbf{53.1}  & 66.1  & 63.0  & 61.3         & 65.2            & 58.8 \\
					kNN-TAPT (random, in-context)  &44.62 &	62.64 &	51.4	&  65.28	& 64.93 &	64.31 &	66.96 &	60.0\\
					kNN-TAPT (neighbors, in-context)   & \textbf{44.9  }&\textbf{ 63.4}  & 52.1  & {66.2 } & \textbf{65.3 } &\textbf{ 66.5}         & \textbf{68.3}            & \textbf{61.0}	\\
					\bottomrule
				\end{tabular}
			}
			\vspace{-3mm}
			\caption{kNN-TAPT, which augments the Task Adaptive PreTraining (TAPT) setting of~\cite{gururangan2020don}, harnesses the in-context bias and improves SentEval sentence similarity scores.  \label{tab:1} }
			\vspace{-3.5mm}
		\end{center}
		\ifdefined\SQUEEZE \vspace{-5mm} \fi
	\end{table}

	\ifdefined\SQUEEZE \vspace{-5mm} \fi
	\subsection{kNN Pretraining}\label{sec:3:2}
	\ifdefined\SQUEEZE \vspace{-3mm} \fi

	We extended the above to more general \textit{kNN-Pretraining}, designing pretraining examples with related non-neighboring sentences given only the general pretraining corpus.
	kNN-Pretraining is also motivated by the kNN-LM results of
	~\cite{khandelwal2019generalization}, who show significant benefits of using nearest neighbors in representation space at \textit{inference} time. Their results exemplify the potential impact of integrating cross-corpus related examples; our kNN-Pretraining approach provably biases the model to learn these correlations at \textit{pretraining time},
	via the in-context bias.
	
	Specifically, we performed kNN search over Wikipedia sentences for every sentence in Wikipedia, and created each training example similarly to the protocol in the previous subsection.    
	During kNN-Pretraining, half of the batch contained regular pretraining examples and half contained the prepared kNN examples, in order to retain longer ranged LM abilities.
	To examine the effect of kNN-Pretraining, we pretrained GPT-base and GPT-medium ($110$M and $345$M parameters) architectures from scratch over Wikipedia in the regular pretraining scheme, and switched to kNN-Pretraining at two different points during pretraining ($200$K and $400$K). The training examples were of maximal size $256$, and the batch size was $128$ for the GPT-medium models and 256 for the GPT-base models.
	
	In order to directly probe the acquired ability to integrate non-neighboring sentences, we evaluated the resultant models on the very challenging setup of zero-shot closed-book open domain question answering.  In this setup, the unidirectional pretrained model decodes an answer conditioned on the given open ended question. We evaluated the models on questions from the Natural Questions (NQ) benchmark~\citep{kwiatkowski2019natural}, using the same phrasing employed in~\cite{GPT3}, and employing the standard ``open-domain”
	version as used e.g. by \cite{lee2019latent,asai2019learning,roberts2020much}.
	NQ is composed of questions that have answers within Wikipedia, our pretraining corpus. kNN pretraining can imtuitively improve in cases where the passage containing the answer has elucidating nearest neighbors from across wikipedia that would help the model to better internalize the answer, such that it is more accessible to the model in zero shot.
	As figure~\ref{fig:fig1} demonstrates, $3$ baseline models,  pretrained with the regular scheme, achieve very low F1$<10^{-3}$ scores on this task. In contrast, kNN-Pretraining shows a low-scoring but significant improvement.
	
	To increase the credibility of the signal,
	we evaluated our models on the first $20$K examples from the NQ training set (we tested zero-shot performance, so the training set was not used earlier).
	Indeed, the attained F1 scores are low, but they correspond to $100$s of correct answers that the kNN-Pretrained model provide after roughly $10$\% of the overall training time, versus much less in the $3$ randomly initialized baseline models. 
	Finally, we include in appendix \ref{knn_pret_results} NQ scores of models of different sizes when starting kNN-Pretraining at different checkpoints, and in appendix F zero-shot scores on several GLUE tasks, which demonstrate clear gains of kNN-Pretraining over the baselines.  
	
	\ifdefined\SQUEEZE \vspace{-6mm} \fi
	\section{Discussion}\label{sec:6}
	\ifdefined\SQUEEZE \vspace{-4mm} \fi
	Modern NLM pretraining schemes have tremendously advanced the natural language landscape, since they allowed powerful models to train on huge amounts of unlabeled text. 
	But NLMs are now challenged with tasks which require deeper and more nuanced understanding of text, and means of improving the basic pretraining process should be considered.
	For a given architecture, pretraining can be improved by adding more data or finding more sophisticated training objectives to apply over existing data. 
	In this paper we highlight a parallel path for improvement, which employs the same data and objective, but redistributes the available strength of the Transformer architecture such that important connections within the pretraining corpus are learned more effectively.
	Specifically, we highlight the bias of the trained NLM towards modeling dependencies between chunks of text that appeared within the same training examples. In current pretraining schemes, this means that dependencies between non-neighboring chunks of text are under-favored. If such dependencies matter for the task at hand, we suggest rearranging the data into corresponding training examples. 
	
	We formalize the above notion.	
	Our theoretical setup asks expressivity questions that pertain to the training set rather than to a single example. 
	We thus tie the
	construction of the training example with the available expressivity of the architecture: we prove that the connections that can be modeled between different training examples are bounded by the connections that can learned by a shallower and weaker architecture, if these examples were inserted within the same input.
	
	The advantage in including related text in the input of the NLM is noticed and leveraged in the empirical landscape.
	With that, it is clear that showing the model related data is meaningful even if it is in different training examples, and many leading methods elect to do just that.
	Our quantification of this trade-off is intended to aid informed decisions and highlight the expressivity advantage to be gained by smarter training example designs. 
	We follow up on these recommendations and demonstrate the immediately available gains to be achieved by designing training examples that include nearest neighbors in embedding space.
	This method can be enhanced, and other more explicit biases can be introduced. For example, multiple mentions of the same entity, event, or concept can be concatenated within the same training example.
	
	The gains achieved by using similarity in representation space indicate a path for self-improving representations, left for future work. After a first cycle of kNN-Pretraining, the representation is refined and applying a new kNN search over it can lead to more informative next round of kNN-Pretraining. This way, deeper insight can be elicited from a given pretraining corpus. 
	
	Lastly, while this paper focused on leveraging the identified in-context bias for pretraining, it can also be tied to recent successes of in-context inference methods.  
	From the in-context few-shot prompts of \citet{GPT3}, to in context augmentations such as in \cite{gao2020making,schick2020s} and many others, the benefits of biasing the prediction by appending text in-context are now widely established.
	The tools brought forth here can assist in clarifying the theoretical advantages of such practices. 
	Overall, our work aims
	to provide timely theoretical interpretations, to help guide
	the rapid empirical advances of our field.

	\bibliography{refs}
	\bibliographystyle{iclr2022_conference}
	\appendix
	\section{Proof of Corollary \ref{corr:1}}\label{corr_1_proof}
	\begin{proposition}
		Let $S_{1}=\left\{ w_{1}^{j}\right\} _{j=1}^{N}$ and $S_{2}=\left\{ w_{2}^{j}\right\} _{j=1}^{N}$
		be two sentences, $\mathbf{g}{}_{\mathcal{W}}^{i,L,d_{x}}$ the Transformer
		operation of $\mathbf{y}_{\textrm{in-context}}^{i,L,d_{x}}\left(S_{1},S_{2}\right)$
		and $M^{\textrm{V}}$ be a vocabulary embedding matrix. Then:
		\[
		\text{sep-seq}\left(\mathbf{y}_{\textrm{in-context}}^{i,L,d_{x}}\left(S_{1},S_{2}\right)\right)\leq\text{sep}_{\left(\left[N\right],\left[2N\right]\backslash\left[N\right]\right)}\left(\mathbf{g}_{\mathcal{W}}^{i,L,d_{x}}\right)
		\]
	\end{proposition}
	\begin{proof}
		Assume that $\text{sep}_{\left(\left[N\right],\left[2N\right]\backslash\left[N\right]\right)}\left(\mathbf{g}_{\mathcal{W}}^{i,L,d_{x}}\right)=R$,
		then by definition there exist $g_{1},\ldots,g_{R}:\left(\mathbb{R}^{d_{x}}\right)^{N}\to\mathbb{R}$
		and $g'_{1},\ldots,g'_{R}:\left(\mathbb{R}^{d_{x}}\right)^{N}\to\mathbb{R}$
		such that for any $\left\{ \mathbf{x}^{j}\right\} _{j=1}^{2N}$, 
		\[
		\mathbf{g}_{\mathcal{W}}^{i,L,d_{x}}\left(\mathbf{x}^{1},\ldots,\mathbf{x}^{2N}\right)=\sum_{r=1}^{R}g_{r}\left(\mathbf{x}^{1},\ldots,\mathbf{x}^{N}\right)g'_{r}\left(\mathbf{x}^{N+1},\ldots,\mathbf{x}^{2N}\right)
		\]
		
		Now, given $\mathbf{a},\mathbf{b}\in\mathbb{R}^{d_{x}}$, we can write:
		\begin{align*}
		\mathcal{Z}_{\mathbf{y}_{\textrm{in-context}}^{i,L,d_{x}}\left(S_{1},S_{2}\right)}\left(\mathbf{a},\mathbf{b}\right) & =\sum_{r=1}^{R}g_{r}\left(E_{M^{\text{V}}}\left(w_{1}^{1}\right)\odot\mathbf{a},\ldots,E_{M^{\text{V}}}\left(w_{1}^{N}\right)\odot\mathbf{a}\right)\\
		& \cdot g'_{r}\left(E_{M^{\text{V}}}\left(w_{2}^{1}\right)\odot\mathbf{b},\ldots,E_{M^{\text{V}}}\left(w_{2}^{N}\right)\odot\mathbf{b}\right)
		\end{align*}
		Clearly, this form of presenting $\mathcal{Z}_{\mathbf{y}_{\textrm{in-context}}^{i,L,d_{x}}\left(S_{1},S_{2}\right)}\left(\mathbf{a},\mathbf{b}\right)$
		is separable with respect to $\left(\mathbf{a},\mathbf{b}\right)$,
		and since it has $R$ summands, we can conclude that:
		\[
		\text{sep-seq}\left(\mathbf{y}_{\textrm{in-context}}^{i,L,d_{x}}\left(S_{1},S_{2}\right)\right)=\text{sep}_{\left(\mathbf{a},\mathbf{b}\right)}\left(\mathcal{Z}_{\mathbf{y}_{\textrm{in-context}}^{i,L,d_{x}}\left(S_{1},S_{2}\right)}\right)\leq R
		\]
	\end{proof}
	Corollary 1 now follows from an upper bound on $\text{sep}_{\left(\left[N\right],\left[2N\right]\backslash\left[N\right]\right)}\left(\mathbf{g}_{\mathcal{W}}^{i,L,d_{x}}\right)$ given in \cite{levine2020limits}. 
	\section{Upper bound for the $\varepsilon$-separation rank}\label{upper_bound}
	\begin{definition}
	For an expression that can be represented as a sum of some terms,
	\[
	f=\sum_{n=1}^{N}a_n
	\]
	denote by $f^{+}$ the corresponding sum, but with each term replaced by its absolute value, that is:
	\[
	f^{+}\coloneqq\sum_{n=1}^{N}\left|a_n\right|
	\]
	and note that by the triangle inequality it holds that:
	\[
	\left|f\right| \leq f^{+}
	\]
\end{definition}
\begin{theorem}
	\label{thm:sequential}Let $y_{\text{\emph{sequential}}}^{(p,i),L,d_{x},\eta}$
	be be the $p\in[d_{x}]$ entry of the analyzed sequential representation
	defined in eq.~\ref{eq:seq_1}. Assume that all learned parameters
	and all gradients are bounded: 
	$\forall\theta\in\{\mathcal{W},M^{\textrm{V}}\}:\Lambda_{\min}\le\,\left|\theta\right|,
	\, \left|\nicefrac{\partial\mathcal{L}\left(S_{1}\right)}{\partial\theta}\right|
	\le\Lambda_{\max}\, $
	for some $0<\Lambda_{\min}\leq\Lambda_{\max}$, $N<d_{x}$, $\eta\in\left(0,1\right]$,
	$\frac{2\left(1+\eta\right)d_{x}}{\eta}<3^{L}$, $2\left(1+\eta\right)d_{x}^{2}<3^{L}$, 
	In addition, assume that there exists $M \geq 0$ for which it holds that $\mathcal{Z}_{y_{\text{sequential}}^{p,i,H,L,d_{x},\eta}\left(S_{1},S_{2}\right)}^{+} < M$ 
	on its domain. Then: 
	\[
	\log\left[\varepsilon\textrm{-}\mathrm{seq}\textrm{-}\mathrm{sep}\left(y_{\text{\emph{sequential}}}^{(p,i),L,d_{x},\eta}\right)\right]=\tilde{\mathcal{O}}\left(\left[L+0.5\log_{3}\left(\eta\right)\right]\cdot d_{x}\right)
	\]
\end{theorem}
\begin{proof}
	Denote: 
	\[
	\mathcal{Z}_{\Theta_{t+1}\left(\mathbf{a},S_{1};\eta\right)}^{\left(S_{2}\right)}\left(\mathbf{b}\right)\coloneqq\mathcal{Z}_{y_{\text{sequential}}^{p,i,H,L,d_{x},\eta}\left(S_{1},S_{2}\right)}\left(\mathbf{a},\mathbf{b}\right)
	\]
	The proof outline is as follows:
	
	We start by finding a representation of $\mathcal{Z}_{\Theta_{t+1}\left(\mathbf{a},S_{1};\eta\right)}^{\left(S_{2}\right)}$
	as a sum of terms, where each term is separable with respect to $\left(\mathbf{a},\mathbf{b}\right)$.
	We then turn to finding a subset of these terms, denoted $G$, such
	that the sum of all terms in $G$ is an $\varepsilon$-approximation
	of $\mathcal{Z}_{\Theta_{t+1}\left(\mathbf{a},S_{1};\eta\right)}^{\left(S_{2}\right)}$.
	Lastly, since it follows from the definition of the $\varepsilon$-separation
	rank and the construction of $G$ that $\varepsilon\text{-sep}_{\left(\mathbf{a},\mathbf{b}\right)}\left(\mathcal{Z}_{y_{\text{sequential}}^{p,i,H,L,d_{x},\eta}\left(S_{1},S_{2}\right)}\right)$
	is upper bounded by the cardinality of $G$ (which is the number of
	summands in the approximation), we find an upper bound to $\left|G\right|$,
	which is therefore an upper bound to $\varepsilon\text{-sep}_{\left(\mathbf{a},\mathbf{b}\right)}\left(\mathcal{Z}_{y_{\text{sequential}}^{p,i,H,L,d_{x},\eta}\left(S_{1},S_{2}\right)}\right)$
	as well, which by definition is equal to $\varepsilon\text{-seq-sep}\left(y_{\text{sequential}}^{p,i,H,L,d_{x},\eta}\right)$.
	
	\subsubsection*{Step 1 --- A separable representation of $\mathcal{Z}_{y_{\text{sequential}}^{p,i,H,L,d_{x},\eta}\left(S_{1},S_{2}\right)}$}
	
	Following \cite{levine2020limits,wies2021vocabulary},
	$\mathcal{Z}_{\Theta_{t+1}\left(\mathbf{a},S_{1};\eta\right)}^{\left(S_{2}\right)}\left(\mathbf{b}\right)$
	can be written as:
	\begin{align}
	& \mathcal{Z}_{\Theta_{t+1}\left(\mathbf{a},S_{1};\eta\right)}^{\left(S_{2}\right)}\left(\mathbf{b}\right)\nonumber \\
	& \begin{alignedat}{1} & =\sum_{j_{1},\dots,j_{C\left(L\right)}=1}^{N}\sum_{h\in\left[H\right]^{\left[C\left(L\right)\right]}}\sum_{r_{1},\dots,r_{C\left(L\right)+1}=1}^{d_{a}}Q_{r_{1},p}^{\left(0,h\right)}\\
	& \cdot\left(\prod_{c=1}^{C\left(L\right)+1}\left\langle P_{r_{c}}^{\left(c,h\right)},\tilde{w^{j_{c}}}\right\rangle \right)\left(\prod_{c=1}^{C\left(L\right)}\left\langle Q_{r_{c+1}}^{\left(c,h\right)},\tilde{w^{j_{c}}}\right\rangle \right)
	\end{alignedat}
	\label{eq:sequential}
	\end{align}
	where $\tilde{w^{j}}\coloneqq g\left(\mathbf{b}\right)^{j}+\eta f\left(\mathbf{a}\right)^{j}\odot\mathbf{b}$
	~~(this form follows from $g\left(\mathbf{b}\right)^{j}\coloneqq E_{M_{t}^{\text{V}}}\left(w_{2}^{j}\right)\odot\mathbf{b}$
	is the entry-wise product of $\mathbf{b}$ with the embedding of $w_{2}^{j}$ prior to
	the $t$th training step, and $f\left(\mathbf{a}\right)^{j}\coloneqq-\frac{\partial\mathcal{L}\left(S_{1};\mathbf{a}\right)}{\partial\left(M_{t}^{\text{V}}\right)_{w_{2}^{j}}}$
	is the gradient update performed to $w_{2}^{j}$'s embedding at time
	$t$), the $P_{r_{c}}^{\left(c,h\right)}$ and $Q_{r_{c+1}}^{\left(c,h\right)}$
	terms are sums of products of the networks inner (i.e., non-embedding)
	weights which were also updated with respect to $\mathcal{L}\left(S_{1};\mathbf{a}\right)$,
	and for convenience we denote $j_{C\left(L\right)+1}\coloneqq i$
	and $P^{\left(C\left(L\right)+1,h\right)}\coloneqq P^{\left(0,h\right)}$.
	\begin{align*}
	& =\sum_{j_{1},\dots,j_{C\left(L\right)}=1}^{N}\sum_{h\in\left[H\right]^{\left[C\left(L\right)\right]}}\sum_{r_{1},\dots,r_{C\left(L\right)+1}=1}^{d_{a}}Q_{r_{1},p}^{\left(0,h\right)}\\
	& \cdot\left(\prod_{c=1}^{C\left(L\right)+1}\left\langle P_{r_{c}}^{\left(c,h\right)},g\left(\mathbf{b}\right)^{j_{c}}+\eta f\left(\mathbf{a}\right)^{j_{c}}\odot\mathbf{b}\right\rangle \right)\\
	& \cdot\left(\prod_{c=1}^{C\left(L\right)}\left\langle Q_{r_{c+1}}^{\left(c,h\right)},g\left(\mathbf{b}\right)^{j_{c}}+\eta f\left(\mathbf{a}\right)^{j_{c}}\odot\mathbf{b}\right\rangle \right)
	\end{align*}
	Separating to vocab-gradient terms and vocab terms:
	\begin{align*}
	& =\underbrace{\sum_{\substack{\substack{I_{P}\subseteq\left[C\left(L\right)+1\right]\\
					\substack{I_{Q}\subseteq\left[C\left(L\right)\right]}
				}
			}
	}}_{\text{Indices involving both }\mathbf{a}\text{ and }\mathbf{b}}\sum_{j_{1},\dots,j_{C\left(L\right)}=1}^{N}\sum_{h\in\left[H\right]^{\left[C\left(L\right)\right]}}\sum_{r_{1},\dots,r_{C\left(L\right)+1}=1}^{d_{a}}Q_{r_{1},p}^{\left(0,h\right)}\\
	& \cdot\underbrace{\left(\prod_{c\in I_{P}}\left\langle P_{r_{c}}^{\left(c,h\right)},\eta f\left(\mathbf{a}\right)^{j_{c}}\odot\mathbf{b}\right\rangle \right)\left(\prod_{c\in I_{Q}}\left\langle Q_{r_{c+1}}^{\left(c,h\right)},\eta f\left(\mathbf{a}\right)^{j_{c}}\odot\mathbf{b}\right\rangle \right)}_{\text{Terms involving both }\mathbf{a}\text{ and }\mathbf{b}}\\
	& \underbrace{\left(\prod_{c\in\left[C\left(L\right)+1\right]\setminus I_{P}}\left\langle P_{r_{c}}^{\left(c,h\right)},g\left(\mathbf{b}\right)^{j_{c}}\right\rangle \right)\left(\prod_{c\in\left[C\left(L\right)\right]\setminus I_{Q}}\left\langle Q_{r_{c+1}}^{\left(c,h\right)},g\left(\mathbf{b}\right)^{j_{c}}\right\rangle \right)}_{\text{Terms involving just }\mathbf{b}}
	\end{align*}
	
	Opening to indices:
	\begin{align*}
	& =\sum_{\substack{\substack{I_{P}\subseteq\left[C\left(L\right)+1\right]\\
				\substack{I_{Q}\subseteq\left[C\left(L\right)\right]}
			}
		}
	}\sum_{j_{1},\dots,j_{C\left(L\right)}=1}^{N}\sum_{h\in\left[H\right]^{\left[C\left(L\right)\right]}}\sum_{r_{1},\dots,r_{C\left(L\right)+1}=1}^{d_{a}}Q_{r_{1},p}^{\left(0,h\right)}\\
	& \sum_{\substack{\alpha_{1},\dots,\alpha_{C\left(L\right)+1}\\
			\beta_{1},\dots,\beta_{C\left(L\right)}
		}
		=1}^{d_{x}}\left(\prod_{c\in I_{P}}P_{r_{c},\alpha_{c}}^{\left(c,h\right)}\eta f\left(\mathbf{a}\right)_{\alpha_{c}}^{j_{c}}\mathbf{b}_{\alpha_{c}}\right)\left(\prod_{c\in I_{Q}}Q_{r_{c+1},\beta_{c}}^{\left(c,h\right)}\eta f\left(\mathbf{a}\right)_{\beta_{c}}^{j_{c}}\mathbf{b}_{\beta_{c}}\right)\\
	& \left(\prod_{c\in\left[C\left(L\right)+1\right]\setminus I_{P}}P_{r_{c},\alpha_{c}}^{\left(c,h\right)}g\left(\mathbf{b}\right)_{\alpha_{c}}^{j_{c}}\right)\left(\prod_{c\in\left[C\left(L\right)\right]\setminus I_{Q}}Q_{r_{c+1},\beta_{c}}^{\left(c,h\right)}g\left(\mathbf{b}\right)_{\beta_{c}}^{j_{c}}\right)
	\end{align*}
	Separating to weights and variables:
	\begin{align*}
	& =\sum_{\substack{\alpha_{1},\dots,\alpha_{C\left(L\right)+1}\\
			\beta_{1},\dots,\beta_{C\left(L\right)}
		}
		=1}^{d_{x}}\tau_{\alpha_{1},\dots,\beta_{C\left(L\right)}}\sum_{\substack{\substack{I_{P}\subseteq\left[C\left(L\right)+1\right]\\
				\substack{I_{Q}\subseteq\left[C\left(L\right)\right]}
			}
		}
	}\sum_{j_{1},\dots,j_{C\left(L\right)}=1}^{N}\\
	& \left(\prod_{c\in I_{P}}\eta f\left(\mathbf{a}\right)_{\alpha_{c}}^{j_{c}}\mathbf{b}_{\alpha_{c}}\right)\left(\prod_{c\in I_{Q}}\eta f\left(\mathbf{a}\right)_{\beta_{c}}^{j_{c}}\mathbf{b}_{\beta_{c}}\right)\\
	& \cdot\left(\prod_{c\in\left[C\left(L\right)+1\right]\setminus I_{P}}g\left(\mathbf{b}\right)_{\alpha_{c}}^{j_{c}}\right)\left(\prod_{c\in\left[C\left(L\right)\right]\setminus I_{Q}}g\left(\mathbf{b}\right)_{\beta_{c}}^{j_{c}}\right)
	\end{align*}
	where:
	\begin{align*}
	\tau_{\alpha_{1},\dots,\beta_{C\left(L\right)}} & \coloneqq\sum_{h\in\left[H\right]^{\left[C\left(L\right)\right]}}\sum_{r_{1},\dots,r_{C\left(L\right)+1}=1}^{d_{a}}Q_{r_{1},p}^{\left(0,h\right)}\\
	& \cdot\left[\left(\prod_{c=1}^{C\left(L\right)+1}P_{r_{c},\alpha_{c}}^{\left(c,h\right)}\right)\left(\prod_{c=1}^{C\left(L\right)}Q_{r_{c+1},\beta_{c}}^{\left(c,h\right)}\right)\right]
	\end{align*}
	Compressing summation to count variable powers:
	\begin{align*}
	& =\underbrace{\sum_{N_{\text{A}}=0}^{2C\left(L\right)+1}}_{\text{Total power of }f\left(\mathbf{a}\right)}\underbrace{\sum_{\substack{p_{1}+\dots+p_{d_{x}}=N_{\text{A}}\\
				n_{1}+\dots+n_{d_{x}}=2C\left(L\right)+1-N_{\text{A}}
			}
	}}_{\substack{\text{How many indices}\\
			\text{are equal to each }\alpha\in\left[d_{x}\right]
		}
	}\\
	& \underbrace{\sum_{\substack{z_{1}+\dots+z_{N}=N_{\text{A}}\\
				m_{1}+\dots+m_{N}=2C\left(L\right)+1-N_{\text{A}}\\
				\forall j\in\left[N\right]\backslash\left\{ i\right\} :\hspace{1em}z_{j}+m_{j}\equiv0\mod2\\
				z_{i}+m_{i}\equiv1\mod2
			}
	}}_{\substack{\text{How many indices}\\
			\text{are equal to each }j\in\left[N\right]
		}
	}\underbrace{\sum_{\substack{0\leq p_{1,1},\dots,p_{d_{x},N}\leq N_{\text{A}}\\
				\forall\alpha\in\left[d_{x}\right]\,\sum_{j=1}^{N}p_{\alpha,j}=p_{\alpha}\\
				\forall j\in\left[N\right]\,\sum_{\alpha=1}^{d_{x}}p_{\alpha,j}=z_{j}
			}
	}}_{\text{How to distribute the powers of }f\left(\mathbf{a}\right)}\\
	& \underbrace{\sum_{\substack{0\leq n_{1,1},\dots,n_{d_{x},N}\leq2C\left(L\right)+1-N_{\text{A}}\\
				\forall\alpha\in\left[d_{x}\right]\,\sum_{j=1}^{N}n_{\alpha,j}=n_{\alpha}\\
				\forall j\in\left[N\right]\,\sum_{\alpha=1}^{d_{x}}n_{\alpha,j}=m_{j}
			}
	}}_{\text{How to distribute the powers of }g\left(\mathbf{b}\right)}\lambda_{N_{\text{A}},\boldsymbol{p},\boldsymbol{n}}\\
	& \cdot\left(\prod_{j=1}^{N}\prod_{\alpha=1}^{d_{x}}\left(\eta f\left(\mathbf{a}\right)_{\alpha}^{j}\mathbf{b}_{\alpha}\right)^{p_{\alpha,j}}\right)\left(\prod_{j=1}^{N}\prod_{\alpha=1}^{d_{x}}\left(g\left(\mathbf{b}\right)_{\alpha}^{j}\right)^{n_{\alpha,j}}\right)
	\end{align*}
	where: 
	\begin{align*}
	\lambda_{N_{\text{A}},\boldsymbol{p},\boldsymbol{n}} & \coloneqq\sum_{\substack{\substack{I_{P}\subseteq\left[C\left(L\right)+1\right]\\
				\substack{I_{Q}\subseteq\left[C\left(L\right)\right]}
				\\
				\left|I_{P}\right|+\left|I_{Q}\right|=N_{\text{A}}
			}
		}
	}\\
	& \sum_{\substack{\substack{\alpha_{1},\dots,\alpha_{C\left(L\right)+1}\\
				\beta_{1},\dots,\beta_{C\left(L\right)}
			}
			=1\\
			\forall\kappa\in\left[d_{x}\right]\,\left|\left\{ c\in I_{P}\left|\alpha_{c}=\kappa\right.\right\} \right|+\left|\left\{ c\in I_{Q}\left|\beta_{c}=\kappa\right.\right\} \right|=p_{\kappa}\\
			\kappa\in\left[d_{x}\right]\,\left|\left\{ c\in\left[C\left(L\right)+1\right]\setminus I_{P}\left|\alpha_{c}=\kappa\right.\right\} \right|+\left|\left\{ c\in\left[C\left(L\right)\right]\setminus I_{Q}\left|\beta_{c}=\kappa\right.\right\} \right|=n_{\kappa}
		}
	}^{d_{x}}\tau_{\alpha_{1},\dots,\beta_{C\left(L\right)}}
	\end{align*}
	Pushing in summations on $N$, only the parity matters:
	\begin{align*}
	& =\sum_{N_{\text{A}}=0}^{2C\left(L\right)+1}\eta^{N_{\text{A}}}\sum_{\substack{p_{1}+\dots+p_{d_{x}}=N_{\text{A}}\\
			n_{1}+\dots+n_{d_{x}}=2C\left(L\right)+1-N_{\text{A}}
		}
	}\sum_{\substack{\boldsymbol{e}\in\left\{ 0,1\right\} ^{N}\\
			\sum_{j=1}^{N}e_{j}\equiv N_{\text{A}}\mod2\\
			\sum_{j=1}^{N}e_{j}\leq\min\left\{ N_{\text{A}},2C\left(L\right)-N_{\text{A}}+2e_{i}\right\} 
		}
	}\\
	& \underbrace{\lambda_{N_{\text{A}},\boldsymbol{p},\boldsymbol{n}}}_{\text{Network's weights, function of }\mathbf{a}}\\
	& \cdot\underbrace{\left(\sum_{\substack{0\leq p_{1,1},\dots,p_{d_{x},N}\leq N_{\text{A}}\\
				z_{1}+\dots+z_{N}=N_{\text{A}}\\
				\forall\alpha\in\left[d_{x}\right]\,\sum_{j=1}^{N}p_{\alpha,j}=p_{\alpha}\\
				\forall j\in\left[N\right]\,\sum_{\alpha=1}^{d_{x}}p_{\alpha,j}=z_{j}\\
				\forall j\in\left[N\right]\,z_{j}\equiv e_{j}\mod2
			}
		}\prod_{j=1}^{N}\prod_{\alpha=1}^{d_{x}}\left(f\left(\mathbf{a}\right)_{\alpha}^{j}\right)^{p_{\alpha,j}}\right)}_{=:\phi_{A,\boldsymbol{p},\boldsymbol{e}}\text{, function of }\mathbf{a}}\\
	& \cdot\underbrace{\left(\prod_{\alpha=1}^{d_{x}}\left(\mathbf{b}_{\alpha}\right)^{p_{\alpha}}\right)\left(\sum_{\substack{0\leq n_{1,1},\dots,n_{d_{x},N}\leq2C\left(L\right)+1-N_{\text{A}}\\
				m_{1}+\dots+m_{N}=2C\left(L\right)+1-N_{\text{A}}\\
				\forall\alpha\in\left[d_{x}\right]\,\sum_{j=1}^{N}n_{\alpha,j}=n_{\alpha}\\
				\forall j\in\left[N\right]\,\sum_{\alpha=1}^{d_{x}}n_{\alpha,j}=m_{j}\\
				\forall j\in\left[N\right]\backslash\left\{ i\right\} \,m_{j}\equiv e_{j}\mod2\\
				m_{i}\equiv\left(1-e_{i}\right)\mod2
			}
		}\prod_{j=1}^{N}\prod_{\alpha=1}^{d_{x}}\left(g\left(\mathbf{b}\right)_{\alpha}^{j}\right)^{n_{\alpha,j}}\right)}_{=:\Psi_{B,\boldsymbol{p},\boldsymbol{n},\boldsymbol{e}}\text{, function of }\mathbf{b}}\\
	& =\sum_{N_{\text{A}}=0}^{2C\left(L\right)+1}\eta^{N_{\text{A}}}\sum_{\substack{p_{1}+\dots+p_{d_{x}}=N_{\text{A}}\\
			n_{1}+\dots+n_{d_{x}}=2C\left(L\right)+1-N_{\text{A}}
		}
	}\\
	& \sum_{\substack{\boldsymbol{e}\in\left\{ 0,1\right\} ^{N}\\
			\sum_{j=1}^{N}e_{j}\equiv N_{\text{A}}\mod2\\
			\sum_{j=1}^{N}e_{j}\leq\min\left\{ N_{\text{A}},2C\left(L\right)-N_{\text{A}}+2e_{i}\right\} 
		}
	}\lambda_{N_{\text{A}},\boldsymbol{p},\boldsymbol{n}}\cdot\phi_{A,\boldsymbol{p},\boldsymbol{e}}\cdot\Psi_{B,\boldsymbol{p},\boldsymbol{n},\boldsymbol{e}}
	\end{align*}
	where each summand is separable with respect to $\left(\mathbf{a},\mathbf{b}\right)$.
	
	\subsubsection*{Step 2 --- An $\varepsilon$-approximation of $\mathcal{Z}_{y_{\text{sequential}}^{p,i,H,L,d_{x},\eta}\left(S_{1},S_{2}\right)}$}
	
	Now that we have a representation of $\mathcal{Z}_{y_{\text{sequential}}^{p,i,H,L,d_{x},\eta}\left(S_{1},S_{2}\right)}$
	as a sum of $\left(\mathbf{a},\mathbf{b}\right)$-separable terms,
	we turn to finding a subsum that can approximate $\mathcal{Z}_{y_{\text{sequential}}^{p,i,H,L,d_{x},\eta}\left(S_{1},S_{2}\right)}$
	up to an $\varepsilon$-precision.
	
	First, let us define the set of all legal indices in the last sum: 
	\[
	D=D_{2C\left(L\right)+1,d_{x},N}\coloneqq\left\{ {\scriptstyle \left(N_{\text{A}},\boldsymbol{p},\boldsymbol{n},\boldsymbol{e}\right)\in\mathbb{N}^{2d_{x}+N+1}}\left|\substack{0\leq N_{\text{A}}\leq2C\left(L\right)+1\\
		\sum_{\alpha=1}^{d_{x}}p_{\alpha}=N_{\text{A}}\\
		\sum_{\alpha=1}^{d_{x}}n_{\alpha}=2C\left(L\right)+1-N_{\text{A}}\\
		\forall j\in\left[N\right],\,\,e_{j}\in\left\{ 0,1\right\} \\
		\sum_{j=1}^{N}e_{j}\equiv N_{\text{A}}\mod2\\
		\sum_{j=1}^{N}e_{j}\leq\min\left\{ N_{\text{A}},2C\left(L\right)-N_{\text{A}}+2e_{i}\right\} 
	}
	\right.\right\} 
	\]
	
	And for $G\subseteq D$, denote:
	\[
	\mathcal{Z}_{G}\left(\mathbf{a},\mathbf{b}\right)\coloneqq\sum_{\left(N_{\text{A}},\boldsymbol{p},\boldsymbol{n},\boldsymbol{e}\right)\in G}\eta^{N_{\text{A}}}\lambda_{N_{\text{A}},\boldsymbol{p},\boldsymbol{n}}\cdot\phi_{A,\boldsymbol{p},\boldsymbol{e}}\cdot\Psi_{B,\boldsymbol{p},\boldsymbol{n},\boldsymbol{e}}
	\]
	which is the sum of all terms with indices in $G$. Clearly, summing
	over all possible indices gives us the original expression:
	\[
	\mathcal{Z}_{D}\left(\mathbf{a},\mathbf{b}\right)=\mathcal{Z}_{y_{\text{sequential}}^{p,i,H,L,d_{x},\eta}\left(S_{1},S_{2}\right)}\left(\mathbf{a},\mathbf{b}\right)
	\]
	Given $\varepsilon>0$, we wish to find a subset of the indices, $G\subseteq D$,
	such that the sum of all terms whose indices are in $G$ is an $\varepsilon$-approximation
	of $\mathcal{Z}_{y_{\text{sequential}}^{p,i,H,L,d_{x},\eta}\left(S_{1},S_{2}\right)}\left(\mathbf{a},\mathbf{b}\right)$.
	That is, we are looking for $G\subseteq D$ such that for all $\mathbf{a},\mathbf{b}$:
	\[
	\left|\mathcal{Z}_{D}\left(\mathbf{a},\mathbf{b}\right)-\mathcal{Z}_{G}\left(\mathbf{a},\mathbf{b}\right)\right| =\left|\mathcal{Z}_{D\backslash G}\left(\mathbf{a},\mathbf{b}\right)\right| \leq \varepsilon
	\]
	
	Note that since we assume that $\mathcal{Z}_{D}^{+}\left(\mathbf{a},\mathbf{b}\right)=\mathcal{Z}_{y_{\text{sequential}}^{p,i,H,L,d_{x},\eta}\left(S_{1},S_{2}\right)}^{+} < M$, we can get: 
	\[
	\left|\mathcal{Z}_{D\backslash G}\left(\mathbf{a},\mathbf{b}\right)\right|
	= \frac{\left|\mathcal{Z}_{D\backslash G}\left(\mathbf{a},\mathbf{b}\right)\right|}
	{\mathcal{Z}_{G}^{+}\left(\mathbf{a},\mathbf{b}\right)}\, 
	\mathcal{Z}_{G}^{+}\left(\mathbf{a},\mathbf{b}\right) 
	\leq 
	\frac{\mathcal{Z}_{D\backslash G}^{+}\left(\mathbf{a},\mathbf{b}\right)}
	{\mathcal{Z}_{G}^{+}\left(\mathbf{a},\mathbf{b}\right)}\, M
	\]
	and it follows that it is enough for us to show that:
	\[
	\frac{\mathcal{Z}_{D\backslash G}^{+}\left(\mathbf{a},\mathbf{b}\right)}
	{\mathcal{Z}_{G}^{+}\left(\mathbf{a},\mathbf{b}\right)}\, 
	\leq \frac{\varepsilon}{M}
	\]
	which is equivalent showing that:
	\begin{align}
	\frac{\mathcal{Z}_{D\backslash G}\left(\mathbf{a},\mathbf{b}\right)}
	{\mathcal{Z}_{G}\left(\mathbf{a},\mathbf{b}\right)}\, 
	\leq \frac{\varepsilon}{M}
	\label{eq:eps-approx alternative}
	\end{align}
	under the assumption:
	\begin{align}
	\forall\theta\in\Theta && \theta,-\frac{\partial\mathcal{L}\left(S_1\right)}{\partial\theta},f\left(\mathbf{a}\right)_{\alpha}^{j},\mathbf{b}_{\alpha},g\left(\mathbf{b}\right)_{\alpha}^{j}\in\left[\Lambda_{\min},\Lambda_{\max}\right]
	\label{eq:eps-approx assumption}
	\end{align}
	which we make going forward. This will ensure that $\mathcal{Z}_G$ is an $\varepsilon$-approximation
	of $\mathcal{Z}_{y_{\text{sequential}}^{p,i,H,L,d_{x},\eta}\left(S_{1},S_{2}\right)}\left(\mathbf{a},\mathbf{b}\right)$. 
	
	Now, we assume that $\forall\theta\in\Theta:\theta,-\frac{\partial\mathcal{L}\left(S_1\right)}{\partial\theta}\in\left[\Lambda_{\min},\Lambda_{\max}\right]$,
	and by \cite{levine2020limits}, the $P$s and $Q$s in eq. \ref{eq:sequential}
	are products of up to $L$ matrices, so each of their coordinates
	is bounded in $\left[\Lambda_{\min}^{L},\Lambda_{\max}^{L}\right]$
	and we assume without loss of generality that $\Lambda_{\min}\leq1\leq\Lambda_{\max}$
	(otherwise we could have picked a smaller $\Lambda_{\min}$ and a
	larger $\Lambda_{\max}$). Then for each $\left(N_{\text{A}},\boldsymbol{p},\boldsymbol{n},\boldsymbol{e}\right)\in D$
	the following inequalities hold:
	\begin{align*}
	\tau_{\alpha_{1},\dots,\beta_{C\left(L\right)}} & \leq\sum_{h\in\left[H\right]^{\left[C\left(L\right)\right]}}\sum_{r_{1},\dots,r_{C\left(L\right)+1}=1}^{d_{a}}\Lambda_{\max}\left[\left(\prod_{c=1}^{C\left(L\right)+1}\Lambda_{\max}\right)\left(\prod_{c=1}^{C\left(L\right)}\Lambda_{\max}\right)\right]\\
	& =d_{x}^{C\left(L\right)}d_{a}\Lambda_{\max}^{2C\left(L\right)+2}
	\end{align*}
	\begin{align*}
	& \lambda_{N_{\text{A}},\boldsymbol{p},\boldsymbol{n}}\\
	& \leq\sum_{\substack{\substack{I_{P}\subseteq\left[C\left(L\right)+1\right]\\
				\substack{I_{Q}\subseteq\left[C\left(L\right)\right]}
				\\
				\left|I_{P}\right|+\left|I_{Q}\right|=N_{\text{A}}
			}
		}
	}\sum_{\substack{\substack{\alpha_{1},\dots,\alpha_{C\left(L\right)+1}\\
				\beta_{1},\dots,\beta_{C\left(L\right)}
			}
			=1\\
			\forall\kappa\in\left[d_{x}\right]\,\left|\left\{ c\in I_{P}\left|\alpha_{c}=\kappa\right.\right\} \right|+\left|\left\{ c\in I_{Q}\left|\beta_{c}=\kappa\right.\right\} \right|=p_{\kappa}\\
			\kappa\in\left[d_{x}\right]\,\left|\left\{ c\in\left[C\left(L\right)+1\right]\setminus I_{P}\left|\alpha_{c}=\kappa\right.\right\} \right|+\left|\left\{ c\in\left[C\left(L\right)\right]\setminus I_{Q}\left|\beta_{c}=\kappa\right.\right\} \right|=n_{\kappa}
		}
	}^{d_{x}}d_{x}^{C\left(L\right)}d_{a}\Lambda_{\max}^{2C\left(L\right)+2}\\
	& =\binom{2C\left(L\right)+1}{N_{\text{A}}}\binom{N_{\text{A}}}{p_{1},\ldots,p_{d_{x}}}\binom{2C\left(L\right)+1-N_{\text{A}}}{n_{1},\ldots,n_{d_{x}}}d_{x}^{C\left(L\right)}d_{a}\Lambda_{\max}^{2C\left(L\right)+2}
	\end{align*}
	\begin{align*}
	\phi_{A,\boldsymbol{p},\boldsymbol{e}} & \leq\sum_{\substack{0\leq p_{1,1},\dots,p_{d_{x},N}\leq N_{\text{A}}\\
			z_{1}+\dots+z_{N}=N_{\text{A}}\\
			\forall\alpha\in\left[d_{x}\right]\,\sum_{j=1}^{N}p_{\alpha,j}=p_{\alpha}\\
			\forall j\in\left[N\right]\,\sum_{\alpha=1}^{d_{x}}p_{\alpha,j}=z_{j}\\
			\forall j\in\left[N\right]\,z_{j}\equiv e_{j}\mod2
		}
	}\prod_{j=1}^{N}\prod_{\alpha=1}^{d_{x}}\Lambda_{\max}^{p_{\alpha,j}}\\
	& =\left(\sum_{\substack{0\leq p_{1,1},\dots,p_{d_{x},N}\leq N_{\text{A}}\\
			z_{1}+\dots+z_{N}=N_{\text{A}}\\
			\forall\alpha\in\left[d_{x}\right]\,\sum_{j=1}^{N}p_{\alpha,j}=p_{\alpha}\\
			\forall j\in\left[N\right]\,\sum_{\alpha=1}^{d_{x}}p_{\alpha,j}=z_{j}\\
			\forall j\in\left[N\right]\,z_{j}\equiv e_{j}\mod2
		}
	}1\right)\Lambda_{\max}^{N_{\text{A}}}
	\end{align*}
	\begin{align*}
	\Psi_{B,\boldsymbol{p},\boldsymbol{n},\boldsymbol{e}} & \leq\left(\prod_{\alpha=1}^{d_{x}}\Lambda_{\max}^{p_{\alpha}}\right)\sum_{\substack{0\leq n_{1,1},\dots,n_{d_{x},N}\leq2C\left(L\right)+1-N_{\text{A}}\\
			m_{1}+\dots+m_{N}=2C\left(L\right)+1-N_{\text{A}}\\
			\forall\alpha\in\left[d_{x}\right]\,\sum_{j=1}^{N}n_{\alpha,j}=n_{\alpha}\\
			\forall j\in\left[N\right]\,\sum_{\alpha=1}^{d_{x}}n_{\alpha,j}=m_{j}\\
			\forall j\in\left[N\right]\backslash\left\{ i\right\} \,m_{j}\equiv e_{j}\mod2\\
			m_{i}\equiv\left(1-e_{i}\right)\mod2
		}
	}\prod_{j=1}^{N}\prod_{\alpha=1}^{d_{x}}\Lambda_{\max}^{n_{\alpha,j}}\\
	& =\left(\sum_{\substack{0\leq n_{1,1},\dots,n_{d_{x},N}\leq2C\left(L\right)+1-N_{\text{A}}\\
			m_{1}+\dots+m_{N}=2C\left(L\right)+1-N_{\text{A}}\\
			\forall\alpha\in\left[d_{x}\right]\,\sum_{j=1}^{N}n_{\alpha,j}=n_{\alpha}\\
			\forall j\in\left[N\right]\,\sum_{\alpha=1}^{d_{x}}n_{\alpha,j}=m_{j}\\
			\forall j\in\left[N\right]\backslash\left\{ i\right\} \,m_{j}\equiv e_{j}\mod2\\
			m_{i}\equiv\left(1-e_{i}\right)\mod2
		}
	}1\right)\Lambda_{\max}^{2C\left(L\right)+1}
	\end{align*}
	and therefore:
	\begin{align*}
	& \lambda_{N_{\text{A}},\boldsymbol{p},\boldsymbol{n}}\cdot\phi_{A,\boldsymbol{p},\boldsymbol{e}}\cdot\Psi_{B,\boldsymbol{p},\boldsymbol{n},\boldsymbol{e}}\\
	& \leq\left(\sum_{\substack{0\leq p_{1,1},\dots,p_{d_{x},N}\leq N_{\text{A}}\\
			z_{1}+\dots+z_{N}=N_{\text{A}}\\
			\forall\alpha\in\left[d_{x}\right]\,\sum_{j=1}^{N}p_{\alpha,j}=p_{\alpha}\\
			\forall j\in\left[N\right]\,\sum_{\alpha=1}^{d_{x}}p_{\alpha,j}=z_{j}\\
			\forall j\in\left[N\right]\,z_{j}\equiv e_{j}\mod2
		}
	}1\right)\left(\sum_{\substack{0\leq n_{1,1},\dots,n_{d_{x},N}\leq2C\left(L\right)+1-N_{\text{A}}\\
			m_{1}+\dots+m_{N}=2C\left(L\right)+1-N_{\text{A}}\\
			\forall\alpha\in\left[d_{x}\right]\,\sum_{j=1}^{N}n_{\alpha,j}=n_{\alpha}\\
			\forall j\in\left[N\right]\,\sum_{\alpha=1}^{d_{x}}n_{\alpha,j}=m_{j}\\
			\forall j\in\left[N\right]\backslash\left\{ i\right\} \,m_{j}\equiv e_{j}\mod2\\
			m_{i}\equiv\left(1-e_{i}\right)\mod2
		}
	}1\right)\\
	& \cdot\binom{2C\left(L\right)+1}{N_{\text{A}}}\binom{N_{\text{A}}}{p_{1},\ldots,p_{d_{x}}}\binom{2C\left(L\right)+1-N_{\text{A}}}{n_{1},\ldots,n_{d_{x}}}d_{x}^{C\left(L\right)}d_{a}\Lambda_{\max}^{L\left(2C\left(L\right)+2\right)+2C\left(L\right)+1+N_{\text{A}}}
	\end{align*}
	Relaxing the parity constraints inside the brackets and recalling
	that $\Lambda_{\max}\geq1$ gives us an upper bound:
	\begin{align*}
	& \leq\left(\prod_{\alpha=1}^{d_{x}}\left(\!\!\binom{N}{p_{\alpha}}\!\!\right)\right)\cdot\left(\prod_{\alpha=1}^{d_{x}}\left(\!\!\binom{N}{n_{\alpha}}\!\!\right)\right)\\
	& \cdot\binom{2C\left(L\right)+1}{N_{\text{A}}}\binom{N_{\text{A}}}{p_{1},\ldots,p_{d_{x}}}\binom{2C\left(L\right)+1-N_{\text{A}}}{n_{1},\ldots,n_{d_{x}}}d_{x}^{C\left(L\right)}d_{a}\Lambda_{\max}^{\left(L+2\right)\left(2C\left(L\right)+2\right)}
	\end{align*}
	which we can further bound using lemmas 3 and 4 in \cite{levine2020limits}
	until we are left with:
	\begin{align*}
	& \leq\left(\frac{e\left(2d_{x}N+2C\left(L\right)+1\right)}{d_{x}N}\right)^{2d_{x}N}\\
	& \cdot\binom{2C\left(L\right)+1}{N_{\text{A}}}\binom{N_{\text{A}}}{p_{1},\ldots,p_{d_{x}}}\binom{2C\left(L\right)+1-N_{\text{A}}}{n_{1},\ldots,n_{d_{x}}}d_{x}^{C\left(L\right)}d_{a}\Lambda_{\max}^{\left(L+2\right)\left(2C\left(L\right)+2\right)}
	\end{align*}
	
	On the other hand:
	\begin{align*}
	& \lambda_{N_{\text{A}},\boldsymbol{p},\boldsymbol{n}}\cdot\phi_{A,\boldsymbol{p},\boldsymbol{e}}\cdot\Psi_{B,\boldsymbol{p},\boldsymbol{n},\boldsymbol{e}}\\
	& \geq\binom{2C\left(L\right)+1}{N_{\text{A}}}\binom{N_{\text{A}}}{p_{1},\ldots,p_{d_{x}}}\binom{2C\left(L\right)+1-N_{\text{A}}}{n_{1},\ldots,n_{d_{x}}}\\
	& \cdot d_{x}^{C\left(L\right)}d_{a}\Lambda_{\min}^{L\left(2C\left(L\right)+2\right)}\\
	& \cdot\left(\sum_{\substack{0\leq p_{1,1},\dots,p_{d_{x},N}\leq N_{\text{A}}\\
			z_{1}+\dots+z_{N}=N_{\text{A}}\\
			\forall\alpha\in\left[d_{x}\right]\,\sum_{j=1}^{N}p_{\alpha,j}=p_{\alpha}\\
			\forall j\in\left[N\right]\,\sum_{\alpha=1}^{d_{x}}p_{\alpha,j}=z_{j}\\
			\forall j\in\left[N\right]\,z_{j}\equiv e_{j}\mod2
		}
	}1\right)\Lambda_{\min}^{N_{\text{A}}}\\
	& \cdot\left(\sum_{\substack{0\leq n_{1,1},\dots,n_{d_{x},N}\leq2C\left(L\right)+1-N_{\text{A}}\\
			m_{1}+\dots+m_{N}=2C\left(L\right)+1-N_{\text{A}}\\
			\forall\alpha\in\left[d_{x}\right]\,\sum_{j=1}^{N}n_{\alpha,j}=n_{\alpha}\\
			\forall j\in\left[N\right]\,\sum_{\alpha=1}^{d_{x}}n_{\alpha,j}=m_{j}\\
			\forall j\in\left[N\right]\backslash\left\{ i\right\} \,m_{j}\equiv e_{j}\mod2\\
			m_{i}\equiv\left(1-e_{i}\right)\mod2
		}
	}1\right)\Lambda_{\min}^{2C\left(L\right)+1}\\
	& \geq\binom{2C\left(L\right)+1}{N_{\text{A}}}\binom{N_{\text{A}}}{p_{1},\ldots,p_{d_{x}}}\binom{2C\left(L\right)+1-N_{\text{A}}}{n_{1},\ldots,n_{d_{x}}}d_{x}^{C\left(L\right)}d_{a}\Lambda_{\min}^{L\left(2C\left(L\right)+2\right)+2C\left(L\right)+1+N_{\text{A}}}\\
	& \geq\binom{2C\left(L\right)+1}{N_{\text{A}}}\binom{N_{\text{A}}}{p_{1},\ldots,p_{d_{x}}}\binom{2C\left(L\right)+1-N_{\text{A}}}{n_{1},\ldots,n_{d_{x}}}d_{x}^{C\left(L\right)}d_{a}\Lambda_{\min}^{\left(L+2\right)\left(2C\left(L\right)+2\right)}
	\end{align*}
	Combining the upper and lower bound for $\lambda_{N_{\text{A}},\boldsymbol{p},\boldsymbol{n}}\cdot\phi_{A,\boldsymbol{p},\boldsymbol{e}}\cdot\Psi_{B,\boldsymbol{p},\boldsymbol{n},\boldsymbol{e}}$,
	we get that for all $G\subseteq D$:
	\begin{align*}
	& \frac{\mathcal{Z}_{D\backslash G}\left(\mathbf{a},\mathbf{b}\right)}{\mathcal{Z}_{G}\left(\mathbf{a},\mathbf{b}\right)}\\
	& \leq\left(\frac{e\left(2d_{x}N+2C\left(L\right)+1\right)}{d_{x}N}\right)^{2d_{x}N}\cdot\left(\frac{\Lambda_{\max}}{\Lambda_{\min}}\right)^{\left(L+2\right)\left(2C\left(L\right)+2\right)}\\
	& \cdot\frac{\sum_{\left(N_{\text{A}},\boldsymbol{p},\boldsymbol{n},\boldsymbol{e}\right)\in D\backslash G}\eta^{N_{\text{A}}}\binom{2C\left(L\right)+1}{N_{\text{A}}}\binom{N_{\text{A}}}{p_{1},\ldots,p_{d_{x}}}\binom{2C\left(L\right)+1-N_{\text{A}}}{n_{1},\ldots,n_{d_{x}}}}{\sum_{\left(N_{\text{A}},\boldsymbol{p},\boldsymbol{n},\boldsymbol{e}\right)\in G}\eta^{N_{\text{A}}}\binom{2C\left(L\right)+1}{N_{\text{A}}}\binom{N_{\text{A}}}{p_{1},\ldots,p_{d_{x}}}\binom{2C\left(L\right)+1-N_{\text{A}}}{n_{1},\ldots,n_{d_{x}}}}
	\end{align*}
	so in order to show that (\ref{eq:eps-approx alternative}) holds,
	it suffices to show that:
	\begin{align*}
	& \frac{\sum_{\left(N_{\text{A}},\boldsymbol{p},\boldsymbol{n},\boldsymbol{e}\right)\in D\backslash G}\eta^{N_{\text{A}}}\binom{2C\left(L\right)+1}{N_{\text{A}}}\binom{N_{\text{A}}}{p_{1},\ldots,p_{d_{x}}}\binom{2C\left(L\right)+1-N_{\text{A}}}{n_{1},\ldots,n_{d_{x}}}}{\sum_{\left(N_{\text{A}},\boldsymbol{p},\boldsymbol{n},\boldsymbol{e}\right)\in G}\eta^{N_{\text{A}}}\binom{2C\left(L\right)+1}{N_{\text{A}}}\binom{N_{\text{A}}}{p_{1},\ldots,p_{d_{x}}}\binom{2C\left(L\right)+1-N_{\text{A}}}{n_{1},\ldots,n_{d_{x}}}}\\
	& \leq\left(\frac{d_{x}N}{e\left(2d_{x}N+2C\left(L\right)+1\right)}\right)^{2d_{x}N}\cdot\left(\frac{\Lambda_{\min}}{\Lambda_{\max}}\right)^{\left(L+2\right)\left(2C\left(L\right)+2\right)}\cdot \frac{\varepsilon}{M}
	\end{align*}
	Now, we can limit ourselves to subsets $G\subseteq D$ of the form:
	\[
	G\left(T\right)=\left\{ {\scriptstyle \left(N_{\text{A}},\boldsymbol{p},\boldsymbol{n},\boldsymbol{e}\right)\in D}\left|{\scriptstyle \eta^{N_{\text{A}}}\binom{2C\left(L\right)+1}{N_{\text{A}}}\binom{N_{\text{A}}}{p_{1},\ldots,p_{d_{x}}}\binom{2C\left(L\right)+1-N_{\text{A}}}{n_{1},\ldots,n_{d_{x}}}\geq T}\right.\right\} 
	\]
	and in this case we get that:
	\begin{align*}
	& \frac{\sum_{\left(N_{\text{A}},\boldsymbol{p},\boldsymbol{n},\boldsymbol{e}\right)\in D\backslash G\left(T\right)}\eta^{N_{\text{A}}}\binom{2C\left(L\right)+1}{N_{\text{A}}}\binom{N_{\text{A}}}{p_{1},\ldots,p_{d_{x}}}\binom{2C\left(L\right)+1-N_{\text{A}}}{n_{1},\ldots,n_{d_{x}}}}{\sum_{\left(N_{\text{A}},\boldsymbol{p},\boldsymbol{n},\boldsymbol{e}\right)\in G\left(T\right)}\eta^{N_{\text{A}}}\binom{2C\left(L\right)+1}{N_{\text{A}}}\binom{N_{\text{A}}}{p_{1},\ldots,p_{d_{x}}}\binom{2C\left(L\right)+1-N_{\text{A}}}{n_{1},\ldots,n_{d_{x}}}}\\
	& \leq\frac{\sum_{\left(N_{\text{A}},\boldsymbol{p},\boldsymbol{n},\boldsymbol{e}\right)\in D\backslash G\left(T\right)}T}{\sum_{\left(N_{\text{A}},\boldsymbol{p},\boldsymbol{n},\boldsymbol{e}\right)\in G\left(T\right)}T}\\
	& =\frac{\left|D\right|-\left|G\left(T\right)\right|}{\left|G\left(T\right)\right|}
	\end{align*}
	Let us define:
	\begin{align*}
	\tilde{D}=\tilde{D}_{2C\left(L\right)+1,d_{x}} & \coloneqq\left\{ {\scriptstyle \left(N_{\text{A}},\boldsymbol{p},\boldsymbol{n}\right)\in\mathbb{N}^{2d_{x}+1}}\left|\substack{0\leq N_{\text{A}}\leq2C\left(L\right)+1\\
		\sum_{\alpha=1}^{d_{x}}p_{\alpha}=N_{\text{A}}\\
		\sum_{\alpha=1}^{d_{x}}n_{\alpha}=2C\left(L\right)+1-N_{\text{A}}
	}
	\right.\right\} \\
	\tilde{G}\left(T\right) & \coloneqq\left\{ {\scriptstyle \left(N_{\text{A}},\boldsymbol{p},\boldsymbol{n}\right)\in\tilde{D}}\left|{\scriptstyle \eta^{N_{\text{A}}}\binom{2C\left(L\right)+1}{N_{\text{A}}}\binom{N_{\text{A}}}{p_{1},\ldots,p_{d_{x}}}\binom{2C\left(L\right)+1-N_{\text{A}}}{n_{1},\ldots,n_{d_{x}}}\geq T}\right.\right\} 
	\end{align*}
	and note that:
	\begin{align*}
	\left|\tilde{G}\left(T\right)\right| & \leq\left|G\left(T\right)\right|\leq2^{N}\cdot\left|\tilde{G}\left(T\right)\right|\\
	\left|\tilde{D}\right| & \leq\left|D\right|\leq2^{N}\cdot\left|\tilde{D}\right|
	\end{align*}
	
	and also:
	\begin{align*}
	\left|\tilde{D}\right| & =\left(\!\!\binom{2d_{x}}{2C\left(L\right)+1}\!\!\right)\\
	& \leq\left(\frac{e\left(2d_{x}+2C\left(L\right)+1\right)}{d_{x}}\right)^{2d_{x}}
	\end{align*}
	where the inequality is due to lemma 3 in \cite{levine2020limits}.
	
	Hence:
	\[
	\frac{\left|D\right|-\left|G\left(T\right)\right|}{\left|G\left(T\right)\right|}\leq\frac{2^{N}\cdot\left(\frac{e\left(2d_{x}+2C\left(L\right)+1\right)}{d_{x}}\right)^{2d_{x}}-\left|\tilde{G}\left(T\right)\right|}{\left|\tilde{G}\left(T\right)\right|}
	\]
	
	And after rearranging we get that for each $T\geq0$ such that: 
	\begin{equation}
	\left|\tilde{G}\left(T\right)\right|\geq\frac{2^{N}\cdot\left(e\left(2d_{x}+2C\left(L\right)+1\right)\right)^{2d_{x}}}{\left(1+\left(\frac{d_{x}N}{e\left(2d_{x}N+2C\left(L\right)+1\right)}\right)^{2d_{x}N}\cdot\left(\frac{\Lambda_{\min}}{\Lambda_{\max}}\right)^{\left(L+2\right)\left(2C\left(L\right)+2\right)}\cdot
		\frac{\varepsilon}{M}\right)d_{x}^{2d_{x}}}\label{eq:|G_tilde(T)| lower bound}
	\end{equation}
	In order for $\mathcal{Z}_{G\left(T\right)}\left(\mathbf{a},\mathbf{b}\right)$
	to be an $\varepsilon$-approximation of $\mathcal{Z}_{y_{\text{sequential}}^{p,i,H,L,d_{x},\eta}\left(S_{1},S_{2}\right)}\left(\mathbf{a},\mathbf{b}\right)$.
	
	\subsubsection*{Step 3 --- An upper bound to $\varepsilon\text{-sep}_{\left(\mathbf{a},\mathbf{b}\right)}\left(\mathcal{Z}_{y_{\text{sequential}}^{p,i,H,L,d_{x},\eta}\left(S_{1},S_{2}\right)}\right)$}
	
	In the last step we have found a condition on subsets of indices,
	such that summing over any subset who meets this condition will yield
	an $\varepsilon$-approximation of $\mathcal{Z}_{y_{\text{sequential}}^{p,i,H,L,d_{x},\eta}\left(S_{1},S_{2}\right)}$.
	We will now find a specific subset who meets this condition, and use
	it in order to bound $\varepsilon\text{-sep}_{\left(\mathbf{a},\mathbf{b}\right)}\left(\mathcal{Z}_{y_{\text{sequential}}^{p,i,H,L,d_{x},\eta}\left(S_{1},S_{2}\right)}\right)$
	from above.
	
	We will focus our attention on $T$s of the form: 
	\begin{align*}
	T\left(s\right) & =s\cdot\eta^{\frac{\eta\left(2C\left(L\right)+1\right)}{1+\eta}}\binom{2C\left(L\right)+1}{\frac{\eta\left(2C\left(L\right)+1\right)}{1+\eta}}\\
	& \cdot\binom{\frac{\eta\left(2C\left(L\right)+1\right)}{1+\eta}}{\frac{\eta\left(2C\left(L\right)+1\right)}{\left(1+\eta\right)d_{x}},\ldots,\frac{\eta\left(2C\left(L\right)+1\right)}{\left(1+\eta\right)d_{x}}}\binom{\frac{2C\left(L\right)+1}{1+\eta}}{\frac{2C\left(L\right)+1}{\left(1+\eta\right)d_{x}},\ldots,\frac{2C\left(L\right)+1}{\left(1+\eta\right)d_{x}}}
	\end{align*}
	for some $s\in\left(0,e^{-1.5}\right]$ which we'll determine later.
	
	By lemma \ref{lem:A lower bound on the number of non-negligible summands}
	we have that:
	\[
	\left|\tilde{G}\left(T\left(s\right)\right)\right|\geq\frac{1}{d_{x}\sqrt{\pi}}\left(\frac{\pi e\left(2C\left(L\right)+1\right)}{2d_{x}^{2}\left(1+\eta\right)}\right)^{\frac{d_{x}-1}{2}}\cdot\left(\ln\left(s^{-1}\right)\right)^{\frac{d_{x}-1}{2}}
	\]
	so we can choose: 
	\begin{align*}
	& {\scriptstyle s^{*}=\exp\left(-\left(\frac{2^{2N}\left(e\left(2d_{x}+2C\left(L\right)+1\right)\right)^{4d_{x}}\pi\left(2\left(1+\eta\right)\right)^{d_{x}-1}}{\left(1+\left(\frac{d_{x}N}{e\left(2d_{x}N+2C\left(L\right)+1\right)}\right)^{2d_{x}N}\left(\frac{\Lambda_{\min}}{\Lambda_{\max}}\right)^{\left(L+2\right)\left(2C\left(L\right)+2\right)}\frac{\varepsilon}{M}\right)^{2}d_{x}^{2d_{x}}\left(\pi e\left(2C\left(L\right)+1\right)\right)^{d_{x}-1}}\right)^{\frac{1}{d_{x}-1}}\right)}
	\end{align*}
	and since this upholds $s^{*}<e^{-1.5}$, we get that (\ref{eq:|G_tilde(T)| lower bound})
	indeed holds for $\tilde{G}\left(T\left(s^{*}\right)\right)$, and
	therefore that $\mathcal{Z}_{G\left(T\left(s^{*}\right)\right)}$
	is an $\varepsilon$-approximation of $\mathcal{Z}_{y_{\text{sequential}}^{p,i,H,L,d_{x},\eta}\left(S_{1},S_{2}\right)}$.
	
	Now, from lemma \ref{lem:An upper bound on the number of non-negligible summands}
	we get: 
	\begin{align*}
	\left|\tilde{G}\left(T\left(s^{*}\right)\right)\right| & \leq\frac{2\left(2C\left(L\right)+1\right)}{\left(d_{x}-1\right)^{2}}\cdot\left(\frac{25\sqrt{\eta}}{\left(d_{x}-1\right)}\right)^{d_{x}-1}\\
	& \cdot\frac{2^{2N}\cdot\left(e\left(2d_{x}+2C\left(L\right)+1\right)\right)^{4d_{x}}}{\left(1+\left(\frac{d_{x}N}{e\left(2d_{x}N+2C\left(L\right)+1\right)}\right)^{2d_{x}N}\cdot\left(\frac{\Lambda_{\min}}{\Lambda_{\max}}\right)^{\left(L+2\right)\left(2C\left(L\right)+2\right)}\cdot
		\frac{\varepsilon}{M}
		\right)^{2}\cdot d_{x}^{2d_{x}}}
	\end{align*}
	and therefore:
	\begin{align*}
	& \varepsilon\text{-seq-sep}\left(y_{\text{sequential}}^{p,i,H,L,d_{x},\eta}\right)\\
	& =\varepsilon\text{-sep}_{\left(\mathbf{a},\mathbf{b}\right)}\left(\mathcal{Z}_{y_{\text{sequential}}^{p,i,H,L,d_{x},\eta}\left(S_{1},S_{2}\right)}\right)\\
	& \leq\left|G\left(T\left(s^{*}\right)\right)\right|\\
	& \leq2^{N}\cdot\left|\tilde{G}\left(T\left(s^{*}\right)\right)\right|\\
	& \leq\frac{2\left(2C\left(L\right)+1\right)}{\left(d_{x}-1\right)^{2}}\cdot\left(\frac{25\sqrt{\eta}}{\left(d_{x}-1\right)}\right)^{d_{x}-1}\\
	& \cdot\frac{2^{3N}\cdot\left(e\left(2d_{x}+2C\left(L\right)+1\right)\right)^{4d_{x}}}{\left(1+\left(\frac{d_{x}N}{e\left(2d_{x}N+2C\left(L\right)+1\right)}\right)^{2d_{x}N}\cdot\left(\frac{\Lambda_{\min}}{\Lambda_{\max}}\right)^{\left(L+2\right)\left(2C\left(L\right)+2\right)}\cdot
		\frac{\varepsilon}{M}
		\right)^{2}\cdot d_{x}^{2d_{x}}}
	\end{align*}
\end{proof}

\subsection{Lemmas for estimating the number of coefficients}
\begin{remark}
	For brevity and clarity we will use expressions of the form $\binom{K}{\frac{K}{M},\ldots,\frac{K}{M}}$,
	regardless of whether $K$ is divisible by $M$ or not. For the latter
	case, this expression should actually be:
	\[
	\binom{K}{\left\lfloor \frac{K}{M}\right\rfloor ,\ldots,\left\lfloor \frac{K}{M}\right\rfloor ,\underbrace{{\textstyle \left\lfloor \frac{K}{M}\right\rfloor +1,\ldots,\left\lfloor \frac{K}{M}\right\rfloor +1}}_{\left(K\mod M\right)\text{ times}}}
	\]
	expressions of the form $\binom{K}{\frac{K}{M},\ldots,\frac{K}{M}+1}$
	should be read as:
	\[
	\binom{K}{\left\lfloor \frac{K}{M}\right\rfloor ,\ldots,\left\lfloor \frac{K}{M}\right\rfloor ,\underbrace{{\textstyle \left\lfloor \frac{K}{M}\right\rfloor +1,\ldots,\left\lfloor \frac{K}{M}\right\rfloor +1}}_{\left(K\mod M\right)+1\text{ times}}}
	\]
	and expressions of the form $\binom{K}{\frac{K}{M},\ldots,\frac{K}{M}-1}$
	should be read as:
	\[
	\begin{cases}
	\binom{K}{\frac{K}{M},\ldots,\frac{K}{M},\underbrace{{\scriptstyle \frac{K}{M}-1}}_{\text{just once}}} & \text{if }K\mod M\equiv0\\
	\binom{K}{\left\lfloor \frac{K}{M}\right\rfloor ,\ldots,\left\lfloor \frac{K}{M}\right\rfloor ,\underbrace{{\scriptstyle \left\lfloor \frac{K}{M}\right\rfloor -1,\ldots,\left\lfloor \frac{K}{M}\right\rfloor -1}}_{\left(K\mod M\right)-1\text{ times}}} & \text{otherwise}
	\end{cases}
	\]
\end{remark}
\begin{lemma}
	\label{lem:Maximum of a multinomial coefficient}For all $K,M\in\mathbb{N}$,
	the maximal value of $\binom{K}{a_{1},\ldots,a_{M}}$ is achieved
	when $\forall j_{1},j_{2}\in\left[M\right]$ it holds that $\left|a_{j_{1}}-a_{j_{2}}\right|\leq1$.
\end{lemma}
\begin{proof}
	\noindent Let $a_{1},\ldots,a_{M}$ be a sequence of non-negative
	integers such that $a_{1}+\ldots+a_{M}=K$ and $\binom{K}{a_{1},\ldots,a_{M}}$
	is maximal. Assume towards a contradiction there exist $j_{1},j_{2}\in\left[M\right]$
	such that $a_{j_{1}}-a_{j_{2}}>1\iff\frac{a_{j_{2}}+1}{a_{j_{1}}}<1$,
	then:
	\begin{align*}
	\binom{K}{a_{1},\ldots,a_{M}} & =\frac{K!}{\prod_{i=1}^{M}a_{i}!}\\
	& =\frac{K!}{\left(\prod_{i\in\left[M\right]\backslash\left\{ j_{1},j_{2}\right\} }a_{i}!\right)\cdot\left(a_{j_{1}}-1\right)!\cdot\left(a_{j_{2}}+1\right)!}\cdot\frac{a_{j_{2}}+1}{a_{j_{1}}}\\
	& <\frac{K!}{\left(\prod_{i\in\left[M\right]\backslash\left\{ j_{1},j_{2}\right\} }a_{i}!\right)\cdot\left(a_{j_{1}}-1\right)!\cdot\left(a_{j_{2}}+1\right)!}\\
	& =\binom{K}{a_{1},\ldots,a_{j_{1}}-1,\ldots,a_{j_{2}}+1,\ldots,a_{M}}
	\end{align*}
	in contrary to the maximality of $\binom{K}{a_{1},\ldots,a_{M}}$.
	Therefore, $\forall j_{1},j_{2}\in\left[M\right]$, $\left|a_{j_{1}}-a_{j_{2}}\right|\leq1$.
\end{proof}
\begin{lemma}
	\label{lem:Recurrence relation of =00007BK =00005Cchoose n=00007D=00005Ceta^=00007Bn=00007D=00005Cbinom=00007Bn=00007D=00007B=00005Cfrac=00007Bn=00007D=00007BM=00007D,=00005Cldots,=00005Cfrac=00007Bn=00007D=00007BM=00007D=00007D=00005Cbinom=00007BK-n=00007D=00007B=00005Cfrac=00007BK-n=00007D=00007BM=00007D,=00005Cldots,=00005Cfrac=00007BK-n=00007D=00007BM=00007D=00007D}Let
	$K,M$ be two fixed natural numbers, $\eta\in\left(0,1\right]$ and
	denote 
	\[
	S\left(n\right)\coloneqq{K \choose n}\eta^{n}\binom{n}{\frac{n}{M},\ldots,\frac{n}{M}}\binom{K-n}{\frac{K-n}{M},\ldots,\frac{K-n}{M}}
	\]
	Then for all $n\in\left[K\right]\cup\left\{ 0\right\} $: 
	\begin{align}
	& =\begin{alignedat}{1}S\left(n\right) & =\left(\prod_{j=0}^{\left\lfloor \frac{n}{M}\right\rfloor -1}\left(\frac{\eta\left(K-jM\right)}{\left(j+1\right)M}\right)^{M}\right)\\
	& \cdot\left(\frac{\eta\left(K-\left\lfloor \frac{n}{M}\right\rfloor M\right)}{\left(\left\lfloor \frac{n}{M}\right\rfloor +1\right)M}\right)^{n\mod M}\binom{K}{\frac{K}{M},\ldots,\frac{K}{M}}
	\end{alignedat}
	\label{eq:recurrence}
	\end{align}
\end{lemma}
\begin{proof}
	We will prove by induction.
	\begin{itemize}
		\item \textbf{\textit{Base case:}} $n=0$.
		
		\begin{align*}
		S\left(0\right) & ={K \choose 0}\eta^{0}\binom{n}{\frac{0}{M},\ldots,\frac{0}{M}}\binom{K}{\frac{K-0}{M},\ldots,\frac{K-0}{M}}\\
		& =\left(\prod_{j=0}^{\left\lfloor \frac{0}{M}\right\rfloor -1}\left(\frac{\eta\left(K-jM\right)}{\left(j+1\right)M}\right)^{M}\right)\\
		& \cdot\left(\frac{\eta\left(K-\left\lfloor \frac{0}{M}\right\rfloor M\right)}{\left(\left\lfloor \frac{0}{M}\right\rfloor +1\right)M}\right)^{0\mod M}\cdot\binom{K}{\frac{K}{M},\ldots,\frac{K}{M}}
		\end{align*}
		where the second equality is due to the fact that:
		\[
		\left(\prod_{j=0}^{\left\lfloor \frac{0}{M}\right\rfloor -1}\left(\frac{\eta\left(K-jM\right)}{\left(j+1\right)M}\right)^{M}\right)=1
		\]
		as an empty product, and:
		\[
		\left(\frac{\eta\left(K-0\cdot M\right)}{1\cdot M}\right)^{0\mod M}=1
		\]
		
		Therefore, (\ref{eq:recurrence}) is true for $n=0$.
		\item \textbf{\textit{Induction step:}} Let $n\geq0$ such that (\ref{eq:recurrence})
		holds for $n$.
		
		\begin{align*}
		S\left(n+1\right) & ={K \choose n+1}\eta^{n+1}\binom{n+1}{\frac{n}{M},\ldots,\frac{n}{M}+1}\binom{K-n-1}{\frac{K-n}{M},\ldots,\frac{K-n}{M}-1}\\
		& =\frac{\eta\left(K-n\right)}{\left(\frac{n}{M}+1\right)M}\cdot S\left(n\right)\\
		& =\frac{\eta\left(K-n\right)}{\left(\frac{n}{M}+1\right)M}\cdot\left(\prod_{j=0}^{\left\lfloor \frac{n}{M}\right\rfloor -1}\left(\frac{\eta\left(K-jM\right)}{\left(j+1\right)M}\right)^{M}\right)\\
		& \cdot\left(\frac{\eta\left(K-\left\lfloor \frac{n}{M}\right\rfloor M\right)}{\left(\left\lfloor \frac{n}{M}\right\rfloor +1\right)M}\right)^{n\mod M}\binom{K}{\frac{K}{M},\ldots,\frac{K}{M}}\\
		& =\left(\prod_{j=0}^{\left\lfloor \frac{n+1}{M}\right\rfloor -1}\left(\frac{\eta\left(K-jM\right)}{\left(j+1\right)M}\right)^{M}\right)\\
		& \cdot\left(\frac{\eta\left(K-\left\lfloor \frac{n+1}{M}\right\rfloor M\right)}{\left(\left\lfloor \frac{n+1}{M}\right\rfloor +1\right)M}\right)^{\left(n+1\right)\mod M}\binom{K}{\frac{K}{M},\ldots,\frac{K}{M}}
		\end{align*}
		Thus, (\ref{eq:recurrence}) holds for $n+1$, and the proof of the
		induction step is complete. Hence, by induction, (\ref{eq:recurrence})
		is correct for all $n\in\left[K\right]\cup\left\{ 0\right\} $.
		
	\end{itemize}
\end{proof}
\begin{lemma}
	\label{lem:Maximum of =00007BK =00005Cchoose n=00007D=00005Ceta^=00007Bn=00007D=00005Cbinom=00007Bn=00007D=00007Ba_=00007B1=00007D,=00005Cldots,a_=00007BM=00007D=00007D=00005Cbinom=00007BK-n=00007D=00007Bb_=00007B1=00007D,=00005Cldots,b_=00007BM=00007D=00007D}Let
	$K,M$ be two fixed natural numbers, and $\eta\in\left(0,1\right]$.
	Then the maximum of:
	\[
	{K \choose n}\eta^{n}\binom{n}{a_{1},\ldots,a_{M}}\binom{K-n}{b_{1},\ldots,b_{M}}
	\]
	is achieved when: 
	\[
	n=\left(\left\lfloor \frac{\eta K-M}{M\left(1+\eta\right)}\right\rfloor +1\right)M\simeq\frac{\eta K}{1+\eta}
	\]
	and for all $i\in\left[M\right]$, $\left|a_{i}-\frac{n}{M}\right|\leq1$
	and $\left|b_{i}-\frac{K-n}{M}\right|\leq1$.
\end{lemma}
\begin{proof}
	From lemma \ref{lem:Maximum of a multinomial coefficient} we know
	that a multinumial coefficient reaches its maximum when the sum is
	evenly distributed between all indices. Since the $a_{i}$s and $b_{i}$
	can be chosen independently of each other given $K$ and $n$, we
	may assume without loss of generality that no matter the value of
	$n$, for all $i\in\left[M\right]$, it holds that $\left|a_{i}-\frac{n}{M}\right|\leq1$
	and $\left|b_{i}-\frac{K-n}{M}\right|\leq1$.
	
	Define $S$ as in lemma \ref{lem:Recurrence relation of =00007BK =00005Cchoose n=00007D=00005Ceta^=00007Bn=00007D=00005Cbinom=00007Bn=00007D=00007B=00005Cfrac=00007Bn=00007D=00007BM=00007D,=00005Cldots,=00005Cfrac=00007Bn=00007D=00007BM=00007D=00007D=00005Cbinom=00007BK-n=00007D=00007B=00005Cfrac=00007BK-n=00007D=00007BM=00007D,=00005Cldots,=00005Cfrac=00007BK-n=00007D=00007BM=00007D=00007D}.
	
	Note that the function:
	\[
	T\left(x\right)=\binom{K}{x}\binom{x}{\frac{x}{M},\ldots,\frac{x}{M}}\binom{K-x}{\frac{K-x}{M},\ldots,\frac{K-x}{M}}
	\]
	is Gaussian-shaped and therefore unimodal and has a unique maximum.
	$\eta^{x}$ for $\eta\in\left(0,1\right]$ is monotonically decreasing,
	and therefore their product is also unimodal.
	
	Since we are only interested in solutions for $n\in\mathbb{N}$, we
	can reduce our problem to finding the maximal $n$ such that: 
	\[
	S\left(n-1\right)\leq S\left(n\right)\iff\frac{S\left(n\right)}{S\left(n-1\right)}\geq1
	\]
	and from lemma \ref{lem:Recurrence relation of =00007BK =00005Cchoose n=00007D=00005Ceta^=00007Bn=00007D=00005Cbinom=00007Bn=00007D=00007B=00005Cfrac=00007Bn=00007D=00007BM=00007D,=00005Cldots,=00005Cfrac=00007Bn=00007D=00007BM=00007D=00007D=00005Cbinom=00007BK-n=00007D=00007B=00005Cfrac=00007BK-n=00007D=00007BM=00007D,=00005Cldots,=00005Cfrac=00007BK-n=00007D=00007BM=00007D=00007D}
	we have:
	\begin{align*}
	\frac{S\left(n\right)}{S\left(n-1\right)} & =\frac{\eta\left(K-\left\lfloor \frac{n-1}{M}\right\rfloor M\right)}{\left(\left\lfloor \frac{n-1}{M}\right\rfloor +1\right)M}\geq1\\
	\iff & \left\lfloor \frac{n-1}{M}\right\rfloor M\leq\frac{\eta K-M}{1+\eta}
	\end{align*}
	So we get that $S\left(n\right)$ is monotonically increasing as long
	as $\left\lfloor \frac{n-1}{M}\right\rfloor M\leq\frac{\eta K-M}{1+\eta}$,
	and the largest integer for which this condition holds is: 
	\[
	n=\left(\left\lfloor \frac{\eta K-M}{M\left(1+\eta\right)}\right\rfloor +1\right)M\simeq\frac{\eta K}{1+\eta}
	\]
\end{proof}
\begin{lemma}
	\label{lem:Number of integer lattice points in a high-dimensional ball}Denote
	by $\mathcal{N}\left(\mathcal{B}_{R}^{d}\right)$ the number of integer
	lattice points in $\mathcal{B}_{R}^{d}$ (the $d$-dimensional zero-centered
	ball of radius $R$). Then: 
	\[
	\frac{\left(\frac{\pi e}{2}\right)^{\frac{d}{2}}}{\sqrt{\pi d}}\left(\frac{2R}{\sqrt{d}}-1\right)^{d}\leq\mathcal{N}\left(\mathcal{B}_{R}^{d}\right)\leq\frac{\left(\frac{\pi e}{2}\right)^{\frac{d}{2}}}{\sqrt{\pi d}}\left(\frac{2R}{\sqrt{d}}+1\right)^{d}
	\]
\end{lemma}
\begin{proof}
	\noindent Let $\mathcal{I}\left(\mathcal{B}_{R}^{d}\right)$ be the
	set of all integer lattice points in $\mathcal{B}_{R}^{d}$. For $\boldsymbol{x}\in\mathcal{I}\left(\mathcal{B}_{R}^{d}\right)$,
	define: 
	\[
	C_{\boldsymbol{x}}\coloneqq\times_{i=1}^{d}\left[x_{i}-\frac{1}{2},x_{i}+\frac{1}{2}\right]=\left\{ \boldsymbol{y}\in\mathbb{R}^{d}\left|\forall i\in\left[d\right],x_{i}-\frac{1}{2}\leq y_{i}\leq x_{i}+\frac{1}{2}\right.\right\} 
	\]
	Let $\boldsymbol{y}\in\bigcup_{\boldsymbol{x}\in\mathcal{I}\left(\mathcal{B}_{R}^{d}\right)}C_{\boldsymbol{x}}$,
	so there exists $\boldsymbol{x}\in\mathcal{I}\left(\mathcal{B}_{R}^{d}\right)$
	such that $\boldsymbol{y}\in C_{\boldsymbol{x}}$, and from the triangle
	inequality we get: 
	\[
	\left\Vert \boldsymbol{y}\right\Vert \leq\left\Vert \boldsymbol{x}\right\Vert +\left\Vert \boldsymbol{y}-\boldsymbol{x}\right\Vert \leq R+\sqrt{\left(\frac{1}{2}\right)^{2}+\ldots+\left(\frac{1}{2}\right)^{2}}=R+\frac{\sqrt{d}}{2}
	\]
	and therefore $\boldsymbol{y}\in\mathcal{B}_{R+\frac{\sqrt{d}}{2}}^{d}$.
	Since $\boldsymbol{y}$ was chosen arbitrarily, we get that:
	\[
	\bigcup_{\boldsymbol{x}\in\mathcal{I}\left(\mathcal{B}_{R}^{d}\right)}C_{\boldsymbol{x}}\subseteq\mathcal{B}_{R+\frac{\sqrt{d}}{2}}^{d}
	\]
	Note that $Vol\left(C_{\boldsymbol{x}}\right)=1$ for all $\boldsymbol{x}\in\mathcal{I}\left(\mathcal{B}_{R}^{d}\right)$
	and that for $\boldsymbol{x},\boldsymbol{x}'\in\mathcal{I}\left(\mathcal{B}_{R}^{d}\right)$
	such that $\boldsymbol{x}\neq\boldsymbol{x}'$, $C_{\boldsymbol{x}}\cap C_{\boldsymbol{x}'}$
	is a set of measure zero, hence: 
	\begin{align*}
	\mathcal{N}\left(\mathcal{B}_{R}^{d}\right) & =\left|\mathcal{I}\left(\mathcal{B}_{R}^{d}\right)\right|\\
	& =Vol\left(\bigcup_{\boldsymbol{x}\in\mathcal{I}\left(\mathcal{B}_{R}^{d}\right)}C_{\boldsymbol{x}}\right)\\
	& \leq Vol\left(\mathcal{B}_{R+\frac{\sqrt{d}}{2}}^{d}\right)\\
	& =\frac{\pi^{\frac{d}{2}}}{\Gamma\left(\frac{d}{2}+1\right)}\left(R+\frac{\sqrt{d}}{2}\right)^{d}
	\end{align*}
	Assume for convenience that $d\mod2\equiv0$, so $\Gamma\left(\frac{d}{2}+1\right)=\left(\frac{d}{2}\right)!$,
	and Stirling's approximation yields:
	\[
	\simeq\frac{\pi^{\frac{d}{2}}}{\sqrt{\pi d}\cdot\left(\frac{d}{2e}\right)^{\frac{d}{2}}}\left(R+\frac{\sqrt{d}}{2}\right)^{d}=\frac{\left(\frac{\pi e}{2}\right)^{\frac{d}{2}}}{\sqrt{\pi d}}\left(\frac{2R}{\sqrt{d}}+1\right)^{d}
	\]
	On the other hand, note that $\mathcal{B}_{R-\frac{\sqrt{d}}{2}}^{d}\subseteq\bigcup_{\boldsymbol{x}\in\mathcal{I}\left(\mathcal{B}_{R}^{d}\right)}C_{\boldsymbol{x}}$,
	and therefore:
	\begin{align*}
	\mathcal{N}\left(\mathcal{B}_{R}^{d}\right) & =Vol\left(\bigcup_{\boldsymbol{x}\in\mathcal{I}\left(\mathcal{B}_{R}^{d}\right)}C_{\boldsymbol{x}}\right)\\
	& \geq Vol\left(\mathcal{B}_{R-\frac{\sqrt{d}}{2}}^{d}\right)\\
	& =\frac{\pi^{\frac{d}{2}}}{\Gamma\left(\frac{d}{2}+1\right)}\left(R-\frac{\sqrt{d}}{2}\right)^{d}\\
	& \simeq\frac{\left(\frac{\pi e}{2}\right)^{\frac{d}{2}}}{\sqrt{\pi d}}\left(\frac{2R}{\sqrt{d}}-1\right)^{d}
	\end{align*}
\end{proof}
\begin{lemma}
	\label{lem:Characterizaion of the set of non-negligible multinomial coefficients}Let
	$K,M$ be two fixed natural numbers, $s\in\left(0,1\right)$ a constant
	sensitivity parameter, and let''
	\[
	T_{K,M}\coloneqq\left\{ \boldsymbol{a}\in\mathbb{N}^{M}\left|\binom{K}{a_{1},\ldots,a_{M}}\geq s\cdot\binom{K}{\frac{K}{M},\ldots,\frac{K}{M}}\right.\right\} 
	\]
	Then:
	\begin{align*}
	& \left\{ \boldsymbol{a}\in\mathbb{N}^{M}\left|\boldsymbol{a}^{T}\vec{\boldsymbol{1}}=K,\left\Vert \boldsymbol{a}-\frac{K}{M}\cdot\vec{\boldsymbol{1}}\right\Vert ^{2}\leq\frac{K}{M}\ln\left(s^{-1}\right)\right.\right\} \\
	& \subseteq T_{K,M}\\
	& \subseteq\left\{ \boldsymbol{a}\in\mathbb{N}^{M}\left|\boldsymbol{a}^{T}\vec{\boldsymbol{1}}=K,\left\Vert \boldsymbol{a}-\frac{K}{M}\cdot\vec{\boldsymbol{1}}\right\Vert ^{2}\leq4K\ln\left(s^{-1}\right)\right.\right\} 
	\end{align*}
\end{lemma}
\begin{proof}
	By Stirling's approximation, it holds that:
	\[
	x!\simeq\sqrt{2\pi x}\cdot\left(\frac{x}{e}\right)^{x}
	\]
	and a slightly more accurate version is:
	\[
	x!\simeq\sqrt{2\pi\left(x+\frac{1}{6}\right)}\cdot\left(\frac{x}{e}\right)^{x}
	\]
	and by plugging this approximation to the definition of a multinomial
	coefficient we get after some rearranging:
	\[
	\binom{K}{a_{1},\ldots,a_{M}}\simeq\left(2\pi\right)^{\frac{1-M}{2}}\cdot\frac{\sqrt{K+\frac{1}{6}}}{\prod_{i=1}^{M}\sqrt{a_{i}+\frac{1}{6}}}\cdot\frac{M^{K}}{\exp\left(\sum_{i=1}^{M}a_{i}\ln\left(\frac{M}{K}a_{i}\right)\right)}
	\]
	Note that using this approximation and some rearrangements, we get
	that the following are equivalent:
	\begin{align}
	& \binom{K}{a_{1},\ldots,a_{M}}\geq s\cdot\binom{K}{\frac{K}{M},\ldots,\frac{K}{M}}\nonumber \\
	\iff & \left(2\pi\right)^{\frac{1-M}{2}}\cdot\frac{\sqrt{K+\frac{1}{6}}}{\prod_{i=1}^{M}\sqrt{a_{i}+\frac{1}{6}}}\cdot\frac{M^{K}}{\exp\left(\sum_{i=1}^{M}a_{i}\ln\left(\frac{M}{K}a_{i}\right)\right)}\nonumber \\
	& \geq s\cdot\left(2\pi\right)^{\frac{1-M}{2}}\cdot\frac{\sqrt{K+\frac{1}{6}}}{\prod_{i=1}^{M}\sqrt{\frac{K}{M}+\frac{1}{6}}}\cdot\frac{M^{K}}{\exp\left(\sum_{i=1}^{M}\frac{K}{M}\ln\left(\frac{M}{K}\cdot\frac{K}{M}\right)\right)}\\
	\iff & \left(\prod_{i=1}^{M}\sqrt{a_{i}+\frac{1}{6}}\right)\cdot\exp\left(\sum_{i=1}^{M}a_{i}\ln\left(\frac{M}{K}a_{i}\right)\right)\leq s^{-1}\cdot\left(\frac{K}{M}+\frac{1}{6}\right)^{\frac{M}{2}}\label{eq:T_=00007BK,M=00007D inequality}
	\end{align}
	so we can characterize the set: 
	\[
	\tilde{T}_{K,M}\coloneqq\left\{ \boldsymbol{a}\in\mathbb{N}^{M}\left|{\scriptstyle \boldsymbol{a}^{T}\vec{\boldsymbol{1}}=K,\left(\prod_{i=1}^{M}\sqrt{a_{i}+\frac{1}{6}}\right)\cdot\exp\left(\sum_{i=1}^{M}a_{i}\ln\left(\frac{M}{K}a_{i}\right)\right)\leq s^{-1}\cdot\left(\frac{K}{M}+\frac{1}{6}\right)^{\frac{M}{2}}}\right.\right\} 
	\]
	as $\tilde{T}_{K,M}\simeq T_{K,M}$.
	
	We'll start by finding a condition that will assure us that $\left(a_{1},\ldots,a_{M}\right)\in T_{K,M}$
	(i.e., we will characterize a subset of $T_{K,M}$).
	
	First, note that by the AM-GM inequality, 
	\[
	\sqrt[M]{\prod_{i=1}^{M}\left(a_{i}+\frac{1}{6}\right)}\leq\frac{\sum_{i=1}^{M}\left(a_{i}+\frac{1}{6}\right)}{M}=\frac{K+\frac{1}{6}\cdot M}{M}=\frac{K}{M}+\frac{1}{6}
	\]
	and therefore:
	\[
	\prod_{i=1}^{M}\sqrt{a_{i}+\frac{1}{6}}=\left(\sqrt[M]{\prod_{i=1}^{M}\left(a_{i}+\frac{1}{6}\right)}\right)^{\frac{M}{2}}\leq\left(\frac{K}{M}+\frac{1}{6}\right)^{\frac{M}{2}}
	\]
	Since we're interested in a subset of $T_{K,M}$, we can show that
	the last inequality holds when we replace the first term in the left-hand
	side ($\prod_{i=1}^{M}\sqrt{a_{i}+\frac{1}{6}}$) with $\left(\frac{K}{M}+\frac{1}{6}\right)^{\frac{M}{2}}$
	(if the new inequality holds, (\ref{eq:T_=00007BK,M=00007D inequality})
	must hold as well), and we're left with:
	\begin{align}
	& \left(\frac{K}{M}+\frac{1}{6}\right)^{\frac{M}{2}}\cdot\exp\left(\sum_{i=1}^{M}a_{i}\ln\left(\frac{M}{K}a_{i}\right)\right)\leq s^{-1}\cdot\left(\frac{K}{M}+\frac{1}{6}\right)^{\frac{M}{2}}\nonumber \\
	\iff & \sum_{i=1}^{M}a_{i}\ln\left(\frac{M}{K}a_{i}\right)\leq\ln\left(s^{-1}\right)\label{eq:T_=00007BK,M=00007D subset with a's}
	\end{align}
	Now, for $i\in\left[M\right]$, let $h_{i}\coloneqq\frac{M}{K}a_{i}-1$,
	so:
	\[
	\sum_{i=1}^{M}a_{i}\ln\left(\frac{M}{K}a_{i}\right)=\frac{K}{M}\sum_{i=1}^{M}\left(1+h_{i}\right)\ln\left(1+h_{i}\right)
	\]
	and the last inequality becomes:
	\begin{align*}
	& \frac{K}{M}\sum_{i=1}^{M}\left(1+h_{i}\right)\ln\left(1+h_{i}\right)\leq\ln\left(s^{-1}\right)\\
	\iff & \sum_{i=1}^{M}\left(1+h_{i}\right)\ln\left(1+h_{i}\right)\leq\frac{M}{K}\ln\left(s^{-1}\right)
	\end{align*}
	Observe that for all $x\geq-1$, it holds that:
	\[
	\ln\left(1+x\right)\leq x
	\]
	and therefore:
	\[
	\left(1+x\right)\ln\left(1+x\right)\leq x^{2}+x
	\]
	so it suffices (again, we're only interested in a subset of $T_{K,M}$)
	to show that:
	\begin{equation}
	\sum_{i=1}^{M}\left(h_{i}^{2}+h_{i}\right)\leq\frac{M}{K}\ln\left(s^{-1}\right)\label{eq:T_=00007BK,M=00007D subset with h's}
	\end{equation}
	Now, observe that:
	\begin{align*}
	\sum_{i=1}^{M}h_{i} & =\sum_{i=1}^{M}\left(\frac{M}{K}a_{i}-1\right)\\
	& =\frac{M}{K}\overbrace{\sum_{i=1}^{M}a_{i}}^{=K}-M\\
	& =0
	\end{align*}
	so (\ref{eq:T_=00007BK,M=00007D subset with h's}) becomes:
	\begin{align*}
	& \sum_{i=1}^{M}h_{i}^{2}\leq\frac{M}{K}\ln\left(s^{-1}\right)\\
	\iff & \sum_{i=1}^{M}\left(a_{i}-\frac{K}{M}\right)^{2}\leq\frac{K}{M}\ln\left(s^{-1}\right)
	\end{align*}
	and we get that:
	\[
	\left\{ \boldsymbol{a}\in\mathbb{N}^{M}\left|\boldsymbol{a}^{T}\vec{\boldsymbol{1}}=K,\left\Vert \boldsymbol{a}-\frac{K}{M}\cdot\vec{\boldsymbol{1}}\right\Vert ^{2}\leq\frac{K}{M}\ln\left(s^{-1}\right)\right.\right\} \subseteq T_{K,M}
	\]
	
	Let us now turn to finding a condition that will assure us that $\left(a_{1},\ldots,a_{M}\right)\notin T_{K,M}$
	(i.e., we will characterize a subset of the complement of $T_{K,M}$).
	Note that:
	\begin{align*}
	& \left(\prod_{i=1}^{M}\sqrt{a_{i}+\frac{1}{6}}\right)\cdot\exp\left(\sum_{i=1}^{M}a_{i}\ln\left(\frac{M}{K}a_{i}\right)\right)\\
	& =\left(\frac{K}{M}+\frac{1}{6}\right)^{\frac{M}{2}}\cdot\exp\left(\sum_{i=1}^{M}\left[a_{i}\ln\left(\frac{M}{K}a_{i}\right)+\frac{1}{2}\ln\left(\frac{6Ma_{i}+M}{6K+M}\right)\right]\right)
	\end{align*}
	so if $\left(a_{1},\ldots,a_{M}\right)\notin T_{K,M}$, we must have
	that:
	\begin{align*}
	& \left(\frac{K}{M}+\frac{1}{6}\right)^{\frac{M}{2}}\cdot\exp\left(\sum_{i=1}^{M}\left[a_{i}\ln\left(\frac{M}{K}a_{i}\right)+\frac{1}{2}\ln\left(\frac{6Ma_{i}+M}{6K+M}\right)\right]\right)\\
	& >s^{-1}\cdot\left(\frac{K}{M}+\frac{1}{6}\right)^{\frac{M}{2}}\\
	\iff & \sum_{i=1}^{M}\left[a_{i}\ln\left(\frac{M}{K}a_{i}\right)+\frac{1}{2}\ln\left(\frac{6Ma_{i}+M}{6K+M}\right)\right]>\ln\left(s^{-1}\right)
	\end{align*}
	and using the same definition of $h_{i}$ as before, we get:
	\begin{align*}
	& \sum_{i=1}^{M}\left[a_{i}\ln\left(\frac{M}{K}a_{i}\right)+\frac{1}{2}\ln\left(\frac{6Ma_{i}+M}{6K+M}\right)\right]\\
	& =\sum_{i=1}^{M}\left[\frac{K}{M}\left(1+h_{i}\right)\ln\left(1+h_{i}\right)+\frac{1}{2}\ln\left(\frac{6K\left(1+h_{i}\right)+M}{6K+M}\right)\right]\\
	& =\sum_{i=1}^{M}\left[\frac{K}{M}\left(1+h_{i}\right)\ln\left(1+h_{i}\right)\right.\\
	& \left.+\frac{1}{2}\ln\left(\frac{6K\left(1+h_{i}\right)+M}{6K+M}\right)-\left(\frac{K}{M}+\frac{3K}{6K+M}\right)h_{i}\right]
	\end{align*}
	where the last equality is due to the fact that $\sum_{i=1}^{M}h_{i}=0$.
	
	Observing the first order Taylor polynomial of the function:
	\[
	f\left(x\right)=\frac{K}{M}\left(1+x\right)\ln\left(1+x\right)+\frac{1}{2}\ln\left(\frac{6K\left(1+x\right)+M}{6K+M}\right)-\left(\frac{K}{M}+\frac{3K}{6K+M}\right)x
	\]
	at $x=0$ with the remainder in the Lagrange form, we get:
	\begin{align}
	f\left(h_{i}\right) & =\left(\frac{K}{M}+\frac{3K}{6K+M}\right)h_{i}+f''\left(\xi_{i}\right)\cdot\frac{h_{i}^{2}}{2}-\left(\frac{K}{M}+\frac{3K}{6K+M}\right)h_{i}\nonumber \\
	& =f''\left(\xi_{i}\right)\cdot\frac{h_{i}^{2}}{2}\nonumber \\
	& =\left(\frac{K}{M\left(1+\xi_{i}\right)}-\frac{18K^{2}}{\left(6K\left(1+\xi_{i}\right)+M\right)^{2}}\right)\cdot\frac{h_{i}^{2}}{2}\label{eq:f(h_i)}
	\end{align}
	for some $\xi_{i}$ between $0$ and $h_{i}$.
	
	Note that $f''\left(\xi_{i}\right)$ is monotonically decreasing with
	$\xi_{i}$ for $\xi_{i}\in\left(-1,M-1\right]$, and using this fact
	and the fact that $K\geq1$, (\ref{eq:f(h_i)}) is lower bounded by:
	\[
	\frac{K}{4M^{2}}\cdot h_{i}^{2}
	\]
	So we can limit ourselves to looking at the cases where:
	\begin{align*}
	& \sum_{i=1}^{M}\frac{K}{4M^{2}}\cdot h_{i}^{2}>\ln\left(s^{-1}\right)\\
	\iff & \sum_{i=1}^{M}\left(a_{i}-\frac{K}{M}\right)^{2}>4K\ln\left(s^{-1}\right)
	\end{align*}
	and we get that:
	\begin{align*}
	& \left\{ \boldsymbol{a}\in\mathbb{N}^{M}\left|\boldsymbol{a}^{T}\vec{\boldsymbol{1}}=K,\left\Vert \boldsymbol{a}-\frac{K}{M}\cdot\vec{\boldsymbol{1}}\right\Vert ^{2}>4K\ln\left(s^{-1}\right)\right.\right\} \\
	& \subseteq\left\{ \boldsymbol{a}\in\mathbb{N}^{M}\left|\mathbf{a}^{T}\vec{\boldsymbol{1}}=K\right.\right\} \backslash T_{K,M}\\
	\iff & T_{K,M}\subseteq\left\{ \boldsymbol{a}\in\mathbb{N}^{M}\left|\boldsymbol{a}^{T}\vec{\boldsymbol{1}}=K,\left\Vert \boldsymbol{a}-\frac{K}{M}\cdot\vec{\boldsymbol{1}}\right\Vert ^{2}\leq4K\ln\left(s^{-1}\right)\right.\right\} 
	\end{align*}
	
	Combining the two results together we get:
	\begin{align*}
	& \left\{ \boldsymbol{a}\in\mathbb{N}^{M}\left|\boldsymbol{a}^{T}\vec{\boldsymbol{1}}=K,\left\Vert \boldsymbol{a}-\frac{K}{M}\cdot\vec{\boldsymbol{1}}\right\Vert ^{2}\leq\frac{K}{M}\ln\left(s^{-1}\right)\right.\right\} \\
	& \subseteq T_{K,M}\\
	& \subseteq\left\{ \boldsymbol{a}\in\mathbb{N}^{M}\left|\boldsymbol{a}^{T}\vec{\boldsymbol{1}}=K,\left\Vert \boldsymbol{a}-\frac{K}{M}\cdot\vec{\boldsymbol{1}}\right\Vert ^{2}\leq4K\ln\left(s^{-1}\right)\right.\right\} 
	\end{align*}
\end{proof}
\begin{lemma}
	\label{lem:Number of non-negligible multinomial coefficients}Let
	$K,M$ be two fixed natural numbers, and $s\in\left(0,1\right)$ a
	constant sensitivity parameter. Then the number of multinomial coefficients,
	$\binom{K}{a_{1},\ldots,a_{M}}$, which uphold: 
	\[
	\binom{K}{a_{1},\ldots,a_{M}}\geq s\cdot\binom{K}{\frac{K}{M},\ldots,\frac{K}{M}}
	\]
	is upper bounded by: 
	\[
	\frac{\left(\frac{\pi e}{2}\right)^{\frac{M-1}{2}}}{\left(M-1\right)\sqrt{\pi}}\left(4\sqrt{\frac{K\ln\left(s^{-1}\right)}{M-1}}+1\right)^{M-1}
	\]
	and lower bounded by: 
	\[
	\frac{\left(\frac{\pi e}{2}\right)^{\frac{M-1}{2}}}{M\sqrt{\pi}}\left(2\frac{\sqrt{K\ln\left(s^{-1}\right)}}{M}-1\right)^{M-1}
	\]
\end{lemma}
\begin{proof}
	Let $\binom{K}{a_{1},\ldots,a_{M}}$ be a multinomial coefficient
	for which it holds that:
	\[
	\binom{K}{a_{1},\ldots,a_{M}}\geq s\cdot\binom{K}{\frac{K}{M},\ldots,\frac{K}{M}}
	\]
	and denote: 
	\begin{align*}
	T_{K,M} & \coloneqq\left\{ \boldsymbol{a}\in\mathbb{N}^{M}\left|\binom{K}{a_{1},\ldots,a_{M}}\geq s\cdot\binom{K}{\frac{K}{M},\ldots,\frac{K}{M}}\right.\right\} \\
	T_{K,M}^{L} & \coloneqq\left\{ \boldsymbol{a}\in\mathbb{N}^{M}\left|\boldsymbol{a}^{T}\vec{\boldsymbol{1}}=K,\left\Vert \boldsymbol{a}-\frac{K}{M}\cdot\vec{\boldsymbol{1}}\right\Vert ^{2}\leq\frac{K}{M}\ln\left(s^{-1}\right)\right.\right\} \\
	T_{K,M}^{U} & \coloneqq\left\{ \boldsymbol{a}\in\mathbb{N}^{M}\left|\boldsymbol{a}^{T}\vec{\boldsymbol{1}}=K,\left\Vert \boldsymbol{a}-\frac{K}{M}\cdot\vec{\boldsymbol{1}}\right\Vert ^{2}\leq4K\ln\left(s^{-1}\right)\right.\right\} 
	\end{align*}
	then by lemma \ref{lem:Characterizaion of the set of non-negligible multinomial coefficients},
	\[
	T_{K,M}^{L}\subseteq T_{K,M}\subseteq T_{K,M}^{U}
	\]
	and therefore:
	\[
	\left|T_{K,M}^{L}\right|\leq\left|T_{K,M}\right|\leq\left|T_{K,M}^{U}\right|
	\]
	
	so in order to bound the cardinality of $T_{K,M}$ (which is the quantity
	we are interested in), we can find an upper bound on the cardinality
	of $T_{K,M}^{U}$ and a lower bound on the cardinality of $T_{K,M}^{L}$.
	
	Let $B\in\left\{ L,U\right\} $ and $\boldsymbol{a}\in T_{K,M}^{B}$,
	and denote: 
	\[
	R\left(B\right)\coloneqq\begin{cases}
	\sqrt{\frac{K}{M}\ln\left(s^{-1}\right)} & \text{if }B=L\\
	\sqrt{4K\ln\left(s^{-1}\right)} & \text{if }B=U
	\end{cases}
	\]
	For $i\in\left[M\right]$, denote $x_{i}\coloneqq a_{i}-\frac{K}{M}$.
	So the problem has changed to finding the number of integer $M$-tuples
	$x_{1},\ldots,x_{M}$ such that $\sum_{i=1}^{M}x_{i}=0$ and $\left\Vert \boldsymbol{x}\right\Vert \leq R\left(B\right)$,
	which is the number of integer lattice points $\boldsymbol{x}$ in
	the zero-centered $M$-dimensional ball of radius $R\left(B\right)$
	that uphold $\sum_{i=1}^{M}x_{i}=0$.
	
	Note that the intersection of a $d$-dimensional ball of radius $R$
	with the hyperplane $\mathcal{H}=\left\{ \boldsymbol{y}\in\mathbb{R}^{d}\left|\left\langle \boldsymbol{y},\vec{\boldsymbol{1}}\right\rangle =0\right.\right\} $
	is a $\left(d-1\right)$-dimensional ball, so we can assume that the
	number of integer lattice points in the zero-centered $d$-dimensional
	ball of radius $R$ whose coordinates add-up to zero is $\sim\frac{1}{\sqrt{d}}\cdot\mathcal{N}\left(\mathcal{B}_{R}^{d-1}\right)$,
	where $\frac{1}{\sqrt{d}}$ is the cosine of the angle between $\mathcal{H}$
	and one of the axis-aligned $\left(d-1\right)$-dimensional hyperplanes
	in $\mathbb{R}^{d}$. In our case, $R=R\left(B\right)$ and $d=M$,
	so by lemma \ref{lem:Number of integer lattice points in a high-dimensional ball},
	the cardinality of $T_{K,M}^{U}$ is upper bounded by:
	\begin{align*}
	\frac{1}{\sqrt{M}}\cdot\mathcal{N}\left(\mathcal{B}_{\sqrt{4K\ln\left(s^{-1}\right)}}^{M-1}\right) & \leq\frac{1}{\sqrt{M}}\cdot\frac{\left(\frac{\pi e}{2}\right)^{\frac{M-1}{2}}}{\sqrt{\pi\left(M-1\right)}}\left(\frac{2\sqrt{4K\ln\left(s^{-1}\right)}}{\sqrt{M-1}}+1\right)^{M-1}\\
	& \leq\frac{\left(\frac{\pi e}{2}\right)^{\frac{M-1}{2}}}{\left(M-1\right)\sqrt{\pi}}\left(4\sqrt{\frac{K\ln\left(s^{-1}\right)}{M-1}}+1\right)^{M-1}
	\end{align*}
	and the cardinality of $T_{K,M}^{L}$ is lower bounded by:
	\begin{align*}
	\frac{1}{\sqrt{M}}\cdot\mathcal{N}\left(\mathcal{B}_{\sqrt{\frac{K}{M}\ln\left(s^{-1}\right)}}^{M-1}\right) & \geq\frac{1}{\sqrt{M}}\cdot\frac{\left(\frac{\pi e}{2}\right)^{\frac{M-1}{2}}}{\sqrt{\pi\left(M-1\right)}}\left(\frac{2\sqrt{\frac{K}{M}\ln\left(s^{-1}\right)}}{\sqrt{M-1}}-1\right)^{M-1}\\
	& \geq\frac{\left(\frac{\pi e}{2}\right)^{\frac{M-1}{2}}}{M\sqrt{\pi}}\left(2\frac{\sqrt{K\ln\left(s^{-1}\right)}}{M}-1\right)^{M-1}
	\end{align*}
\end{proof}
\begin{lemma}
	\label{lem:Number of non-negligible =00005Cbinom=00007BK=00007D=00007Bn=00007D=00005Ceta^=00007Bn=00007D terms}Let
	$K$ be a fixed natural number, $\eta\in\left(0,1\right]$, and $s\in\left(0,1\right)$
	a constant sensitivity parameter.Then number of integer $n$s such
	that $\binom{K}{n}\eta^{n}\geq s\cdot\binom{K}{\frac{\eta K}{1+\eta}}\eta^{\frac{\eta K}{1+\eta}}$
	is upper bounded by:
	\[
	\frac{K\sqrt{\left(1+2\ln\left(s\right)\right)^{2}\left(1+\eta\right)^{2}-4\eta}}{2\left(1+\eta\right)\left(\ln\left(s^{-1}\right)-1\right)}
	\]
\end{lemma}
\begin{proof}
	Denote $n\coloneqq\frac{\eta K}{1+\eta}+x$, so:
	\begin{align*}
	& \binom{K}{n}\eta^{n}\geq s\cdot\binom{K}{\frac{\eta K}{1+\eta}}\eta^{\frac{\eta K}{1+\eta}}\\
	\iff & \binom{K}{\frac{\eta K}{1+\eta}+x}\eta^{x}\geq s\cdot\binom{K}{\frac{\eta K}{1+\eta}}
	\end{align*}
	Recall that by Stirling's approximation we know that:
	\[
	\binom{K}{n}\simeq\sqrt{\frac{K}{2\pi n\left(K-n\right)}}\cdot\frac{K^{K}}{n^{n}\left(K-n\right)^{K-n}}
	\]
	So the number $x$s which uphold:
	\begin{align}
	& \sqrt{\frac{K}{2\pi\left(\frac{\eta K}{1+\eta}+x\right)\left(\frac{K}{1+\eta}-x\right)}}\cdot\frac{K^{K}}{\left(\frac{\eta K}{1+\eta}+x\right)^{\frac{\eta K}{1+\eta}+x}\left(\frac{K}{1+\eta}-x\right)^{\frac{K}{1+\eta}-x}}\eta^{x}\nonumber \\
	& \geq s\cdot\sqrt{\frac{K}{2\pi\left(\frac{\eta K}{1+\eta}\right)\left(\frac{K}{1+\eta}\right)}}\cdot\frac{K^{K}}{\left(\frac{\eta K}{1+\eta}\right)^{\frac{\eta K}{1+\eta}}\left(\frac{K}{1+\eta}\right)^{\frac{K}{1+\eta}}}\nonumber \\
	\iff & \begin{alignedat}{1} & \ln\left(s^{-1}\cdot\left(\frac{\eta K}{1+\eta}\right)^{\frac{\eta K}{1+\eta}+\frac{1}{2}}\left(\frac{K}{1+\eta}\right)^{\frac{K}{1+\eta}+\frac{1}{2}}\right)\\
	& \geq\ln\left(\left(\frac{\eta K}{1+\eta}+x\right)^{\frac{\eta K}{1+\eta}+x+\frac{1}{2}}\left(\frac{K}{1+\eta}-x\right)^{\frac{K}{1+\eta}-x+\frac{1}{2}}\cdot\eta^{-x}\right)
	\end{alignedat}
	\label{eq:binomial eta ineq}
	\end{align}
	approximates the number we are trying to quantify.
	
	After some rearranging, one can observe that (\ref{eq:binomial eta ineq})
	is equivalent to:
	\begin{equation}
	\begin{alignedat}{1}\ln\left(s^{-1}\right) & \geq\left(\frac{\eta K}{1+\eta}+x+\frac{1}{2}\right)\ln\left(1+\frac{\left(1+\eta\right)x}{\eta K}\right)\\
	& +\left(\frac{K}{1+\eta}-x+\frac{1}{2}\right)\ln\left(1-\frac{\left(1+\eta\right)x}{K}\right)
	\end{alignedat}
	\label{eq:binomial eta ineq alt}
	\end{equation}
	
	and since for all $t>-1$, $\frac{t}{1+t}\leq\ln\left(1+t\right)$,
	the number of $x$s which uphold (\ref{eq:binomial eta ineq alt})
	is upper bounded by the number of integer $x$s for which it holds
	that:
	\begin{align}
	& \ln\left(s^{-1}\right)\geq\frac{\left(1-\eta^{2}\right)Kx-2\left(1+\eta\right)^{2}x^{2}}{2\eta K^{2}+2\left(1-\eta^{2}\right)Kx-2\left(1+\eta\right)^{2}x^{2}}\nonumber \\
	& \begin{alignedat}{1}\iff & 2\left(1+\eta\right)^{2}\left(\ln\left(s^{-1}\right)-1\right)x^{2}\\
	& +\left(1-2\ln\left(s^{-1}\right)\right)\left(1-\eta^{2}\right)Kx-2\ln\left(s^{-1}\right)\eta K^{2}\leq0
	\end{alignedat}
	\label{eq:x's upper bound ineq}
	\end{align}
	
	Recall that for inequalities of the form $ax^{2}+bx+c\leq0$ where
	$a>0$, the set of all values of $x$ which satisfy this inequality
	is $\left(\frac{-b-\sqrt{b^{2}-4ac}}{2a},\frac{-b+\sqrt{b^{2}-4ac}}{2a}\right)$
	and the number of integer values of $x$ which satisfy this condition
	is approximately $\frac{\sqrt{b^{2}-4ac}}{a}$ (the interval's length).
	
	In our case, (\ref{eq:x's upper bound ineq}) gives us: 
	\begin{align*}
	a & =2\left(1+\eta\right)^{2}\left(\ln\left(s^{-1}\right)-1\right)\\
	b & =\left(1-2\ln\left(s^{-1}\right)\right)\left(1-\eta^{2}\right)K\\
	c & =-2\ln\left(s^{-1}\right)\eta K^{2}
	\end{align*}
	
	and therefore:
	\[
	\frac{\sqrt{b^{2}-4ac}}{a}=\frac{K\sqrt{\left(2\ln\left(s^{-1}\right)-1\right)^{2}\left(1+\eta\right)^{2}-4\eta}}{2\left(1+\eta\right)\left(\ln\left(s^{-1}\right)-1\right)}
	\]
\end{proof}
\begin{lemma}
	\label{lem:An upper bound on the number of non-negligible summands}Let
	$K,M$ be two fixed natural numbers, $\eta\in\left(0,1\right]$, and
	$s\in\left(0,e^{-1.5}\right]$ a constant sensitivity parameter. Denote:
	\[
	D_{K,M}\coloneqq\left\{ \left(n,\boldsymbol{a},\boldsymbol{b}\right)\in\mathbb{N}^{2M+1}\left|\substack{0\leq n\leq K\\
		a_{1}+\ldots+a_{M}=n\\
		b_{1}+\ldots+b_{M}=K-n
	}
	\right.\right\} 
	\]
	define $F:D_{K,M}\longrightarrow\mathbb{R}$ as:
	\[
	F\left(n,\boldsymbol{a},\boldsymbol{b}\right)=\binom{K}{n}\eta^{n}\binom{n}{a_{1},\ldots,a_{M}}\binom{K-n}{b_{1},\ldots,b_{M}}
	\]
	and let $\mathbf{x^{\star}}\coloneqq\underset{\mathbf{x}\in D_{K,M}}{\arg\max}F\left(\mathbf{x}\right)$.
	If $\frac{\eta K}{1+\eta}\geq\left(M-1\right)$, then the number of
	$\mathbf{x}\in D_{K,M}$ which uphold $F\left(\mathbf{x}\right)\geq s\cdot F\left(\mathbf{x^{\star}}\right)$
	is bounded from above by:
	\[
	\frac{2K}{\left(M-1\right)^{2}\pi}\cdot\left(\frac{25\pi eK\ln\left(s^{-1}\right)\sqrt{\eta}}{2\left(M-1\right)\left(1+\eta\right)}\right)^{M-1}
	\]
\end{lemma}
\begin{proof}
	By lemma \ref{lem:Maximum of =00007BK =00005Cchoose n=00007D=00005Ceta^=00007Bn=00007D=00005Cbinom=00007Bn=00007D=00007Ba_=00007B1=00007D,=00005Cldots,a_=00007BM=00007D=00007D=00005Cbinom=00007BK-n=00007D=00007Bb_=00007B1=00007D,=00005Cldots,b_=00007BM=00007D=00007D},
	$\mathbf{x^{\star}}\simeq\left(\frac{\eta K}{M\left(1+\eta\right)},\frac{\eta K}{M\left(1+\eta\right)},\ldots,\frac{\eta K}{M\left(1+\eta\right)},\frac{K}{M\left(1+\eta\right)},\ldots,\frac{K}{M\left(1+\eta\right)}\right)$.
	By lemma \ref{lem:Number of non-negligible =00005Cbinom=00007BK=00007D=00007Bn=00007D=00005Ceta^=00007Bn=00007D terms},
	the number of $n$s between $0\leq n\leq K$ such that $\binom{K}{n}\eta^{n}\geq s\cdot\binom{K}{\frac{\eta K}{1+\eta}}\eta^{\frac{\eta K}{1+\eta}}$
	is upper bounded by $\frac{K\sqrt{\left(2\ln\left(s^{-1}\right)-1\right)^{2}\left(1+\eta\right)^{2}-4\eta}}{2\left(1+\eta\right)\left(\ln\left(s^{-1}\right)-1\right)}$,
	by lemma \ref{lem:Number of non-negligible multinomial coefficients}
	the number of non-negative integer $M$-tuples $a_{1},\ldots,a_{M}$
	such that $a_{1}+\ldots+a_{M}=\frac{\eta K}{1+\eta}$ and $\binom{\frac{\eta K}{1+\eta}}{a_{1},\ldots,a_{M}}\geq s\cdot\binom{\frac{\eta K}{1+\eta}}{\frac{\eta K}{M\left(1+\eta\right)},\ldots,\frac{\eta K}{M\left(1+\eta\right)}}$
	is bounded from above by $\frac{\left(\frac{\pi e}{2}\right)^{\frac{M-1}{2}}}{\left(M-1\right)\sqrt{\pi}}\left(4\sqrt{\frac{\eta K\ln\left(s^{-1}\right)}{\left(M-1\right)\left(1+\eta\right)}}+1\right)^{M-1}$,
	and the number of non-negative integer $M$-tuples $b_{1},\ldots,b_{M}$
	such that $b_{1}+\ldots+b_{M}=\frac{K}{1+\eta}$ and $\binom{\frac{K}{1+\eta}}{b_{1},\ldots,b_{M}}\geq s\cdot\binom{\frac{K}{1+\eta}}{\frac{K}{M\left(1+\eta\right)},\ldots,\frac{K}{M\left(1+\eta\right)}}$
	is bounded from above by $\frac{\left(\frac{\pi e}{2}\right)^{\frac{M-1}{2}}}{\left(M-1\right)\sqrt{\pi}}\left(4\sqrt{\frac{K\ln\left(s^{-1}\right)}{\left(M-1\right)\left(1+\eta\right)}}+1\right)^{M-1}$.
	In total, without taking into consideration the interactions between
	the three multiplicands (so our bound is not tight), the number of
	$\mathbf{x}\in D_{K,M}$ which uphold $F\left(\mathbf{x}\right)\geq s\cdot F\left(\mathbf{x^{\star}}\right)$
	is upper bounded by:
	\begin{align*}
	& \underbrace{\frac{K\sqrt{\left(2\ln\left(s^{-1}\right)-1\right)^{2}\left(1+\eta\right)^{2}-4\eta}}{2\left(1+\eta\right)\left(\ln\left(s^{-1}\right)-1\right)}}_{\text{binomial}\,+\,\eta\text{-exponent}}\\
	& \cdot\underbrace{\frac{\left(\frac{\pi e}{2}\right)^{\frac{M-1}{2}}}{\left(M-1\right)\sqrt{\pi}}\left(4\sqrt{\frac{\eta K\ln\left(s^{-1}\right)}{\left(M-1\right)\left(1+\eta\right)}}+1\right)^{M-1}}_{a-\text{multinomial}}\\
	& \cdot\underbrace{\frac{\left(\frac{\pi e}{2}\right)^{\frac{M-1}{2}}}{\left(M-1\right)\sqrt{\pi}}\left(4\sqrt{\frac{K\ln\left(s^{-1}\right)}{\left(M-1\right)\left(1+\eta\right)}}+1\right)^{M-1}}_{b-\text{multinomial}}
	\end{align*}
	and since $\frac{\eta K}{\left(M-1\right)\left(1+\eta\right)}\geq1$
	and $s\leq e^{-1.5}$ (and hence $\frac{2\ln\left(s^{-1}\right)-1}{2\left(\ln\left(s^{-1}\right)-1\right)}\leq2$),
	this can be further bounded by:
	\[
	\frac{2K}{\left(M-1\right)^{2}\pi}\cdot\left(\frac{25\pi eK\ln\left(s^{-1}\right)\sqrt{\eta}}{2\left(M-1\right)\left(1+\eta\right)}\right)^{M-1}
	\]
\end{proof}
\begin{lemma}
	\label{lem:A lower bound on the number of non-negligible summands}Let
	$K,M$ be two fixed natural numbers, $\eta\in\left(0,1\right]$, and
	$s\in\left(0,e^{-1.5}\right]$ a constant sensitivity parameter. Denote:
	\[
	D_{K,M}\coloneqq\left\{ \left(n,\boldsymbol{a},\boldsymbol{b}\right)\in\mathbb{N}^{2M+1}\left|\substack{0\leq n\leq K\\
		a_{1}+\ldots+a_{M}=n\\
		b_{1}+\ldots+b_{M}=K-n
	}
	\right.\right\} 
	\]
	define $F:D_{K,M}\longrightarrow\mathbb{R}$ as:
	\[
	F\left(n,\boldsymbol{a},\boldsymbol{b}\right)=\binom{K}{n}\eta^{n}\binom{n}{a_{1},\ldots,a_{M}}\binom{K-n}{b_{1},\ldots,b_{M}}
	\]
	and let $\mathbf{x^{\star}}\coloneqq\underset{\mathbf{x}\in D_{K,M}}{\arg\max}F\left(\mathbf{x}\right)$.
	If $\frac{K}{1+\eta}\geq M^{2}$, then the number of $\mathbf{x}\in D_{K,M}$
	which uphold $F\left(\mathbf{x}\right)\geq s\cdot F\left(\mathbf{x^{\star}}\right)$
	is bounded from below by:
	\[
	\frac{1}{M\sqrt{\pi}}\left(\frac{\pi eK\ln\left(s^{-1}\right)}{2M^{2}\left(1+\eta\right)}\right)^{\frac{M-1}{2}}
	\]
\end{lemma}
\begin{proof}
	By lemma \ref{lem:Number of non-negligible multinomial coefficients},
	the number of non-negative integer $M$-tuples $b_{1},\ldots,b_{M}$
	such that $b_{1}+\ldots+b_{M}=\frac{K}{1+\eta}$ and $\binom{\frac{K}{1+\eta}}{b_{1},\ldots,b_{M}}\geq s\cdot\binom{\frac{K}{1+\eta}}{\frac{K}{M\left(1+\eta\right)},\ldots,\frac{K}{M\left(1+\eta\right)}}$
	is bounded from below by $\frac{\left(\frac{\pi e}{2}\right)^{\frac{M-1}{2}}}{M\sqrt{\pi}}\left(2\frac{\sqrt{\frac{K}{1+\eta}\ln\left(s^{-1}\right)}}{M}-1\right)^{M-1}$.
	Since these $b$s are only a subset of the elements of $D_{K,M}$
	which $F$ takes into consideration, this is also a (quite loose)
	lower bound on the number of $\mathbf{x}\in D_{K,M}$ which uphold
	$F\left(\mathbf{x}\right)\geq s\cdot F\left(\mathbf{x^{\star}}\right)$.
	
	Now, we know that:
	\[
	\frac{K}{1+\eta}\geq M^{2}
	\]
	and since $s\leq e^{-1.5}$, it holds that $\ln\left(s^{-1}\right)>1$,
	we get:
	\begin{align*}
	& \frac{K}{1+\eta}\ln\left(s^{-1}\right)>\frac{K}{1+\eta}\geq M^{2}\\
	\iff & -1>-\frac{\sqrt{\frac{K}{1+\eta}\ln\left(s^{-1}\right)}}{M}
	\end{align*}
	and therefore:
	\begin{align*}
	& \frac{\left(\frac{\pi e}{2}\right)^{\frac{M-1}{2}}}{M\sqrt{\pi}}\left(2\frac{\sqrt{\frac{K}{1+\eta}\ln\left(s^{-1}\right)}}{M}-1\right)^{M-1}\\
	& >\frac{\left(\frac{\pi e}{2}\right)^{\frac{M-1}{2}}}{M\sqrt{\pi}}\left(2\frac{\sqrt{\frac{K}{1+\eta}\ln\left(s^{-1}\right)}}{M}-\frac{\sqrt{\frac{K}{1+\eta}\ln\left(s^{-1}\right)}}{M}\right)^{M-1}\\
	& =\frac{1}{M\sqrt{\pi}}\left(\frac{\pi eK\ln\left(s^{-1}\right)}{2M^{2}\left(1+\eta\right)}\right)^{\frac{M-1}{2}}
	\end{align*}
\end{proof}

	\section{Lower bounds on the $\varepsilon$-separation rank}\label{lower_bound}
	\subsection{Preliminaries}
	\subsubsection{Tensors and their matricization}
	We begin by laying out basic concepts in tensor theory required for
	the upcoming analysis. The core concept of a \emph{tensor} may be thought of as a
	multi-dimensional array. The \emph{order} of a
	tensor is defined to be the number of indexing entries in the array,
	referred to as \emph{modes}. The \emph{dimension} of a tensor in a particular
	mode is defined as the number of values taken by the index in that
	mode. If $\A$ is a tensor of order $N$ and dimension $M_i$ in each mode
	$i\in[N]$, its entries are denoted $\A_{d_1...d_N}$, where the index in each
	mode takes values $d_i\in [M_i]$.
	
	We will make use of the concept of the \emph{matricization of $\A$ \wrt~the balanced partition $(P,Q)$}, denoted $\mat{\A}_{P,Q}\in\R^{M^{\nicefrac{N}{2}}\times M^{\nicefrac{N}{2}}}$, which is essentially the arrangement of the tensor elements as a matrix whose rows correspond to $P$ and columns to $Q$.
	Suppose $\A\in\R^{M{\times\cdots\times}M}$
	is a tensor of order $N$, and let $(P,Q)$ be a balanced partition of $[N]$, \ie~$P$
	and~$Q$ are disjoint size $\nicefrac{N}{2}$ subsets of $[N]$ whose union gives~$[N]$.
	The \emph{matricization of $\A$ \wrt~the partition $(P,Q)$}, denoted
	$\mat{\A}_{P,Q}$, is the $M^{{\nicefrac{N}{2}}}$-by-$M^{{\nicefrac{N}{2}}}$ matrix holding the entries of $\A$ such that $\A_{d_1{\ldots}d_N}$ is placed in row index $1+\sum_{t=1}^{{\nicefrac{N}{2}}}(d_{p_t}-1)M^{{\nicefrac{N}{2}}-t}$ and column index $1+\sum_{t=1}^{{\nicefrac{N}{2}}}(d_{q_t}-1)M^{{\nicefrac{N}{2}}-t}$.
	
	We now present the concept of grid tensors, which are a form of function discretization~\citep{hackbusch2012tensor}. Essentially, the function is
	evaluated for a set of points on an exponentially large grid in the
	input space and the outcomes are stored in a tensor. Formally, fixing a set of \emph{template} vectors
	$\x^{(1)},\ldots,\x^{(Z)}\in\left[V\right]$, the points on the grid are the set
	$\{(\x^{(d_1)},\ldots,\x^{(d_N)})\}_{d_1,\ldots,d_N=1}^Z$. Given a function
	$y(\x^1,\ldots,\x^N)$, the set of its values on the grid arranged in the form of a tensor are
	called the grid tensor induced by $y$, denoted
	$\A(y)_{d_1,\ldots,d_N} \equiv y(\x^1=\x^{(d_1)},\ldots,\x^N=\x^{(d_N)})$.
	
	\subsubsection{$\varepsilon$-rank} 
	We will make use of the concept of \emph{$\varepsilon$-rank} ~\cite{alon2013approximate} of a matrix $A$ defined for any
	$\varepsilon>0$ as the minimum rank over matrices that approximate every entry of $A$ to within an additive $\varepsilon$.
	We will prove lower bounds on the $\varepsilon$s for which the $\varepsilon$-rank a matrix remain high by the following lemma:
	\begin{lemma}\label{lemma:lambda_k_lower_bound}
		Let $M\in\R^{n\times n}$ be symmetric matrix and $\varepsilon>0$, then:
		\begin{equation}
			\forall k\le  n\quad\lambda_k\left(M\right)\ge\varepsilon\implies\frac{\varepsilon}{2n}\textrm{-}\rank{M}\ge k
		\end{equation}
	\end{lemma}
	\begin{proof}
		Let $E\in\left[\frac{-\varepsilon}{2n},\frac{\varepsilon}{2n}\right]^{n\times n}$, we need to prove that $\rank{M+E}\ge k$.  Since $M$ is symmetric, $M$ is diagonalizable with eigenvalues $\lambda_1\ge\lambda_2\ge\dots\ge\lambda_n$. Denote by $v_1,v_2,\dots,v_n$ the eigenvectors that are normalized according to the $l_1$ norm, then for any $i\le k$ we have that:
		\begin{align}
			\norm{\left(M+E\right)v_i}_1\ge\norm{Mv_i}_1-\norm{Ev_i}_1=\lambda_i-\norm{Ev_i}_1\ge\lambda_i-n\frac{\varepsilon}{2n}=\frac{\varepsilon}{2}>0
		\end{align}
		In particular for any $i\le k$ we have that $\left(M+E\right)v_i\ne0$ and since $v_1,v_2,\dots,v_k$ are linearly independent we conclude that $\rank{M+E}\ge k$.
	\end{proof}	

	Finally, we will use the following lemma for lower bounding the amount of small eigenvalues of symmetric matrices:	
	\begin{lemma}\label{lemma:number_of_high_eigenvalues_lower_bound}
		Let $M\in\R^{n\times n}$ be a symmetric matrix with diagonal entries that equal to $1$, and denote its eigenvalues by $\lambda_1\ge\lambda_2,\dots,\lambda_n$. Then:
		\begin{equation}\label{eq:number_of_high_eigenvalues_lower_bound}
			r\coloneqq\max\left\{k:\lambda_k\ge\frac{1}{n}\right\}\ge\frac{n-1}{\norm{M}_F}
		\end{equation}
	\end{lemma}
	\begin{proof}
		Since the trace of a matrix equals to both the sum of its eigenvalues and its diagonal entries, we get:
		\begin{align}
			n=\sum_{i=1}^{n}\left[M\right]_{ii}=\trace{M}=\sum_{i=1}^{n}\lambda_i\le r\lambda_1+\frac{n-r}{n}\le r\lambda_1+1
		\end{align}
		Eq.~\ref{eq:number_of_high_eigenvalues_lower_bound} follows from the fact that:
		\begin{align}
			\norm{M}_F=\sqrt{\lambda_1^2+\dots\lambda_n^2}\ge\lambda_1
		\end{align}
	\end{proof}	
	
	\subsubsection{High-Dimensional Spheres} 
	We will use the Lebesgue measure on the sphere for taking expectations on $\SSS^d$.
	For any measurable subset $A\subseteq\SSS^d$ this measure is defined as the $d+1$ dimensional volume of the "wedge" in the ball $\B^{d+1}$:
	\begin{equation}
		\mu\left(A\right)\coloneqq\frac{\lambda^{d+1}\left(\left\{ tx\,|\,x\in A,t\in\left[0,1\right]\right\} \right)}{\lambda^{d+1}\left(\B^{d+1}\right)}
	\end{equation}
	Where $\lambda^{d+1}$ denotes the Lebesgue measure on $\R^{d+1}$.
	We will also use an unnormalized version of this measure:
	\begin{equation}
		\mu_{\text{unnormalized}}\left(A\right)\coloneqq\mu\left(A\right)\cdot{\lambda^{d+1}\left(\B^{d+1}\right)}
	\end{equation}
	
	\cite{smith1989unitball} showed that $\mu_{\text{unnormalized}}\left(\SSS^d\right)$ is monotonically decreasing for $d>5$. The following lemma bounds the rate of this decrease:
	\begin{lemma}\label{lemma:high_dimensional_unit_balls_voulume_ratio}
		For any $d>5$ the following holds:
		\begin{equation}
			\frac{\mu_{\text{unnormalized}}\left(\SSS^{d-1}\right)}{\mu_{\text{unnormalized}}\left(\SSS^d\right)}\le\frac{d+1}{2\pi}
		\end{equation}
	\end{lemma}
	\begin{proof}
		Since $d>5$ we know that $\mu_{\text{unnormalized}}\left(\SSS^d\right)>\mu_{\text{unnormalized}}\left(\SSS^{d+1}\right)$ and therefore:
		\begin{equation}
			\frac{\mu_{\text{unnormalized}}\left(\SSS^{d-1}\right)}{\mu_{\text{unnormalized}}\left(\SSS^d\right)}\le\frac{\mu_{\text{unnormalized}}\left(\SSS^{d-1}\right)}{\mu_{\text{unnormalized}}\left(\SSS^{d+1}\right)}
		\end{equation}
		Finally, \cite{gipple2014volume} showed that $\frac{\mu_{\text{unnormalized}}\left(\SSS^{d-1}\right)}{\mu_{\text{unnormalized}}\left(\SSS^{d+1}\right)}=\frac{d+1}{2\pi}$, completing the proof.
	\end{proof}
	Finally, we will use a well known fact regarding the variation of the sphere volume for different radii (see for example~\cite{smith1989unitball}):
	\begin{fact}\label{fact:high_dimensional_balls_voulume_ratio}
		For any $d\in\N,\,R>0$ the following holds:
		\begin{equation}
			\frac{\mu_{\text{unnormalized}}\left(R\SSS^{d}\right)}{\mu_{\text{unnormalized}}\left(\SSS^d\right)}=R^{d+1}
		\end{equation}
	\end{fact}
	\subsection{Proof of the lower bound}\label{sec:lower_bound_proof}
	In this subsection, we prove theorem~\ref{theorem:tight} of the main text. We will follow the proofs of~\cite{levine2020limits,wies2021vocabulary}, with important adjustments to the $\varepsilon$-sequential-separation rank definition.
	
	We begin by showing that high \textit{$\varepsilon$-rank} \cite{alon2013approximate} of the grid tensor matricization implies high \textit{$\varepsilon$-sequential-separation rank} (see section~\ref{sec:2:2} ~of the main text) of the function. Essentially, we apply claim~1 from \cite{levine2020limits} to $\varepsilon$-approximations obtained from the \textit{$\varepsilon$-separation-rank} definition.
	This relation, which holds for all functions, is
	formulated below for functions realized by the analyzed Transformer network:
	
	\begin{lemma}\label{lemma:grid_sep_deep}
		For $y^{(p,i),L,d_x}_\text{\emph{in-context}}$ as defined in theorem~\ref{theorem:upper_bounds} of the main text.
		Let $\varepsilon\textrm{-}\mathrm{seq}\textrm{-}sep\left(y^{(p,i),L,d_x}_\text{\emph{in-context}}\right)$ denote its $\varepsilon$-sequential-separation rank. Then, for any integer $Z$, any $\varepsilon>0$, any set of template vectors $\x^{(1)},\ldots,\x^{(Z)}\in\R^{d_x}$ and any sub-matrix $M$ of $\mat{\A(\mathcal{Z}_{y^{(p,i),L,d_x}_\text{\emph{in-context}}})}_{\aaa,\bb}$ it holds that:
		\begin{equation}
			\varepsilon\textrm{-}\mathrm{seq}\textrm{-}sep\left(y^{(p,i),L,d_x}_\text{\emph{in-context}}\right)\geq \varepsilon\textrm{-}\rank{M},
		\end{equation}
		where $\A(\mathcal{Z}_{y^{(p,i),L,d_x}_\text{\emph{in-context}}})$ is the grid tensor of $\mathcal{Z}_{y^{(p,i),L,d_x}_\text{\emph{in-context}}}$ with
		respect to the above template vectors.
	\end{lemma}
	\begin{proof}
		If $\varepsilon\textrm{-}\mathrm{seq}\textrm{-}\mathrm{sep}\left(y^{(p,i),L,d_x}_\text{\emph{in-context}}\right)=\infty$ then the inequality is trivially satisfied. Otherwise, assume that $\varepsilon\textrm{-}\mathrm{seq}\textrm{-}\mathrm{sep}\left(y^{(p,i),L,d_x}_\text{\emph{in-context}}\right)=K\in\N$, and let $\tilde{y}$ be an $\varepsilon$-approximation for $\mathcal{Z}_{y^{(p,i),L,d_x}_\text{\emph{in-context}}}$ with $\mathrm{seq}\textrm{-}sep\left(\tilde{y}\right)=K$. By claim~1 of \cite{levine2020limits} we have that $\rank{\mat{\A(\tilde{y})}_{\aaa,\bb}}\leq K$. Denote by $\tilde{M}$  the sub-matrix of $\mat{\A(\tilde{y})}_{\aaa,\bb}$ that corresponds to the rows and columns in $M$.	Now, since $\tilde{y}$ is an $\varepsilon$-approximation for $\mathcal{Z}_{y^{(p,i),L,d_x}}$ we have that $\norm{M-\tilde{M}}_{\infty}\le\varepsilon$.
		Finally, by definition $\varepsilon\textrm{-}\rank{M}\leq\rank{\tilde{M}} \leq\rank{\mat{\A(\tilde{y})}_{\aaa,\bb}}\leq K$. 
	\end{proof}
	
	Relying on lemma~\ref{lemma:grid_sep_deep}, we will bound the $\varepsilon$-sequential-separation rank from below via the $\varepsilon$-rank of sub-matrices of $\mat{\A(\mathcal{Z}_{y^{(p,i),L,d_x}_\text{\emph{in-context}}})}_{\aaa,\bb}$. 
	Denote $d\coloneqq\nicefrac{\left(d_x-H\right)}{2},\lambda\coloneqq3^{L-2}$, lemmas~\ref{lemma:lambda_k_lower_bound},~\ref{lemma:number_of_high_eigenvalues_lower_bound} assure us that for $n\coloneqq\multiset{d}{\lambda}=\Omega\left(2^{L\cdot d_x}\right)$ it is enough to prove that there exists an assignment to the network's weights, as well as choice of templates vectors, for which there exists sub-matrix $M\in\R^{n\times n}$ of $\mat{\A(\mathcal{Z}_{y^{(p,i),L,d_x}_\text{\emph{in-context}}})}_{\aaa,\bb}$ that is symmetric with diagonal entries that equals to $1$ and with $\norm{M}_F\le\left(\sqrt{\left(d+1\right)}\right)n^{\frac{3}{4}}$, in order to show that  $	\varepsilon\textrm{-}\mathrm{seq}\textrm{-}sep\left(y^{(p,i),L,d_x}_\text{\emph{in-context}}\right)\ge \frac{n^{\nicefrac{1}{4}}}{2\sqrt{d+1}}$ for $\varepsilon\le\frac{1}{2n^2}$. 
	
	Now we will use a corollary that is direct results of the proof in~\cite{levine2020limits}. This corollary shows that if $\mathcal{Z}_{y^{(p,i),L,d_x}_\text{\emph{in-context}}}$ is able to produce vectors that do not change the analysis in~\cite{levine2020limits}, then for any matrix $B\in\R^{n\times d}$ with rows that are $l_2$ normalized, there exists an assignment to the networks weights, as well as choice of templates vectors, for which there exists a sub-matrix of the grid tensor matricization that is equal to\footnote{We ignored the constant that appear in eq~28 of~\cite{levine2020limits} since we can get rid of this constant by dividing the last layer output matrices. Importantly, this constant is larger than $1$ and therefore the network weights boundedness assumption is not violated} $M=\left(BB^T\right)^{\odot \lambda}$.
	Importantly $M$ is symmetric. In addition, since the rows of $B$ are $l_2$ normalized, the diagonal entries of $M$ equals to $1$. Therefore, $M$ upholds the assumptions of lemmas~\ref{lemma:lambda_k_lower_bound},~\ref{lemma:number_of_high_eigenvalues_lower_bound},~\ref{lemma:grid_sep_deep} and it is enough to find $B$ for which $\norm{M}_F\le\left(\sqrt{\left(d+1\right)}\right)n^{\frac{3}{4}}$.
	
	\begin{corollary}\label{corollary:sufficient_assigmnet_layer_1}
		Let $d,\lambda>0$, assume that for any matrix $A\in\R^{\multiset{d}{3^{L-2}}\times d}$ with rows that are $l_2$ normalized, there exists a choice of template vectors $\x^{(1)},\ldots,\x^{(Z)}$, an assignment to the embedding layer and the first self-attention layer key and query weights, such that for any $j_1,j_2\in\left[\multiset{d}{3^{L-2}}\right]$ the output of the first self-attention layer on $\left(j_1,j_2+\multiset{d}{3^{L-2}}\right)$ is: $$\y^{(1,j)}=\left(\sum_{h=1}^H W^{O,1,h} W^{V,1,h} \right)\uu$$ for	
		\begin{align*}
			\forall \alpha\in\left[d_x\right]\quad\uu_\alpha=\begin{cases}
				A_{j_1,\phi(\alpha)} & (\alpha-1)\bmod d_{a}<\frac{d_{a}-1}{2}\wedge\phi(\alpha)\le d\\
				A_{j_2,\phi\left(\alpha-\frac{d_{a}-1}{2}\right)} & \frac{d_{a}-1}{2}\leq(\alpha-1)\bmod d_{a}<d_{a}-1\wedge\phi(\alpha-\frac{d_{a}-1}{2})\le d\\
				2N & (\alpha-1)\bmod d_{a}=d_{a}-1\\
				0 & \text{Otherwise}
			\end{cases}
		\end{align*}
		where $\phi(j) \equiv \left\lfloor \nicefrac{j - 1}{d_a} \right\rfloor \cdot (d_a - 1) + (j - 1 \bmod d_a) + 1$.
		
		Then for any matrix $B\in\R^{n\times d}$ with rows that are $l_2$ normalized, there exists an assignment to the networks weights, as well as choice of templates vectors for which there exists sub-matrix of the grid tensor matricization that equal to $M=\left(BB^T\right)^{\odot \lambda}$.
	\end{corollary}
	Now we will shows that indeed $\mathcal{Z}_{y^{(p,i),L,d_x}_\text{\emph{in-context}}}$ is able to produce vectors that do not change the analysis in~\cite{levine2020limits} and the assumptions of corollary~\ref{corollary:sufficient_assigmnet_layer_1} holds.
	\begin{lemma}\label{lemma:vocab_sequential_assigmnet}
		Let $A\in\R^{\multiset{d}{3^{L-2}}\times d}$ with rows that are $l_2$ normalized, then there exists a choice of template vectors $\x^{(1)},\ldots,\x^{(Z)}$, an assignment to the embedding layer and the first self-attention layer key and query weights, such that for any $j_1,j_2\in\left[\multiset{d}{3^{L-2}}\right]$ the output of the first self-attention layer on $\left(j_1,j_2+\multiset{d}{3^{L-2}}\right)$ is: $$\y^{(1,j)}=\left(\sum_{h=1}^H W^{O,1,h} W^{V,1,h} \right)\uu$$ for	
		\begin{align*}
			\forall \alpha\in\left[d_x\right]\quad\uu_\alpha=\begin{cases}
				A_{j_1,\phi(\alpha)} & (\alpha-1)\bmod d_{a}<\frac{d_{a}-1}{2}\wedge\phi(\alpha)\le d\\
				A_{j_2,\phi\left(\alpha-\frac{d_{a}-1}{2}\right)} & \frac{d_{a}-1}{2}\leq(\alpha-1)\bmod d_{a}<d_{a}-1\wedge\phi(\alpha-\frac{d_{a}-1}{2})\le d\\
				2N & (\alpha-1)\bmod d_{a}=d_{a}-1\\
				0 & \text{Otherwise}
			\end{cases}
		\end{align*}
		where $\phi(j) \equiv \left\lfloor \nicefrac{j - 1}{d_a} \right\rfloor \cdot (d_a - 1) + (j - 1 \bmod d_a) + 1$.
	\end{lemma}
	\begin{proof}
		We will ignore $\mathcal{Z}_{y^{(p,i),L,d_x}_\text{\emph{in-context}}}$'s element-wise multiplication with vocabulary embedding matrix by choosing $\forall i,j\,M^{\textrm{V}}_{i,j}=1$ (by the terms of corollary~\ref{corollary:sufficient_assigmnet_layer_1} it suffices to find any assignment of the learned weights).
		
		For any $i\in\left[2\multiset{d}{3^{L-2}}+1\right]$ our templates vectors will be: 
		\begin{align*}
			x_{j}^{(i)}=\frac{1}{N}\begin{cases}
				A_{i,\phi(j)} & \substack{i\leq\multiset{d}{3^{L-2}}\\
					\wedge(j-1)\bmod d_{a}<\frac{d_{a}-1}{2}\wedge\phi(\alpha)\le d
				}
				\\
				A_{i-\multiset{d}{3^{L-2}}+1,\phi\left(j-\frac{d_{a}-1}{2}\right)} & \substack{\multiset{d}{3^{L-2}}<i\leq2\multiset{d}{3^{L-2}}\\
					\wedge\frac{d_{a}-1}{2}\leq(j-1)\bmod d_{a}<d_{a}-1\wedge\phi(\alpha-\frac{d_{a}-1}{2})\le d
				}
				\\
				N & (j-1)\bmod d_{a}=d_{a}-1\\
				0 & \text{Otherwise}
			\end{cases}
		\end{align*}
		
		We will implement summation of the inputs embedding in the first self-attention layer, we will follow \cite{levine2020limits} and set the first layer self-attention key and query weights to:
		\begin{align*}		
			W_{i,j}^{K,1,h}&=W_{i,j}^{Q,1,h}=1_{i=1 \wedge j=d_{a}}
		\end{align*}
		This assignment implements summation of the inputs embedding in the first self-attention layer since:
		\begin{align}
			\y^{(1,i)}(\x^{(d_{1})},\ldots,\x^{(d_{2})})_{\alpha}&=\sum_{j=1}^{N}\sum_{h=1}^{H}\left\langle W^{Q,1,h}\left(M_{\textrm{V}}w_{1}^{j}\odot\x^{(d_{1})}\right),W^{K,1,h}\left(M_{\textrm{V}}w_{1}^{j}\odot\x^{(d_{1})}\right)\right\rangle \\&W^{O,1,h}W^{V,1,h}\left(M_{\textrm{V}}w_{1}^{j}\odot\x^{(d_{1})}\right)\\&+\sum_{j=1}^{N}\sum_{h=1}^{H}\left\langle W^{Q,1,h}\left(M_{\textrm{V}}w_{2}^{j}\odot\x^{(d_{2})}\right),W^{K,1,h}\left(M_{\textrm{V}}w_{2}^{j}\odot\x^{(d_{2})}\right)\right\rangle \\&W^{O,1,h}W^{V,1,h}\left(M_{\textrm{V}}w_{2}^{j}\odot\x^{(d_{2})}\right)\\&\overset{1}{=}\sum_{j=1}^{N}\sum_{h=1}^{H}\overbrace{\x^{(d_{1})}_{d_{a}}}^{=1}\cdot\overbrace{\x^{(d_{1})}_{d_{a}}}^{=1}W^{O,1,h}W^{V,1,h}\x^{(d_{1})}\\&+\sum_{j=1}^{N}\sum_{h=1}^{H}\overbrace{\x^{(d_{2})}_{d_{a}}}^{=1}\cdot\overbrace{\x^{(d_{2})}_{d_{a}}}^{=1}W^{O,1,h}W^{V,1,h}\x^{(d_{2})}\\&\overset{2}{=}\left(\sum_{h=1}^{H}W^{O,1,h}W^{V,1,h}\right)N\left(\x^{(d_{1})}+\x^{(d_{2})}\right)
		\end{align}
		where $(1)$ is because $W^{Q,1,h} = W^{K,1,h}$ are matrices that are zero everywhere except for entry $(1,d_a)$ and that all the entries in the vocabulary embedding matrix equals to $1$, and $(2)$ because of linearity.	
		Therefore, for any $j_1,j_2\in\left[\multiset{d}{3^{L-2}}\right]$ the output of the first self-attention layer on $j_1,j_2$ is:
		\begin{align}
			\label{equation:summation_assigmnet_first_layer_output}
			\y^{(1,j)}=\left(\sum_{h=1}^{H}W^{O,1,h}W^{V,1,h}\right)\underbrace{N\left(\x^{(d_{1})}+\x^{(d_{2})}\right)}_{\eqqcolon\uu}
		\end{align}
		
		Finally, we need to show that indeed for any $j_1,j_2\in\left[\multiset{d}{3^{L-2}}\right]$ eq~\ref{equation:summation_assigmnet_first_layer_output} give the desired $\uu$:
		\begin{align*}
			\forall\alpha\in\left[d_{x}\right]\quad\uu_{\alpha}=\begin{cases}
				A_{j_{1},\phi(\alpha)} & (\alpha-1)\bmod d_{a}<\frac{d_{a}-1}{2}\wedge\phi(\alpha)\le d\\
				A_{j_{2},\phi\left(\alpha-\frac{d_{a}-1}{2}\right)} & \frac{d_{a}-1}{2}\leq(\alpha-1)\bmod d_{a}<d_{a}-1\wedge\phi(\alpha-\frac{d_{a}-1}{2})\le d\\
				2N & (\alpha-1)\bmod d_{a}=d_{a}-1\\
				0 & \text{Otherwise}
			\end{cases}
		\end{align*}
		The third and forth cases are clear from $\x's$ definition, so it remain to prove the first and second cases. For this we will examine $d_1,d_2$.		
		$d_1=j_1\leq \multiset{d}{3^{L-2}}$ and therefore:
		\begin{align*}
			\x_{\alpha}^{(d_{1})}=\begin{cases}
				A_{j_{1},\phi(\alpha)} & (\alpha-1)\bmod d_{a}<\frac{d_{a}-1}{2}\wedge\phi(\alpha)\le d\\
				0 & \frac{d_{a}-1}{2}\leq(\alpha-1)\bmod d_{a}<d_{a}-1\wedge\phi(\alpha-\frac{d_{a}-1}{2})\le d
			\end{cases}
		\end{align*}
		$d_2=j_2+\multiset{d}{3^{L-2}}\in\left[1+\multiset{d}{3^{L-2}},2\multiset{d}{3^{L-2}}\right] $ and therefore:
		\begin{align*}
			\x_{\alpha}^{(d_{2})}=\begin{cases}
				0 & (\alpha-1)\bmod d_{a}<\frac{d_{a}-1}{2}\wedge\phi(\alpha)\le d\\
				A_{j_{2},\phi\left(\alpha-\frac{d_{a}-1}{2}\right)} & \frac{d_{a}-1}{2}\leq(\alpha-1)\bmod d_{a}<d_{a}-1\wedge\phi(\alpha-\frac{d_{a}-1}{2})\le d
			\end{cases}
		\end{align*}
		So it clear that also the first and second cases upholds.
	\end{proof}
	
	Returning to finding $B$ for which $\norm{M}_F\le\left(\sqrt{\left(d+1\right)}\right)n^{\frac{3}{4}}$, we will use the probabilistic method for proving the existence of such $B$, \ie we will show that for random $B$ the expectation of $\norm{M}_F\le\left(\sqrt{\left(d+1\right)}\right)n^{\frac{3}{4}}$ and therefore in particular there exists such $B$.
	\begin{lemma}\label{lemma:expectations_of_frobenius_norm}
		For any $d,\lambda\in\N$ such that $\lambda\ge d$ the following holds:
		\begin{equation}
			\EE_{B\sim\left(\SSS^d\right)^n}\left[\norm{\left(BB^T\right)^{\odot \lambda}}_F\right]\le\sqrt{\left(d+1\right)}\multiset{d}{\lambda}^{\frac{3}{4}}
		\end{equation}
	\end{lemma}
	\begin{proof}
		We start by bounding the expectation of the squared norm:
		\begin{align}
			\EE_{B\sim\left(\SSS^{d}\right)^{n}}\left[\norm{\left(BB^{T}\right)^{\odot\lambda}}_{F}^{2}\right]&=\sum_{i,j=1}^{n}\EE_{B\sim\left(\SSS^{d}\right)^{n}}\left[\left[\left(BB^{T}\right)^{\odot2\lambda}\right]_{i,j}\right]\\&=\sum_{i,j=1}^{n}\EE_{u,v\sim\SSS^d}\left[{\langle u, v\rangle}^{2\lambda}\right]=n^2\EE_{u,v\sim\SSS^d}\left[{\langle u, v\rangle}^{2\lambda}\right]
		\end{align}
		Now, $\EE_{u,v\sim\SSS^d}\left[{\langle u,v\rangle}^{2\lambda}\right]\le\left(d+1\right)\multiset{d}{\lambda}^{-\frac{1}{2}}$ by lemma~\ref{lemma:expectations_of_cosine_similarity_powers_upper_bound} and therefore we got that:
		\begin{align}
			\EE_{B\sim\left(\SSS^{d}\right)^{n}}\left[\norm{\left(BB^{T}\right)^{\odot\lambda}}_{F}^{2}\right]\le\left(d+1\right)\multiset{d}{\lambda}^{\frac{3}{2}}
		\end{align}
		Finally, by Jensen inequality we have that:
		\begin{align}
			\EE_{B\sim\left(\SSS^d\right)^n}\left[\norm{\left(BB^T\right)^{\odot \lambda}}_F\right]\le\sqrt{\EE_{B\sim\left(\SSS^d\right)^n}\left[\norm{\left(BB^T\right)^{\odot \lambda}}_F^2\right]}\le\sqrt{\left(d+1\right)}\multiset{d}{\lambda}^{\frac{3}{4}}
		\end{align}
	\end{proof}	

	\subsection{Technical lemmas}\label{sec:technical_lemmas}
	\begin{lemma}\label{lemma:expectations_of_cosine_similarity_powers_upper_bound}
		For any $d,\lambda\in\N$ such that $\lambda\ge d$ the following holds:
		\begin{equation}
		\EE_{u,v\sim\SSS^d}\left[{\langle u, v\rangle}^{2\lambda}\right]\le\left(d+1\right)\multiset{d}{\lambda}^{-\frac{1}{2}}
		\end{equation}
	\end{lemma}
	\begin{proof}
		We will use conditional expectation to make reduction for simpler expectation. From rotational invariance of the uniform measure on $\SSS^d$ we know that for every rotation matrix $R\in SO\left(d\right)$ and unit vector $v\in\SSS^d$ we have that:
		\begin{equation}
		\EE_{u\sim\SSS^d}\left[{\langle u, v\rangle}^{2\lambda}|v\right]=\EE_{u\sim\SSS^d}\left[{\langle Ru, Rv\rangle}^{2\lambda}|v\right]=\EE_{u\sim\SSS^d}\left[{\langle u, Rv\rangle}^{2\lambda}|v\right]
		\end{equation}
		where the first equality holds because $R$ is orthogonal. Therefore, by choosing $R$ such that $Rv=e_1$ we will get that:
		\begin{align}
		\EE_{u,v\sim\SSS^{d}}\left[{\langle u,v\rangle}^{2\lambda}\right]&=\EE_{v\sim\SSS^{d}}\left[\EE_{u\sim\SSS^{d}}\left[{\langle u,v\rangle}^{2\lambda}|v\right]\right]\\&=\EE_{v\sim\SSS^{d}}\left[\EE_{u\sim\SSS^{d}}\left[{u_{1}}^{2\lambda}|v\right]\right]=\EE_{u\sim\SSS^{d}}\left[u_{1}^{2\lambda}\right]
		\end{align}
		Now we can calculate the last expectation directly:
		\begin{align}
		\EE_{u\sim\SSS^{d}}\left[u_{1}^{2\lambda}\right]=\int_{u\in\SSS^{d}}&u_{1}^{2\lambda}d\mu\left(u\right)=\int_{-1}^{1}\frac{\mu_{\text{unnormalized}}\left(\left(\sqrt{1-x^{2}}\right)\SSS^{d-1}\right)}{\mu_{\text{unnormalized}}\left(\SSS^{d}\right)}x^{2\lambda}dx
		\end{align}
		Now, by lemma~\ref{lemma:high_dimensional_unit_balls_voulume_ratio} and fact~\ref{fact:high_dimensional_balls_voulume_ratio} we have that:
		\begin{align}
		0<\frac{\mu_{\text{unnormalized}}\left(\left(\sqrt{1-x^{2}}\right)\SSS^{d-1}\right)}{\mu_{\text{unnormalized}}\left(\SSS^{d}\right)}\le\frac{d+1}{2\pi}\left(1-x^{2}\right)^{\frac{d}{2}}
		\end{align}
		Therefore we have that:
		\begin{align}
		\EE_{u,v\sim\SSS^{d}}\left[{\langle u,v\rangle}^{2\lambda}\right]\le\frac{d+1}{2}\int_{-1}^{1}\left(1-x^{2}\right)^{\frac{d}{2}}x^{2\lambda}dx=\left(d+1\right)\int_{0}^{1}\left(1-x^{2}\right)^{\frac{d}{2}}x^{2\lambda}dx
		\end{align}
		Finally, by lemma~\ref{lemma:intagrad_upper_bound} each term in the integral is upper bounded by $\multiset{d}{\lambda}^{-\frac{1}{2}}$ and thus:
		\begin{align}
		\EE_{u,v\sim\SSS^{d}}\left[{\langle u,v\rangle}^{2\lambda}\right]\le\left(d+1\right)\multiset{d}{\lambda}^{-\frac{1}{2}}
		\end{align}
	\end{proof}	
	\begin{lemma}\label{lemma:intagrad_upper_bound}
		For any $x\in\left[0,1\right]$ and $d,\lambda\in\N$ such that $\lambda\ge d$ the following holds:
		\begin{equation}
		x^{2\lambda}\left(1-x^2\right)^{\frac{d}{2}}\le\multiset{d}{\lambda}^{-\frac{1}{2}}
		\end{equation}
	\end{lemma}
	\begin{proof}
		Note that since $x^{2\lambda}\left(1-x^2\right)^{\frac{d}{2}}=0$ in the boundaries $\left(x\in\left\{0,1\right\}\right)$, it is enough to prove the inequality for critical points.
		\begin{align}
		0&=\left(x^{2\lambda}\left(1-x^{2}\right)^{\frac{d}{2}}\right)^{'}=\left(2\lambda x^{2\lambda-1}\right)\left(1-x^{2}\right)^{\frac{d}{2}}-\frac{2d}{2}x^{2\lambda+1}\left(1-x^{2}\right)^{\frac{d}{2}-1}\\&\iff0=2\lambda\left(1-x^{2}\right)-dx^{2}\iff x^{2}=\frac{2\lambda}{2\lambda+d}
		\end{align}
		Therefore, $x^{2}=\frac{2\lambda}{2\lambda+d}$ is the only critical point and:
		\begin{align}
		&x^{2\lambda}\left(1-x^{2}\right)^{\frac{d}{2}}\le\left(\frac{2\lambda}{2\lambda+d}\right)^{\lambda}\left(1-\frac{2\lambda}{2\lambda+d}\right)^{\frac{d}{2}}\\&=\left(1-\frac{d}{2\lambda+d}\right)^{\lambda}\left(\frac{d}{2\lambda+d}\right)^{\frac{d}{2}}\\&\le\left(1-\frac{d}{3\lambda}\right)^{\lambda}\left(\frac{d}{1.5\left(\lambda+d\right)}\right)^{\frac{d}{2}}\\&\le \left(\sqrt[3]{e}\right)^{-d}\left(\sqrt{\frac{e}{1.5}}\right)^{d}\left(\frac{d}{e\left(\lambda+d\right)}\right)^{\frac{d}{2}}\\&\le\left(\frac{d}{e\left(\lambda+d\right)}\right)^{\frac{d}{2}}\le\multiset{d}{\lambda}^{-\frac{1}{2}}
		\end{align}
		where the last inequality follow from the fact that:
		\begin{align}
		\multiset{\lambda}{d}=\binom{\lambda+d-1}{d}\le\left(\frac{e\left(\lambda+d-1\right)}{d}\right)^{d}\le\left(\frac{e\left(\lambda+d\right)}{d}\right)^{d}
		\end{align}
	\end{proof}

\section{Experimental Details}\label{exp_details}
\subsection{kNN-TAPT}
We conducted the network training described in section \ref{sec:3:1} of the main text with AdamW optimizer (with the parameters suggested in the original RoBERTa paper: $\beta_{1}=0.9$, $\beta_{2}=0.98$, $\varepsilon=10^{-6}$ and weight decay of 0.01), with batch sizes of 128 or 256 (depending on model size) and sequences of 256 tokens each. We started with pretrained RoBERTA-base weights from the HuggingFace Transformers repository \footnote{\url{https://huggingface.co/transformers/}}, and continued training them on the MLM task with masking probability of 15\%, where each masked token had a probability of 80\% of being replaced with the special $\mathtt{\left[MASK\right]}$ token, 10\% of being replaced with a random token and 10\% of being kept the same. The data used for this phase of training was created using the four different procedures described in section \ref{sec:3:1}. After the training was finished, we evaluated the models' performance using the SentEval kit.

\subsection{kNN-Pretraining}
We conducted the network training described in section \ref{sec:3:2} of the main text with AdamW optimizer (with the parameters suggested in the original GPT-2 paper: $\beta_{1}=0.9$, $\beta_{2}=0.95$, $\varepsilon=10^{-8}$ and weight decay of 0.1), with batch size of 512 and sequences of 256 tokens each. We pretrained a HuggingFace Transformers implementation of GPT-2 from scratch on Wikipedia with the standard LM objective, and switched to a mixture of the standard data and our generated kNN data in two different points during training. After the training was finished, we evaluated the models' performance on the Natural Questions benchmark.

\section{kNN-Pretraining at different checkpoints}\label{knn_pret_results}
The following table includes F1 evaluation scores of zero shoe closed book Natural Questions examples for different model sizes at different training checkpoints. Overall, further pretraining seems to improve the effectiveness of kNN-Pretraining. 	\begin{table}[h]
		\centering
		\begin{tabular}{ccc}
			\toprule
			Model  & Reg.+kNN & NQ \\
			Size & Steps& F1\\
			\midrule		
			$110$M  & 200 / 400+0 & $<10^{-3}$   \\
			$110$M  & 200+40 & $6.2\cdot10^{-3}$   \\
			$110$M  & 400+40 & $7.9\cdot10^{-3}$   \\
			$345$M  & 200 / 400+0 & $<10^{-3}$   \\	
			$345$M & 200+10 & $9.6\cdot10^{-3}$       \\
			$345$M& 400+10 & $1.4\cdot10^{-2}$       \\
			\bottomrule
		\end{tabular}
		\vspace{-0.5mm}
		\caption{Impact of model size and regular pretraining steps on the Natural Questions F1 score of kNN-Pretraining. \label{tab:ppl}\vspace{-1.2em}} \end{table}

\section{kNN-Pretraining on additional benchmarks}\label{app:glue_tasks}

The main text describes experiments on the Natural Questions dataset. We test how kNN-Pretraining affects other NLU tasks, by examining several tasks from the GLUE benchmark~\citep{wang2018glue} -- Multi-Genre Natural Language Inference (MNLI)~\citep{williams2017broad}, Recognizing Textual Entailment (RTE)~(\citet{dagan2010recognizing} and others), and the The Winograd Schema Challenge (WNLI)~\citep{levesque2012winograd}. 
As in the case of Natural Questions, we evaluate the zero-shot performance of our models since it is a direct probe to the abilities of the model straight after the process of pretraining. 
In contrast to Natural Questions, the GLUE tasks we examined are classification tasks and not generation tasks, so assessing zero shot performance on them is not straightforward. 
We therefore follow the template-based method of \citet{eval-harness} for converting the tasks' data into a format processable by unidirectional language models. 

Notably, the examined GLUE classification tasks are not easy for the examined unidirectional models in zero shot. 
Table~\ref{tab:glue} includes the zero-shot scores of the $345$M parameter model that trained regularly for $200$K steps and then continued training for $20$K steps of kNN-Pretraining, versus the average of 3 baselines that trained regularly for the same number of overall steps (the same models used in figure~\ref{fig:fig1}). Similarly to the results on Natural Questions (figure~\ref{fig:fig1}), all examined models score only slightly better than random guess on the examined GLUE tasks. However (and again similarly to the case of Natural Questions), we get a clear signal that kNN-Pretraining significantly moves the needle when applied for just $10\%$ of the regular pretraining time.
We conjecture that when using stronger models (that train for longer and over more data), the positive effect of kNN-Pretraining will be enhanced, since as the model improves, it can better understand and utilize the various in-context hints that kNN-Pretraining provides.

\begin{table}[h]
\centering
\begin{tabular}{lccc}
\toprule
GLUE Task        & MNLI & RTE& WNLI  \\
\midrule
Random guess & 33.3 & 50.0  & 50.0 \\
3 baselines -- Average score & 35.1 & 52.0  & 51.1 \\
3 baselines -- Max score & 35.3 & 52.3  & 54.9 \\
kNN-Pretraining  &\textbf{ 35.5} & \textbf{53.0}  & \textbf{56.3} \\
\bottomrule
\end{tabular}
\caption{Zero-shot accuracy scores on several GLUE tasks of a kNN-Pretrained model versus 3 baselines that trained regularly on the same data.} \label{tab:glue}\vspace{-1.2em}
\end{table}
\end{document}